\pdfoutput=1
\documentclass[10pt]{article}
\usepackage{enumerate}
\usepackage[OT1]{fontenc}
\usepackage{natbib}
\usepackage[usenames,dvipsnames]{xcolor}

\usepackage{amsmath}
\usepackage{amssymb}
\usepackage{mathtools}
\usepackage{amsthm}

\usepackage{geometry}
\usepackage{parskip}

\usepackage{dsfont}
\usepackage{pgfplots}
\pgfplotsset{compat=1.18}
\usepackage{smile}
\usepackage{multirow}
\usepackage{rotating}
\usepackage{esvect}
\usepackage{enumitem}
\usepackage{tikz}
\usetikzlibrary{patterns}
\usetikzlibrary{arrows}
\usepackage[colorlinks,
            linkcolor=red,
            anchorcolor=blue,
            citecolor=blue
            ]{hyperref}
\usepackage{algorithm}
\usepackage{algorithmic}

\def\supp{{\rm supp}}

\def\cS{{\mathcal{S}}}

\providecommand{\norm}[1]{\vvvert#1\vvvert}

\newcommand{\bel}{\begin{eqnarray}\label}
\newcommand{\eel}{\end{eqnarray}}
\newcommand{\bes}{\begin{eqnarray*}}
\newcommand{\ees}{\end{eqnarray*}}

\def\real{{\mathbb{R}}}

\newcommand{\mc}{\mathcal}

\let\emptyset\varnothing
\let\hat\widehat
\let\tilde\widetilde
\def\mid{\,|\,}
\def\eps{\epsilon}

\def\E{{\mathbb E}}

\def\supp{\mathop{\text{supp}\kern.2ex}}
\def\argmin{\mathop{\text{\rm arg\,min}}}

\def\given{{\,|\,}}

\def\supp{\mathop{\text{supp}}}

\def\v{{\mathbb{V}}}
\def\pp{{\mathbb{P}}}

\theoremstyle{plain}

\theoremstyle{plain}

\usepackage{mathrsfs}
\usepackage{fullpage}

\def \s{\mathcal{S}}
\def \a{\mathcal{A}}
\def \m{\mathcal{M}}
\def \r{\mathbb{R}}
\def \rl{\mathcal{R}}
\def \p{\mathbb{P}}

\def\RR{\mathsf{R}}

\def \V{\mathcal{V}}

\usepackage{hyperref}
\usepackage[protrusion=false,expansion=true]{microtype}

\def\##1\#{\begin{align}#1\end{align}}
\def\$#1\${\begin{align*}#1\end{align*}}
\newcommand{\Bin}{{\rm Bin}}
\newcommand{\abs}[1]{{\left| #1 \right|}}
\newcommand{\paren}[1]{{\left( #1 \right)}}
\newcommand{\brac}[1]{{\left[ #1 \right]}}
\newcommand{\set}[1]{{\left\{ #1 \right\}}}
\newcommand{\sets}[1]{{\{ #1 \}}}
\newcommand{\Eqref}[1]{Eq.\eqref{#1}}

\newcommand{\defeq}{\mathrel{\mathop:}=}

\makeatletter
\newcommand\footnoteref[1]{\protected@xdef\@thefnmark{\ref{#1}}\@footnotemark}
\makeatother

 \title{Is Inverse Reinforcement Learning Harder than \\ Standard Reinforcement Learning? A Theoretical Perspective}
\date{\today}
\author{Lei Zhao\thanks{University of Science and Technology of China. Email: \texttt{zl20071451@mail.ustc.edu.cn}.} \qquad Mengdi Wang\thanks{Princeton University. Email: \texttt{mengdiw@princeton.edu}.} \qquad Yu Bai\thanks{Salesforce AI Research. Email: \texttt{yu.bai@salesforce.com}.}}

\begin{document}

\maketitle

\def \parab{\overline{\para}}
\newcommand{\rew}{\mathcal{R}^{\msf{all}}}
\newcommand{\rewfb}{\mathcal{R}^{\msf{feas}}_{[-B,B]}}
\newcommand{\rewfone}{\mathcal{R}^{\msf{feas}}_{[-1,1]}}
\newcommand{\rewfth}{\mathcal{R}^{\msf{feas}}_{[-3H,3H]}}

\newcommand{\tO}{\widetilde{\cO}}

\def \Vf{V^{\pp,r}}
\def \Qf{Q^{\pp,r}}
\def \piE{\pi^{\msf{E}}}
\def \hatpiE{\hat{\pi}^{\msf{E}}}
\def \pA{\overline{\cA}}
\def \pV{\overline{\cV}}
\def \para{\Theta}
\def \rewm{\mathscr{R}}
\def \RR{\rewm}
\def \fedb{e}
\def \wellpi{\Delta}
\def \wellsupp{n^{\msf{E}}}
\def \valsupp{n^{\msf{val}}}
\newcommand{\dpi}{d^\pi}
\newcommand{\dpiTheta}{d^\pi}
\newcommand{\DpiTheta}{D^\pi_\para}
\newcommand{\dall}{d^{\sf all}}
\newcommand{\dallTheta}{d^{{\sf all}}}
\newcommand{\DallTheta}{D^{\sf all}_\Theta}
\def \cn{\log\mathcal{N}(\Theta; \eps/H)}
\def \cns{\log\mathcal{N}}
\def \dimt{\dim\paren{\para}}
\def \pib{\pi^{\msf{b}}}
\def \pival{\pi^{\msf{eval}}}
\def \dpival{d^{\pival}}
\def \Dpivalt{D^{\pival}_\Theta}
\def \DpiEt{D^{\piE}_\Theta}
\def \hRR{\widehat{\RR}}
\def \sRR{\RR^\star}
\def \hpp{\widehat{\pp}}
\def \DH{D^{\msf{H}}}
\def \DM{D^{\msf{M}}}
\def \Pidet{\Pi^{\mrm{determ}}}
\def \dpip{d^{\pp,\pi}}
\def \wellp{\eta}
\def \dt{\mrm{dim}\paren{\Theta}}
\def \hsa{[H]\times \cS\times \cA}

\newcommand{\lone}[1]{\norm{#1}_1} %
\newcommand{\ltwo}[1]{\norm{#1}_2} %
\newcommand{\lop}[1]{\norm{#1}_{\mathrm{op}}}
\newcommand{\lops}[1]{\|{#1}\|_{\mathrm{op}}}
\newcommand{\linf}[1]{\norm{#1}_\infty} %
\newcommand{\lzero}[1]{\norm{#1}_0} %
\newcommand{\dnorm}[1]{\norm{#1}_*} %
\newcommand{\lfro}[1]{\left\|{#1}\right\|_{\sf Fr}} %
\newcommand{\lnuc}[1]{\left\|{#1}\right\|_*} %
\newcommand{\matrixnorm}[1]{\left|\!\left|\!\left|{#1}
  \right|\!\right|\!\right|} %
\newcommand{\normbigg}[1]{\bigg\|{#1}\bigg\|} %
\newcommand{\normbig}[1]{\big\|{#1}\big\|} %
\newcommand{\normBig}[1]{\Big\|{#1}\Big\|} %
\newcommand{\lonebigg}[1]{\normbigg{#1}_1} %
\newcommand{\ltwobigg}[1]{\normbigg{#1}_2} %
\newcommand{\ltwobig}[1]{\normbig{#1}_2} %
\newcommand{\ltwoBig}[1]{\normBig{#1}_2} %
\newcommand{\linfbigg}[1]{\normbigg{#1}_\infty} %
\newcommand{\norms}[1]{\|{#1}\|} %
\newcommand{\matrixnorms}[1]{|\!|\!|{#1}|\!|\!|} %
\newcommand{\matrixnormbigg}[1]{\bigg|\!\bigg|\!\bigg|{#1}\bigg|\!\bigg|\!\bigg|} %
\newcommand{\lones}[1]{\norms{#1}_1} %
\newcommand{\ltwos}[1]{\norms{#1}_2} %
\newcommand{\ltwob}[1]{\big\|#1\big\|_2}
\newcommand{\linfs}[1]{\norms{#1}_\infty} %
\newcommand{\lzeros}[1]{\norms{#1}_0} %
\newcommand{\lambdamax}[1]{\lambda_{\max}({#1})}
\newcommand{\lambdamin}[1]{\lambda_{\min}({#1})}
\newcommand{\vmax}[1]{v_{\max}({#1})}
\newcommand{\vmin}[1]{v_{\min}({#1})}

\def \dis{{d}}
\def \cblue{}
\def \cblack{}
\newcommand{\online}{}
\newcommand{\bsh}{\textbackslash}
\newcommand{\va}{\V\times\a}
\def \dtheta{\mathrm{dim}\paren{\Theta}}
\def \Ctran{C^{\msf{tran}}}
\def \wCtran{C^{\msf{wtran}}}

\begin{abstract}

Inverse Reinforcement Learning (IRL)---the problem of learning reward functions from demonstrations of an \emph{expert policy}---plays a critical role in developing intelligent systems. While widely used in applications, theoretical understandings of IRL present unique challenges and remain less developed compared with standard RL. For example, it remains open how to do IRL efficiently in standard \emph{offline} settings with pre-collected data, where states are obtained from a \emph{behavior policy} (which could be the expert policy itself), and actions are sampled from the expert policy.

This paper provides the first line of results for efficient IRL in vanilla offline and online settings using polynomial samples and runtime. Our algorithms and analyses seamlessly adapt the pessimism principle commonly used in offline RL, and achieve IRL guarantees in stronger metrics than considered in existing work. We provide lower bounds showing that our sample complexities are nearly optimal. As an application, we also show that the learned rewards can \emph{transfer} to another target MDP with suitable guarantees when the target MDP satisfies certain similarity assumptions with the original (source) MDP.

\end{abstract}

\def \irlalg{\textsc{Inverse-Reinforcement-Learning DEC}}

\section{Introduction}\label{sec:inro}

Inverse Reinforcement Learning (IRL) aims to recover reward functions from demonstrations of an \emph{expert policy}~\citep{ng2000algorithms,abbeel2004apprenticeship}, in contrast to standard reinforcement learning which aims to learn optimal policies for a given reward function. IRL has applications in numerous domains such as robotics~\citep{argall2009survey,finn2016guided}, target-driven navigation tasks~\citep{ziebart2008maximum,sadigh2017active,kuderer2015learning,pan2020imitation,barnes2023massively}, game AI~\citep{ibarz2018reward,vinyals2019grandmaster}, and medical decision-making~\citep{woodworth2018preference,hantous2022detecting}. The learned reward functions in these applications are typically used for replicating the expert behaviors in similar or varying downstream environments. Broadly, the problem of learning reward functions from data is of rising importance beyond the scope of IRL, and is used in procedures such as Reinforcement Learning from Human Feedback (RLHF)~\citep{christiano2017deep} for aligning large language models~\citep{ouyang2022training,bai2022constitutional,openai2023gpt4,touvron2023llama}.

Despite the success of IRL in practical applications \citep{agarwal2019reinforcement,finn2016guided,sadigh2017active,kuderer2015learning,woodworth2018preference,wu2020efficient,ravichandar2020recent,vasquez2014inverse}, theoretical understanding is still in an early stage and presents several unique challenges, especially when compared with standard RL (finding optimal policy under a given reward) where the theory is more established. First, the solution is {\bf inherently non-unique} for \emph{any} IRL problem---For example, for any given expert policy, zero reward is always a feasible solution (making the expert policy optimal under this reward). A sensible definition of IRL would require not just recovering a single reward function but instead a \emph{set} of feasible rewards~\citep{metelli2021provably,lindner2023active}. Second, theoretical results for IRL is {\bf lacking even for some standard learning settings}, such as learning from an offline dataset of trajectories from the expert policy (akin to an imitation setting). Finally, as a more nuanced challenge (but related to both challenges above), so far there is {\bf no commonly agreed performance metric} for measuring the distance between the estimated reward set and the ground truth reward set. Existing performance metrics in the literature either require strong feedback such as a simulator~\citep{metelli2021provably,metelli2023theoretical}, or do not require the returned solution to be aware of the transition dynamics~\citet{lindner2023active} (see Section~\ref{sec:relationship} for a discussion). These challenges motivate the following open question:

\begin{center}
\textbf{Is IRL more difficult than standard RL?}
\end{center}

In this paper, we theoretically study IRL in standard episodic tabular Markov Decision Processes without Rewards (\MDPR{}'s) under vanilla offline and online learning settings. Our contributions can be summarized as follows.
\begin{itemize}[leftmargin=1.5em, itemsep=0pt, topsep=0pt]
\item The goal of IRL is to output a set of rewards that approximate the ground truth set of \emph{feasible} rewards, i.e. rewards under which the expert policy is optimal. We define new metrics for both reward functions and for IRL using the concept of \emph{reward mapping}, which can be viewed as a ``generating function'' of the (ground truth) set of feasible rewards (Section~\ref{sec:irl} \&~\ref{our_metric}). We show that our metrics are stronger / more appropriate than existing metrics in certain aspects (Section~\ref{sec:relationship}).
\item We show that any estimated reward that is similar in our metric and satisfies monotonicity with respect to the true reward admits an approximate planning/learning guarantee (Section~\ref{sec:planning}).

\item We design an algorithm, \RLPfull{} (\RLP{}) that performs IRL from any given offline demonstration dataset (Section~\ref{sec:offline learning}). Our algorithm returns an estimated reward mapping that is $\eps$-close in our metric and satisfies monotonicity, and requires a number of episodes that is polynomial in the size of the MDP as well as the single-policy concentrability coefficient between the \emph{evaluation policy} and the behavior policy that generated the states of the offline dataset. To our best knowledge, this is the first provably sample-efficient algorithm for IRL in the standard offline setting. 

Technically, the algorithm seamlessly adapts the pessimism principle from the offline RL literature to achieve the desired monotonicity and closeness conditions, demonstrating that IRL is ``not much harder than standard RL" in a certain sense.

\item We next design an algorithm \RLEfull{} (\RLE{}), which operates in a natural online setting where the learner can both actively explore the environment and query the expert policy, and achieves IRL guarantee in a stronger metric from polynomial samples (Section~\ref{sec:online learning}). Algorithm \RLE{} builds on a simple reduction to reward-free exploration~\citep{jin2020reward,li2023minimaxoptimal} and the \RLP{} algorithm. 

\item We establish sample complexity lower bounds for both the offline and online settings, showing that our upper bounds are nearly optimal up to a small factor (Section~\ref{sec:lower_bound_offline} \&~\ref{sec:lower_bound_online}).

\item We extend our results to a \emph{transfer learning} setting, where the learned reward mapping is transferred to and evaluated in a target \MDPR{} different from the source \MDPR{}. We provide guarantees for \RLP{} and \RLE{} under certain similarity assumptions between the source and target \MDPR{}s (Section~\ref{Sec:transfer learning} \& Appendix~\ref{appendix:proof for transfer learning}).
\end{itemize}

\subsection{Related work}

\textbf{Inverse reinforcement learning}~~
Inverse reinforcement learning (IRL) was first proposed by \cite{ng2000algorithms} and since then significantly developed in various follow-up approaches such as feature matching \citep{abbeel2004apprenticeship}, maximum margin \citep{ratliff2006maximum}, maximum entropy \citep{ziebart2008maximum}, relative entropy \citep{boularias2011relative}, and generative adversarial imitation learning \citep{ho2016generative}. Other notable approaches include Bayesian IRL 
\citep{ramachandran2007bayesian} which subsume IRL, and the reduction method \citep{brantley2019disagreement}. 

IRL has been successfully applied in many domains including target-driven navigation tasks \citep{ziebart2008maximum,sadigh2017active,kuderer2015learning,pan2020imitation}, robotics \citep{argall2009survey,finn2016guided,hadfield2016cooperative,kretzschmar2016socially,okal2016learning,kumar2023graph,jara2019theory}, medical decision-making \citep{woodworth2018preference,hantous2022detecting,gong2023federated,yu2019deep,chadi2022inverse}, and game AI 
\citep{finn2016guided,fu2017learning,qureshi2018adversarial,brown2019extrapolating}.

\textbf{Theoretical understandings of IRL}~~
Despite their successful applications, theoretical understandings of IRL are still in an early stage. Recently, \citet{metelli2021provably} pioneered the investigation of the sample complexity of IRL under the simulator (generative model) setting where the learner can directly query feedback from any (state, action) pair. This work was later extended by \citet{metelli2023theoretical}, who introduced a framework based on Hausdorff-based metrics for measuring distances between reward sets, examined relationships between different metrics, and provided corresponding lower bounds. However, their results critically rely on the simulator setting and do not generalize to more realistic offline/online learning settings. \citet{dexter2021inverse} also performed a theoretical analysis for IRL in the simulator setting with continuous states and discrete actions. 

The recent work of~\citet{lindner2023active} considers IRL in the online setting where the learner can interact with the \MDPR{} in an online fashion, which is closely related to our results for the online setting. Compared with our metric, their metric is defined for an estimated IRL problem (instead of an estimated reward set). Further, their metric does not effectively take into account the estimated transitions, which can lead to a family of counter-exmaples where the estimated IRL problem achieves perfect recovery under their metric, but the induced reward sets are actually far from the true feasible reward set in our metric (cf. Section~\ref{sec:relationship} for a detailed discussion). Our work improves upon the above works by introducing new performance metrics for IRL, and providing new algorithms for standard learning settings such as offline learning.

\textbf{Relationship with standard RL theory}~~
Our work builds upon various existing techniques from the sample-efficient RL literature to design our algorithms and establish our theoretical results. 
For the offline setting, our algorithm and analysis build upon the pessimism principle and the single-policy concentrability condition commonly used in offline RL~\citep{kidambi2020morel,jin2021pessimism,yu2020mopo,kumar2020conservative,rashidinejad2021bridging,xie2021policy,xie2022role}. For the online setting, we adapt the reward-free learning algorithm of~\citet{li2023minimaxoptimal} to find a policy that achieves a certain concentrability-like condition with respect to all policies.

We note theoretical results on imitation learning \citep{abbeel2004apprenticeship, ratliff2006maximum, ziebart2008maximum,levine2011nonlinear, fu2017learning, chang2021mitigating} and RLHF~\citep{zhu2023principled, zhu2023fine, wang2023rlhf, zhan2023provable}, which are related to but different from (and do not imply) our results. Additional related work is discussed in Appendix~\ref{app:additional-related}.

\section{Preliminaries}\label{Sec:prelim}

\textbf{Markov Decision Processes without Reward}~~
We consider episodic Markov Decision Processes without Reward (\MDPR), specified by $\mathcal{M} = (\mathcal{S}, \mathcal{A}, H, \pp)$, where $\s$ is the state space with $|\cS|=S$, $\a$ is the action space with $|\cA|=A$, $H$ is the horizon length, $\pp = \{\pp_h\}_{h\in[H]}$ where $\pp_h(\cdot|s, a) \in \Delta(\cS)$ is the transition probability at step $h$. Without loss of generality, we assume that the initial state is deterministically some $s_1\in\cS$.

\textbf{Reward functions}~~
A reward function $r :  [H]\times\s \times\a  \to [-1, 1]$ maps a state-action-time step triplet $(h,s,a)$ to a reward $r_h(s,a)$. Given an \MDPR{} $\m$ and a reward function $r$, we denote the MDP induced by $\cM$ and $r$ as $\m \cup r$. A policy $\pi = \{\pi_h(\cdot\given s)\}_{h\in[H],s\in\s}$, where $\pi_h:\cS\to \Delta(\cA)$ maps a state to an action distribution.

\textbf{Values and visitation distributions}~~
A policy $\pi=(\pi_h)_{h\in[H]}$, where each $\pi_h(\cdot|s)\in\Delta(\cA)$ for each $s\in\cS$. Let $\supp(\pi_h(\cdot|s))\defeq \sets{a: \pi_h(a|s)>0}$ denote the support set of $\pi_h(\cdot|s)$.
For any policy $\pi$ and any reward function $r$,
we define the value function $V^{\pi}_h(\cdot;r):\s\to \real$ at each time step $h \in [H]$ by the expected cumulative reward: $V^\pi_h(s;r)=\EE_{\pi}\sbr{\sum_{h'=h}^H r_{h'}(s_{h'},a_{h'})\given s_h=s }$, where $\EE_{\pi}$ denotes the expectation with respect to the random trajectory induced by $\pi$ in the \MDPR{}, that is, $(s_1, a_1, s_2, a_2, . . . , s_H, a_H )$, where $a_h \sim \pi_h(s_h), r_h = r_h(s_h, a_h), s_{h+1} \sim \p_h(\cdot\given s_h, a_h)$. Similarly, we denote the $Q$-function at time step $h$ as : $Q^\pi_h(s,a;r)=\EE_{\pi}\sbr{\sum_{h'=h}^H r_{h'}(s_{h'},a_{h'}) \given s_h=s, a_h=a }$. For any reward $r$, the corresponding advantage function $A^\pi_h(\cdot; r): \s\times\a \to \r$ is defined as $A_h^\pi(s,a;r) \defeq Q_h^\pi(s,a;r)-V^\pi_h(s;r)$ and we say a policy is an optimal policy of $\cM\cup r$ if $A_h^\pi(s,a;r)\leq 0$ holds for all $(h,s,a)\in [H]\times \cS\times \cA$\footnote{This definition of optimal policy requires $\pi$ to be optimal starting from any time step $h$ and state $s\in\cS$ (not necessarily visitable ones), which is stronger than the standard definition but is commonly adopted in the IRL literature~\citep{ng2000algorithms}.}. Additionally, we represent the set of all optimal policies for $\cM\cup r$ as $\Pi^\star_{\cM\cup r}$ and denote the set of all deterministic policies for $\cM\cup r$ as $\Pi^{\det}_{\cM\cup r}$.

We introduce $\dis^{\pi}_h$ 
  to denote the state(-action) visitation distributions associated with policy at time step $h \in [H]$:
  $\dis^{\pi}_h(s) := \p(s_h = s|\pi)$ and $\dis^{\pi}_h(s, a) := \p(s_h = s, a_h = a|\pi)$.
Lastly, we define the operators $\p_h$ and $\v_h$ by $[\p_hV_{h+1}](s, a)\defeq \mathbb{E}[V_{h+1}(s_{h+1})|s_h = s, a_h = a]$ and $[\v_hV_{h+1}](s, a)\defeq \mathsf{Var}[V_{h+1}(s_{h+1})|s_h = s, a_h = a]$ applying to any value function $V_{h+1}$ at time step $h+1$. In this paper, we will frequently employ $\widehat{\p}_h$ and $\widehat{\v}_h$ to represent empirical counterparts of these operators constructed based on estimated models. For any function $f:\cS\to\real$, define its infinity norm as $\|{f}\|_{\infty}\defeq \sup_{s\in\cS} \abs{f(s)}$ (and we define similarly for any $f:\cS\times\cA\to\real$).

\subsection{Inverse Reinforcement Learning}
\label{sec:irl}
An Inverse Reinforcement Learning (IRL) problem is denoted as a pair $(\m, \piE)$, where $\m$ is an \MDPR{} and $\pi^\msf{E}$ is a policy called the \emph{expert policy}. The goal of IRL is to interact with $(\m, \piE)$, and recover reward function $r$'s that are \emph{feasible} for $(\m, \piE)$, in the sense that $\piE$ an optimal policy for MDP $\m\cup r$. 

\paragraph{Reward mapping}
Noting that learning \emph{one} feasible reward function is trivial (the zero reward $r\equiv 0$ is feasible for any $\piE$), we consider the stronger goal of recovering the \emph{set} of all feasible rewards, which can be characterized by an explicit formula by the classical result of~\citet{ng2000algorithms}. Here we restate this result through the concept of a \emph{reward mapping}.

Let $\rew$ denote the set of all possible reward functions, and $\rewfb\defeq \sets{r\in\rew: r~\textrm{is feasible and}~|r|\le B}$ denote the set of all feasible rewards bounded by $B$ for any $B>0$. Let $\pV\defeq \pV_1\times\dots\times\pV_H$ and $\pA\defeq \pA_1\times\dots\times\pA_H$, where $\pV_h\defeq\set{V_h\in \real^{\cS}\mid\|V_h\|_{\infty}\leq H-h+1}$ and 
$\pA_h\defeq\set{A_h\in \real^{\cS\times \cA}_{\geq 0}\mid \|A_h\|_{\infty}\leq H-h+1}$ denote the set of all possible ``value functions'' and ``advantage functions'' respectively.

\begin{definition}[Reward mapping]
The (ground truth) \emph{reward mapping} $\RR^\star:\pV\times \pA\mapsto \rew$ of an IRL problem $(\m,\pi^\msf{E})$ is the mapping that maps any $(V,A)\in\pV\times\pA$ to the following reward function $r$:
\begin{align}
\label{def:rew-map}
& r_h(s,a) = [\RR^\star(V,A)]_h(s,a)\defeq 
-A_h(s, a) \\
& \quad \times \indic{a\notin \supp\paren{\piE_h(\cdot\given s) }} 
+ V_h(s)-[\p_h V_{h+1}](s,a), \nonumber
\end{align}
where we recall that $\p_h$ is the transition probability of $\m$ at step $h\in[H]$.
\end{definition}

With the definition of reward mapping ready, we now restate the classical result of~\citet{ng2000algorithms}, which shows that the reward mapping $\RR$ generates a set of rewards that is a superset of $\rewfone$---the set of all $[-1,1]$-bounded feasible rewards---by ranging over $(V,A)\in\pV\times\pA$.

\begin{lemma}[Reward mapping produces all bounded feasible rewards]
\label{lemma:metric_1}
The set of rewards $\RR^\star(\pV\times\pA)=\sets{\RR(V,A):(V,A)\in\pV\times\pA}$ induced by $\RR^\star$ satisfies
\begin{align}
    \rewfone \subseteq \RR^\star(\pV\times\pA) \subseteq \rewfth.
\end{align}
In words, $\RR^\star$ always produces feasible rewards bounded in $[-3H, 3H]$, and the set $\RR^\star(\pV\times\pA)$ contains (is a superset of) all $[-1,1]$-bounded feasible rewards.
\end{lemma}

As IRL is concerned precisely with the recovery of the set $\rewfone$, we consider the recovery of the reward mapping $\RR^\star$ itself as a natural learning goal---An accurate estimator $\hat{\RR}\approx \RR^\star$ guarantees $\hat{\RR}(V,A)\approx \RR^\star(V,A)$ for any $(V,A)\in\pV\times\pA$, and thus imply accurate estimation of $\RR^\star(\pV\times\pA)$ in precise ways which we specify in the sequel.

We will also consider recovering the reward mapping on a \emph{subset} $\para\subset\pV\times\pA$. We use the following standard definition of covering numbers to measure the capacity of such $\para$'s:

\begin{definition}[Covering number]
The \emph{$\eps$-covering number} of $\para\subset \pV\times\pA$ is defined as
\[
\textstyle
\mathcal{N}(\para; \eps)\defeq\max_{h\in [H]}\mathcal{N}(\pV^\para_h; \eps),
\]
where $\pV^\para_h\defeq \sets{V_h: (V,A)\in\Theta}$ denotes the restriction of $\para$ onto $\pV_h$, and $\mathcal{N}(\pV^\para_h; \eps)$ is the $\eps$-covering number of $\pV^\para_h$ in $\linfs{\cdot}$ norm.
\end{definition}
Note that $\log\mathcal{N}(\Theta; \eps) \le \min\set{ \log|\para|, \cO\paren{S\log(H/\eps)} }$ by combining the (trivial) bound for the finite case and the standard covering number bound for $\para=\pV\times\pA$~\citep{vershynin2018high}. In addition, the left-hand side may be much smaller than the right-hand side if $\para$ admits additional structure (for example, if $\pV^\para_h$ lies in a low-dimensional subspace of $\real^\cS$).

\section{Performance metrics for IRL}
\label{sec:metric}

\subsection{Metric for IRL}
\label{our_metric}

We now define our performance metric for IRL based on the recovery of reward mapping $\RR^\star$. Fixing any \MDPR{} $\m$, we begin by defining our \emph{base metric} $\dpi$ (indexed by a policy $\pi$) and $\dall$ between two rewards.

\begin{definition}[Base metric for rewards]
\label{def:metric_r}
We define the metric\footnote{\label{footnote:semimetric}Technically a semi-metric.} $\dpi$ (indexed by any policy $\pi$) between any pair of rewards $r,r'\in \rew$ as
\begin{align}
\label{dpi} 
\dpi\paren{r,r'}\defeq
\sup_{h\in[H]}\EE_{s_h\sim \pi} \abs{V^\pi_h(s_h;r)-V^\pi_h(s_h;{r'})}.
\end{align}
We further define $\dall\paren{r,r'}\defeq \sup_{\pi}\dpi\paren{r, r'}$.
\end{definition}

In words, metric $d^\pi$ compares the rewards $r$ and $r'$ when executing $\pi$. Concretely, ~\eqref{dpi} compares the difference in the value functions $V^\pi_h(\cdot; r)$ and $V^\pi_h(\cdot; r')$ averaged over the visitation distribution $s_h\sim \pi$, which is sensible for our learning settings as it takes into account the transition structure of $\m$ (compared with other existing metrics based the sup-distance over all states; cf. Section~\ref{sec:relationship}).
The stronger metric $\dall$ takes the supremum of $d^\pi$ over all policy $\pi$'s.

We now define our main metric $\DpiTheta$ for the recovery of reward mappings, which simply takes the supremum of $d^\pi$ between all pairs of rewards induced by the two reward mappings using the \emph{same} parameter $(V,A)\in\para$.

\begin{definition}[Metric for reward mappings]
\label{def:metric_rm}
Given any policy $\pi$ and any parameter set $\para$, we define the metric\footnoteref{footnote:semimetric} $\DpiTheta$ between any pair of reward mappings $\RR,\RR'$ as
\begin{align}
\label{Dpi} 
\DpiTheta\paren{\RR,\RR'}\defeq
\sup_{(V,A)\in \para} \dpi\paren{\RR\paren{V,A},\RR'\paren{V,A}}.
\end{align}
We further define $\DallTheta\paren{\RR,\RR'}\defeq \sup_\pi \DpiTheta\paren{\RR,\RR'}$.
\end{definition}
\eqref{Dpi} compares two reward mappings $\RR$ and $\RR'$ by measuring the distance between $\RR(V,A)$ and $\RR'(V,A)$ using our base metric and taking the sup over all $(V,A)\in\para$. Another common choice in the IRL literature is the Hausdorff distance (based on some base metric) between the two sets $\RR(\pV\times\pA)$ and $\RR'(\pV\times\pA)$~\citep{metelli2021provably,metelli2023theoretical,lindner2023active}. We show that~\eqref{Dpi} is always stronger than the Haussdorff distance in the sense that a metric of the form~\eqref{Dpi} is greater or equal to the Hausdorff distance regardless of the base metric (Lemma~\ref{lem:dm dh}), and the inequality can be strict for some base metric (Lemma~\ref{lem:h_metric_bad}).

\subsection{Implications for learning with estimated reward}
\label{sec:planning}

For IRL, a natural desire for a base metric between rewards is that, a small metric between $r$ and $\hat{r}$ should imply that learning (planning) using reward $\hat{r}$ in $\m$ should at most incur a small error when the true reward is $r$. The following result shows that our metric $\dpi$ satisfies such a desiderata. The proof can be found in Appendix~\ref{app:proof-mon_use_reward}.

\begin{proposition}[Planning with estimated reward]
\label{prop:mon_use_reward}
Given an \MDPR{} $\m$, let $r,\widehat{r}$ be a pair of rewards such that
\begin{enumerate}[topsep=0pt, itemsep=0pt, label=(\alph*)]
\item (Small $d^\pi$ on near-optimal policy) $d^{\pi}(r,\widehat{r})\leq \epsilon$ for some $\bar{\epsilon}$ near-optimal policy $\pi$ for MDP $\m\cup{r}$;
\item (Monotonicity) $\widehat{r}_h(s,a)\le r_h(s,a) $ for any $(h,s,a)\in[H]\times \cS\times \cA$.
\end{enumerate}
Then, letting $\widehat{\pi}$ be any $\epsilon'$ near-optimal policy for MDP $\m\cup \widehat{r}$, i.e, $V_1^\star(s_1; \widehat{r})-V_1^{\widehat{\pi}}(s_1; \widehat{r})\leq \epsilon'$, we have
\begin{align}
\label{eq: use_reward}
    V_1^\star(s_1; r)-V_1^{\widehat{\pi}}(s_1; r)\leq \epsilon+\epsilon'+2\bar{\epsilon},
\end{align}
i.e. $\hat{\pi}$ is also $(\epsilon+\epsilon'+2\bar{\epsilon})$ near-optimal for $\m\cup r$.
\end{proposition}
Proposition~\ref{prop:mon_use_reward} ensures that any estimated reward $\hat{r}$ that satisfies (a) small $\DpiTheta$ and (b) monotonicity with respect to the true reward will incur a small error when used in planning. We emphasize that monotonicity is necessary in order for~\eqref{eq: use_reward} to hold, similar to how pessimism is necessary for near-optimal learning in offline bandits/RL~\citep{jin2021pessimism}. Throughout the rest of the paper, we focus on designing IRL algorithms that satisfy (a) \& (b). These guarantees can then directly yield planning/learning guarantees as corollaries by Proposition~\ref{prop:mon_use_reward}, and we will omit such statements.

\subsection{Relationship with existing metrics}
\label{sec:relationship}

Our metrics $\dpi$ and $\dall$ differ from several metrics for IRL used in existing theoretical work, which we discuss here. 

\citet{lindner2023active} measures the difference between two reward mappings implicitly by a metric $D^L$ (see~\eqref{eqn:metric_lindner}) between the two inducing IRL problems (the ground truth problem $(\m, \piE)$ and the estimated problem $(\hat{\m}, \hatpiE)$ returned by an algorithm). The following result shows that $D^L$ is weaker than our metric $\DallTheta$ in a strong sense.
\begin{theorem}[Relationship with $D^L$; informal]
\label{theorem:dl}
The metric $D^L$ defined in~\eqref{eqn:metric_lindner} satisfies the following:
\begin{enumerate}[topsep=0pt, itemsep=0pt, label=(\alph*)]
\item (Informal version of Prop.~\ref{prop:can do lind}) Under the same setting as Theorem~\ref{thm:online_main} (in which our algorithm \RLE{} achieves $\eps$ error in $\DallTheta$), \RLE{} also achieves $\eps$ error in $D^L$ with the same sample complexity therein.
\item (Informal version of Prop.\ref{prop:lindner}) Conversely, there exists a family of pairs of IRL problems which has distance 0 in the $D^L$ metric but distance $1$ in the $\DallTheta$ metric between the induced reward mappings.
\end{enumerate}
\end{theorem}

In a separate thread, the works of~\citet{metelli2021provably,metelli2023theoretical} consider IRL under access to a simulator. Their metric between two reward functions requires the induced value/Q functions to be close \emph{uniformly over all $(s,a)\in\cS\times\cA$} (cf. Appendix~\ref{appendix:m_b_rewards}), regardless of whether the state is visitable by a policy in this particular \MDPR{}), which is tailored to the simulator setting and does not applicable to the standard offline/online settings considered in this work. By contrast, our metrics $\dpi$ and $\dall$ measure the distance between the induced value functions \emph{averaged over visitation distributions}, which are more tractable for the offline/online settings.

\section{IRL in the offline setting}\label{sec:offline learning}

\begin{algorithm*}[t]
\caption{\RLPfull}\label{alg:RLPfull}
\small
\begin{algorithmic}[1]
\STATE\textbf{Input:} Dataset $\cD=\{(s_h^k, a_h^k, e_h^k)\}_{k=1,h=1}^{K, H}$, parameter set $\para\subset \pV\times \pA$, confidence level $\delta>0$, error tolerance $\epsilon>0$.
\FOR{$(h,s,a)\in [H]\times \cS\times \cA$}  \label{rlp:line_3}
 \STATE Compute the empirical transition kernel $\widehat{\p}_h$, the empirical expert policy $\widehat{\pi}^{\sf E}$ and the penalty term $b^{\theta}_h$ for all $\theta\in \para$ as follows:
  \begin{align}
&\widehat{\p}_h(s'\given s, a)=\frac{1}{{N}_h^b(s,a)\vee 1}\sum_{(s_h,a_h,s_{h+1})\in \cD} \indic{(s_h,a_h,s_{h+1})=(s,a,s')},   \label{rlp:eq_1}\\
  &\widehat{\pi}^{\msf{E}}_h(a\given s)=
 \begin{cases}
 \frac{1}{{N}^b_h(s)\vee 1}\cdot\sum_{(s_h,a_h,e_h)\in \cD}\indic{(s_h,e_h)=(s,a)}& \textrm{in option 1} ,\\
  \frac{1}{{N}^b_{h,1}(s)\vee 1}\cdot \sum_{(s_h,a_h,e_h)\in \cD}\indic{(s_h,a_h,e_h)=(s,a,1)} & \textrm{in option 2} ,
  \end{cases}\label{rlp:eq_2}\\
   &b^\theta_h(s,a) = 
    C\cdot\min\set{\sqrt{\frac{\cn\iota}{{N}_h^b(s,a)\vee 1}\brac{\widehat{\v}_h{V}_{h+1}}(s,a)}+
\frac{H\cn\iota}{{N}_h^b(s,a)\vee 1}+\frac{\epsilon}{H}\paren{1+\sqrt{\frac{\cn\iota}{{N}_h^b(s,a)\vee 1}}}, H},\label{rlp:eq_8}
\end{align}
where the visitation counts ${N}_h^b(s,a)\defeq\sum_{(s_h,a_h)\in \cD}\indic{(s_h,a_h)=(s,a)}$, 
${N}^b_h(s)\defeq\sum_{a\in \cA}{N}_h^b(s,a)$, ${N}^b_{h,1}(s)\defeq\sum_{(s_h,a_h,e_h)}\indic{(s_h,e_h)=(s,1)}$,  $\iota\defeq \log \paren{{HSA}/{\delta}}$ and $C>0$ is an absolute constant.
\ENDFOR
\STATE \textbf{Output}: Estimated reward mapping $\widehat{\RR}$ defined as follows: For all $(V,A)\in\para$,
\begin{align}
[\widehat{\RR}(V,A)]_h(s,a) \defeq
-A_h(s, a)\cdot \indic{a\notin {\rm supp}\paren{\widehat{\pi}^\msf{E}_h(\cdot|s)}} + V_h(s) - [\widehat{\p}_hV_{h+1}](s,a)-b^\theta_h(s,a).
\label{rlp:eq_4}
\end{align}

\end{algorithmic}
\end{algorithm*}

\subsection{Setting}
In the offline setting, the learner does not know $(\m, \piE)$, and only has access to a dataset $\mathcal{D}=\{(s_h^k, a_h^k, e_h^k)\}_{k=1,h=1}^{K, H}$ consisting of $K$ iid trajectories without reward from $\m$, where actions are obtained by executing some \emph{behavior policy} $\pib$ in $\m$: $a_h^k\sim \pib_h(\cdot|s_h^k)$ for all $(k,h)$, and the \emph{expert feedback} $e_h^k$'s are obtained from the expert policy $\piE$ using one of the following two options:
\begin{align}
\fedb_h^k=
\begin{cases}
     a_h^{\msf{E},k} \sim \pi^{\msf{E}}_h(\cdot|s_h^k)&\textrm{in option 1}, \\
    \indic{a_h^k \in {\rm supp}\paren{\pi^{\msf{E}}_h(\cdot|s_h^k)}}&\textrm{in option 2}.\label{def_feedback}
\end{cases}
\end{align}

Option $1$, where the learner directly observes an \emph{expert action} $a_h^{\msf{E},k}$, is the commonly employed setting in the IRL literature \citep{metelli2021provably,lindner2023active,metelli2023theoretical}. In the special case where $\pib=\piE$, we can take $a_h^{\msf{E},k}\defeq a_h^k$, i.e. no need for additional expert feedback when the behavior policy coincides with the expert policy. We also allow option 2, in which $e_h^k$ indicates whether $a_h^k$ ``is an expert action'' (belongs to the support of $\piE_h(\cdot|s)$). As we will see, both options suffice for performing IRL.

Additionally, for option 1, we require the following well-posedness assumption on the expert policy $\piE$.
\begin{assumption}[Well-posedness]\label{def:well_pose}
\label{ass:well_posed_policy}
For any $\Delta\in(0,1]$, we say policy $\piE$ is $\Delta$-well-posed if
\begin{align}
\min_{(h,s,a):\piE_h(a|s)\neq 0}\piE_h(a|s)\geq \wellpi.\label{eq:def_well_pose}
\end{align}
\end{assumption}
This assumption is also made by~\citet[Assumption D.1]{metelli2023theoretical}, and is necessary for ruling out the edge case where $\piE_h(a|s)$ is positive but extremely small for some action $a\in\cA$, in which case a large number of samples is required to determine $\indic{a\in\supp(\piE_h(\cdot|s))}$.

\subsection{Algorithm}

We now present our algorithm \RLPfull{} (\RLP{}; full description in Algorithm~\ref{alg:RLPfull}) for IRL in the offline setting. \RLP{} returns an estimated reward mapping $\hat{\RR}$ given any offline dataset $\cD$. At a high level, \RLP{} consists of two main steps:
\begin{itemize}[topsep=0pt, itemsep=0pt, leftmargin=1.5em]
\item (Empirical MDP) We estimate the transition probabilities $\pp_h$ and expert policy $\piE$ by standard empirical estimates $\hat{\pp}_h$ and $\hatpiE$, as in~\eqref{rlp:eq_1} and~\eqref{rlp:eq_2}.

\item (Pessimism) We compute a bonus function $b^{\theta}_h(s,a)$ for any $\theta=(V,A)\in \para$,  $(h,s,a)\in [H]\times \cS\times \cA$ as in~\eqref{rlp:eq_8}. The final estimated reward (and thus the reward mapping)~\eqref{rlp:eq_4} is defined by the empirical version of the ground truth reward~\eqref{def:rew-map} combined with the negative bonus $-b^\theta_h(s,a)$, for every parameter $(V,A)\in\para$.

The specific design of $b^\theta_h(s,a)$ is based on Bernstein's inequality, and ensures that with high probability, for all $(h,s,a,\theta)$ simultaneously,
\begin{align*}
& b^{\theta}_h(s,a)\geq A_h(s, a)\times \big| \indic{a\notin \supp\paren{\widehat{\pi}^{\sf E}_h(\cdot | s) }} \notag \\
&  - \indic{a\notin \supp\paren{{\pi}^{\sf E}_h(\cdot | s) }} \big| + \big|\big[\big( \widehat{\p}_h-\p \big) V_{h+1}\big](s,a) \big|.
\end{align*}
Combined with the form of the ground truth reward $\RR(V,A)$ in~\eqref{def:rew-map}, a standard pessimism argument ensures the monotonicity condition $[\hat{\RR}(V,A)]_h(s,a) \le [\RR(V,A)]_h(s,a)$ for all $(h,s,a)$ and all $(V,A)$.
\end{itemize}
Therefore, in Algorithm~\ref{alg:RLPfull}, the empirical estimates ensure that the estimated reward~\eqref{rlp:eq_4} is close to the ground truth reward (over $s_h\sim\cD$ or equivalently the behavior policy $\pib$), whereas the pessimism (negative bonus) ensures the monotonicity condition, both being desired properties for IRL as discussed in Section~\ref{our_metric}.

\subsection{Theoretical guarantee}

We now state our theoretical guarantee for Algorithm~\ref{alg:RLPfull}. To measure the quality of the recovered reward mappings, we will be considering the $\dpi$ and $\DpiTheta$ metric with $\pi=\pival$ being any given \emph{evaluation policy}. We assume that $\pival$ satisfies the standard single-policy concentrability condition with respect to the behavior policy $\pib$.

\begin{assumption}[Average form single-policy concentrability]
\label{ass:off_1}
We say $\pival$ satisfies $C^\star$-single-policy concentrability with respect to $\pib$ if (with the convention $0/0=0$)
\begin{align}
\label{eq:def_concentrability}
\frac{1}{HS} \sum_{h\in [H]}\sum_{(s,a)\in \cS\times \cA}\frac{d^{\pival}_h(s,a)}{d^{\pi^{\msf{b}}}_h(s,a)}\leq C^\star.
\end{align}
\end{assumption}
Assumption~\ref{ass:off_1} is standard in the offline RL literature~\citep{jin2021pessimism,rashidinejad2021bridging,xie2021policy}, though we remark that our~\eqref{eq:def_concentrability} only requires the \emph{average} form, instead of the worst-case form made in~\citep{rashidinejad2021bridging,xie2021policy} which requires the distribution ratio to be bounded for all $(h,s,a)$.

We are now ready to present the guarantee for \RLP{} (Algorithm~\ref{alg:RLPfull}). The proof can be found in Appendix~\ref{app:proof-offline_main}. 

\begin{theorem}[Sample complexity of \RLP{}]
\label{thm:offline_main}
Let $\pival$ be any policy that satisfies $C^\star$ single-policy concentrability  (Assumption~\ref{ass:off_1}) with respect to $\pib$. Assume that $\piE$ is $\Delta$-well-posed (Assumption~\ref{def:well_pose}) if we choose option 1 in~\eqref{def_feedback}. 

Then for both options, with probability at least $1-\delta$, \RLP{} (Algorithm~\ref{alg:RLPfull}) outputs a reward mapping $\hRR$ such that
\begin{align*}
    \Dpivalt\paren{\sRR,\hRR}\leq \epsilon, \brac{\hRR(V,A)}_{h}(s,a)\leq \brac{\sRR(V,A)}_{h}(s,a)
\end{align*}
for all $(V,A)\in \para$ and $(h,s,a)\in [H]\times \cS\times \cA$, as long as the number of episodes
\begin{align*}
   K \ge \widetilde{\cO}\paren{\frac{H^4SC^\star \cns}{\epsilon^2}+\frac{H^2SC^\star \eta}{\epsilon}}.
\end{align*}
Above, $\cns\defeq \cn$, $\eta\defeq \Delta^{-1}\indic{{\rm option}~1}$, and $\tO(\cdot)$ hides ${\rm polylog}\paren{H,S,A,1/\delta}$ factors.
\end{theorem}

To our best knowledge, Theorem~\ref{thm:offline_main} provides the first theoretical guarantee for IRL under the standard offline setting, showing that \RLP{} achieves the desired monotonicity condition and small $\DpiTheta$ distance for any evaluation policy $\pival$ that satisfies single-policy concentrability with respect to $\pib$. For small enough $\epsilon$, the sample complexity (number of episodes required) scales as $\tO(H^4SC^\star\cns/\eps^2)$, which depends on the number of states $S$, the concentrability coefficient $C^\star$, as well as the log-covering number $\cns$ which always admits the bound $\cns\le \tO(S)$ in the worst case and may be smaller. 

Apart from the $\cns$ factor, this rate resembles that of standard offline RL under single-policy concentrability~\citep{rashidinejad2021bridging,xie2021policy}. This is no coincidence, as our algorithm and proof (for both the $\Dpivalt$ bound and the monotonicity condition) can be viewed as an adaptation of the pessimism technique for all rewards $(\RR(V,A))_{(V,A)\in\para}$ simultaneously, demonstrating that IRL is ``no harder than standard RL" in this setting. We remark that the $\Delta^{-1}$ factor brought by Assumption~\ref{def:well_pose} appears only in the $\tO(\eps^{-1})$ burn-in term in the rate when the feedback $\sets{e_h^k}_{k,h}$ in~\eqref{def_feedback} comes from option 1.

\paragraph{Result for $\pival=\piE$}
In the special case where $\pival=\piE$, we establish a slightly stronger result where we can improve over Theorem~\ref{thm:offline_main} by one $H$ factor ($H^4\to H^3$) in the main term. The proof uses the specific form of our Bernstein-like bonus~\eqref{rlp:eq_8} combined with a total variance argument~\citep{azar2017minimax,zhang2020almost,xie2021policy}, and can be found in Appendix~\ref{app:proof-pi=piE}.

\begin{theorem}[Improved sample complexity for $\pival=\piE$]
\label{coro:pi=piE}
Suppose $\pival=\piE$ which achieves $C^\star$ single-policy concentrability with respect to $\pib$ (Assumption~\ref{ass:off_1}), and in addition $\sup_{(h,s,a)\in \hsa} \abs{\brac{\sRR(V,A)}_h(s,a)}\leq 1$ for all $(V,A)\in {\para}$. Then under both options in~\eqref{def_feedback}, with probability at least $1-\delta$, \RLP{} (Algorithm~\ref{alg:RLPfull}) achieves the same guarantee as in Theorem~\ref{thm:offline_main} ($\Dpivalt(\RR^\star, \hat{\RR})\le \epsilon$ and monotonicity), as long as the number of episodes
\begin{align*}
   K \ge \widetilde{\cO}\paren{\frac{H^3SC^\star \cns}{\epsilon^2}+\frac{H^2SC^\star \paren{A+H\cns}}{\epsilon}}.
\end{align*}
\vspace{-1em}
\end{theorem}
Theorem~\ref{coro:pi=piE} no longer requires well-posedness of $\piE$ (Assumption~\ref{ass:well_posed_policy}) in option 1. This happens due to the assumed concentrability between $\piE(=\pival)$ and $\pib$, which can aid the learning of $\supp\paren{\piE_h(\cdot|s)}$ even without well-posedness.

\paragraph{IRL from full expert trajectories}
An important special case of Theorem~\ref{coro:pi=piE} is when $\pib$ further coincides with $\piE$. This represents a natural and clean setting where dataset $\cD$ consists of full trajectories drawn from the expert policy $\piE$, and our goal is to recover a reward mapping with a small $\DpiEt$. This case is covered by Theorem~\ref{coro:pi=piE} by taking $C^\star=1$ and admits a sample complexity $\tO(H^3S\cns/\epsilon^2)$.

\subsection{Lower bound}
\label{sec:lower_bound_offline}

We present an information-theoretic lower bound showing that the upper bound in Theorem~\ref{thm:offline_main} is nearly tight.
\begin{theorem}[Informal version of Theorem~\ref{thm:lower for offline}]
\label{thm:lower_offline_informal}
For any $(H,S,A,\epsilon)$ and any $C^\star\ge 1$, there exists a family of offline IRL problems where $\cD$ consists of $K$ episodes, $\pival$ satisfies $C^\star$-concentrability at most $C^\star$, $\Theta=\pV\times\pA$, and $\piE$ is $\Delta$ well-posed with $\Delta=1$, such that the following holds. 

Suppose any IRL algorithm achieves $\Dpivalt(\RR^\star, \hat{\RR})\le\epsilon$ for every problem in this family with probability at least $2/3$, then we must have $K\ge {\Omega}\paren{H^2SC^\star\min\set{S,A}/{\epsilon^2}}$.
\end{theorem}
For $\Theta=\pV\times\pA$, the upper bound in Theorem~\ref{thm:offline_main} scales as $\tO(H^4S^2C^\star/\epsilon^2)$. Ignoring $H$ and polylogarithmic factors, Theorem~\ref{thm:lower_offline_informal} assert that this rate is tight for $S\le A$ (so that $\min\sets{S,A}=S$). The form of this $\min\sets{S,A}$ factor in Theorem~\ref{thm:lower_offline_informal} is due to certain technicalities in the hard instance construction; whether this can be improved to an $S$ factor would be an interesting question for future work.

\vspace{-.5em}
\section{IRL in the online setting}\label{sec:online learning}
\vspace{-.5em}

\begin{algorithm*}[t]
\caption{\RLEfull}\label{alg:RLEF}
\small
\begin{algorithmic}[1]
  \STATE \textbf{Input:}
  Parameter set $\para\subseteq \pV\times\pA$, confidence level $\delta>0$, error tolerance $\eps>0$, $N,K\in\mathbb{Z}_{\ge 0}$, threshold $\xi=c_{\xi}H^3S^3A^3\log \frac{10HSA}{\delta}$.
  \STATE Call Algorithm~\ref{alg:ex} to play in the environment for $NH$ episodes and obtain an explorative behavior policy $\pi^{\msf{b}}$.\label{line:rle_1}
  \STATE Collect a dataset $\cD=\{(s_h^k, a_h^k, e_h^k)\}_{k=1,h=1}^{K, H}$ by executing $\pi^{\msf{b}}$ in $\m$.\label{line:rle_2}
  \STATE Subsampling: subsample $\cD$ to obtain $\cD^{\textrm{trim}}$, such that for each $(h, s, a) \in [H]\times \s \times \a$, $\mathcal{D}^{\textrm{trim}}$ contains $\min\set{\widehat{N}_h^{\msf{b}}(s,a),N_h(s,a)}$ sample transitions randomly drawn from $\mathcal{D}$, where $\widehat{N}_h^\msf{b}(s,a)$ and  $N_h(s,a)$ are defined by%
  \vspace{-.5em}
  \begin{align}
  N_h(s,a)\defeq\sum_{k=1}^K \indic{(s^k_h,a^k_h)=(s,a)}\quad \widehat{N}_h^{\mathsf{b}}(s, a) \defeq\min\brac{\frac{K}{4}, \mathop{\mathbb{E}}_{\pi\sim \mu^{\msf{b}}}[\widehat{d}_{h}^{\pi}(s,a)] - \frac{K\xi}{8N} - 3\log\frac{10HSA}{\delta}}_+,\label{eq:def:hatN}
 \end{align}
  where $\widehat{d}_{h}^{\pi}(s,a)$ is specified in Algorithm~\ref{alg:ex}.\label{line:rle_3}
  \STATE Call \RLP{} (Algorithm~\ref{alg:RLPfull}) on dataset $\cD^{\textrm{trim}}$ with parameters $(\para, \delta/10, \epsilon/10)$ to compute the recovered reward mapping ${\hRR}$. \label{line:rle_4}
  \STATE \textbf{Output:} Estimated reward mapping ${\hRR}$.  
\end{algorithmic}
\end{algorithm*}

\subsection{Setting}

We now consider IRL in a natural online learning setting (also known as ``active exploration IRL"~\citep{lindner2023active}). In each episode, the learner interacts with the IRL problem $(\m, \piE)$ as follows: At each $h\in[H]$, the learner receives the state $s_h\in\cS$ and chooses their action $a_h\in\cA$ from an arbitrary policy. The environment then provides the expert feedback $e_h$ as in~\eqref{def_feedback} (from one of the two options) and transits to the next state $s_{h+1}\sim \pp_h(\cdot|s_h, a_h)$. This setting shares the same expert feedback model ($e_h$) with the offline setting, and differs in that the learner can interact with the environment, instead of learning from a fixed dataset pre-collected by some fixed behavior policy.

\subsection{Algorithm and guarantee}

Our algorithm~\RLEfull{} (\RLE{}; Algorithm~\ref{alg:RLEF}) performs IRL in the online setting by a simple reduction to reward-free learning and the \RLP{} algorithm. \RLE{} consists of two main steps: (1) Call a reward-free exploration subroutine (Algorithm~\ref{alg:ex}, building on the algorithm of~\citet{li2023minimaxoptimal}) to explore the environment $\m$ and obtain an explorative behavior policy $\pib$ (Line~\ref{line:rle_1}); (2) Collect $K$ episodes of data $\cD$ using $\pib$, subsample the data, and call the \RLP{} algorithm on the subsampled data $\cD^{\rm trim}$ to obtain the estimated reward mapping $\hat{\RR}$.

We now present the theoretical guarantee of \RLE{}. The proof can be found in Appendix~\ref{app:proof-online_main}.
\begin{theorem}[Sample complexity of \RLE]\label{thm:online_main}
Suppose $\piE$ is $\Delta$-well-posed (Assumption~\ref{def:well_pose}) when we receive feedback \textrm{in option $1$} of~\eqref{def_feedback}. Then for the online setting, for sufficiently small $\epsilon\le H^{-9}(SA)^{-6}$, with probability at least $1-\delta$, \RLE{} (Algorithm~\ref{alg:RLEF}) with $N=\tO(\sqrt{H^9S^7A^7K})$ outputs a reward mapping $\hRR$ such that
\begin{align*}
    \DallTheta\paren{\sRR,\hRR}\leq \epsilon, \brac{\hRR(V,A)}_{h}(s,a)\leq \brac{\sRR(V,A)}_{h}(s,a)
\end{align*}
for all $(V,A)\in \para$ and $(h,s,a)\in [H]\times \cS\times \cA$, as long as the total the number of episodes
\begin{align*}
   K + NH \ge \widetilde{\cO}\paren{\frac{H^4SA \cns}{\epsilon^2}+\frac{H^2SA \eta}{\epsilon}}.
\end{align*}
Above, $\cns\defeq \cn$, $\eta\defeq \Delta^{-1}\indic{{\rm option}~1}$, and $\tO(\cdot)$ hides ${\rm polylog}\paren{H,S,A,1/\delta}$ factors.
\end{theorem}
For small enough $\epsilon$, \RLE{} requires $\tO(H^4SA\cns/\eps^2)$ episdoes for finding $\RR$ with $\DallTheta(\RR^\star, \hat{\RR})\le\epsilon$.
Compared with the offline setting (Theorem~\ref{thm:offline_main}), the main differences here are that the metric is stronger ($\DallTheta$ versus $\Dpivalt$ therein), and that the concentrability coefficient $C^\star$ in the sample complexity is replaced with the number of actions $A$. This is because using online interaction, our reward-free exploration subroutine (Algorithm~\ref{alg:ex}) can find a policy $\pib$ that achieves a form of ``single-policy concentrability" $A$ with respect to any policy $\pi$; see~\eqref{eq:stopping_rule}.

To our best knowledge, the only existing work that studies IRL in the same online setting is \citet{lindner2023active}, who also achieve a sample complexity\footnote{Extracted from the proof of~\citet[Theorem 8]{lindner2023active} and taking into account the uniform convergence over $\pV\times\pA$ and dependence on $\eta=\Delta^{-1}\indic{{\rm option}~1}$; cf. Appendix~\ref{sub:lindner}.} of $\tO(H^4S^2A/\epsilon^2+H^2SA\eta/\epsilon)$ (for $\para=\pV\times\pA$) in their performance metric $D^L$ (cf.~\eqref{eqn:metric_lindner}). However, our metric $\DallTheta$ is stronger than their $D^L$ and avoids certain indistinguishability issues of theirs, as we have shown in Theorem~\ref{theorem:dl}.

\subsection{Lower bound}
\label{sec:lower_bound_online}

We also provide a lower bound for IRL in the online setting in the $\DallTheta$ metric. The rate of the lower bound is similar to Theorem~\ref{thm:lower_offline_informal}, and ensures that the rate in Theorem~\ref{thm:online_main} is tight up to $H$ and polylogarithmic factors when $S\le A$.

\begin{theorem}[Informal version of Theorem~\ref{thm:lower_bound_online}]
For any $(H,S,A,\epsilon)$, there exists a family of online IRL problems where $\Theta=\pV\times\pA$, and $\piE$ is $\Delta$ well-posed with $\Delta=1$, such that the following holds. Suppose any IRL algorithm achieves $\DallTheta(\RR^\star, \hat{\RR})\le\epsilon$ for every problem in this family with probability at least $2/3$, then we must have $K\ge {\Omega}\paren{H^3SA\min\set{S,A}/{\epsilon^2}}$.
\end{theorem}

\vspace{-.5em}
\section{Transfer learning}\label{Sec:transfer learning}
As a further application, we consider a transfer learning setting, where rewards learned in a source \MDPR{} are transferred to a target \MDPR{} (possibly different from the source \MDPR). Inspired by the single-policy concentrability assumption, we define two concepts called \emph{weak-transferability} and \emph{transferability} (Definition~\ref{def:week-Transferability} \&~\ref{def:Transferability}) that measure the similarity between two \MDPR{}'s.

We show that when the target \MDPR{} exhibits a small {week-transferability} ({transferability}) with respect to the source \MDPR{}, our algorithms \RLP{} and \RLE{} can perform IRL with sample complexity polynomial in these transferability coefficients and other problem parameters (Theorem~\ref{thm:trans_offline} \&~\ref{thm:trans_online}), and provide guarantees for performing RL algorithms with the learned rewards in the target environments (Corollary~\ref{coro:perform_RL_offline} \&~\ref{coro:performing RL online}). We defer the detailed setups and results to Appendix~\ref{appendix:proof for transfer learning}.

\vspace{-.5em}
\section{Conclusion}
\label{sec:conclusion}
\vspace{-.5em}

This paper designs the first provably sample-efficient algorithm for inverse reinforcement learning (IRL) in the offline setting. Our algorithms and analyses seamlessly adapt the pessimism principle in standard offline RL, and we also extend it to an online setting by a simple reduction aided by reward-free exploration. We believe our work opens up many important questions, such as generalization to function approximation settings and empirical verifications.

 \bibliographystyle{ims}
 \bibliography{reference}

 \newpage 

 \appendix

\section{Additional related work}
\label{app:additional-related}

\paragraph{Imitation learning}
A closely related field to IRL is Imitation Learning, which focuses on learning policies from demonstrations, in contrast to IRL's emphasis on learning rewards from expert demonstrations \citep{bain1995framework, abbeel2004apprenticeship, ratliff2006maximum, ziebart2008maximum, pan2017agile, finn2016guided}. Imitation learning has been extensively studied in the active setting \citep{ross2011reduction, ross2014reinforcement, sun2017deeply}, and theoretical analyses for Imitation Learning have been provided by \citet{rajaraman2020toward, xu2020error,chang2021mitigating}. More recently, the concept of Representation Learning for Imitation Learning has gained considerable attention \citep{arora2020provable, nachum2021provable}.
While Imitation learning can be implemented by IRL \citep{abbeel2004apprenticeship, ratliff2006maximum, ziebart2008maximum}, it is important to note that IRL has wider capabilities than Imitation Learning since the rewards learned through IRL can be transferred across different environments \citep{levine2011nonlinear, fu2017learning}.

\paragraph{Reinforcement learning from human feedback}
Reinforcement Learning from Human Feedback (RLHF) bears a close relation to IRL, particularly because the process of learning rewards is a crucial aspect of both approaches \citep{zhu2023principled, zhu2023fine, wang2023rlhf, zhan2023provable}. RLHF has been successfully applied in various domains, including robotics \citep{jain2013learning, sadigh2017active, ding2023learning} and game playing \citep{ibarz2018reward}. Recently, RLHF has attracted considerable attention due to its remarkable capability to integrate human knowledge with large language models \citep{ouyang2022training, openai2023gpt4}. Furthermore, the theoretical foundations of RLHF have been extensively developed in both tabular and function approximation settings \citep{zhan2023provable, xu2020preference, pacchiano2021dueling, novoseller2020dueling, zhu2023principled, wang2023rlhf}.

\section{Technical tools}
\begin{lemma}[\citet{xie2021policy}]
  \label{lemma:binomial-concentration}
  Suppose $N\sim\Bin(n, p)$ where $n\ge 1$ and $p\in[0,1]$. Then with probability at least $1-\delta$, we have
  \begin{align*}
    \frac{p}{N\vee 1} \le \frac{8\log(1/\delta)}{n}.
  \end{align*}
\end{lemma}
\begin{theorem}[\citet{metelli2023theoretical}]\label{thm:BHIneq}
Let $\mathbb{P}$ and $\mathbb{Q}$ be probability measures on the same measurable space $(\Omega, \mathcal{F})$, and let ${A} \in \mathcal{F}$ be an arbitrary event. Then,
\begin{align*}
	\mathbb{P}({A}) + \mathbb{Q}({A}^c) \ge \frac{1}{2} \exp \left( - D_{\textrm{KL}}(\mathbb{P},\mathbb{Q}) \right),
\end{align*}
where ${A}^c = \Omega \setminus \mathcal{A}$ is the complement of ${A}$.
\end{theorem}
\begin{theorem}
[\citet{metelli2023theoretical}]\label{thm:Fano}
Let $\mathbb{P}_0,\mathbb{P}_1,\ldots,\mathbb{P}_M$ be probability measures on the same measurable space $(\Omega, \mathcal{F})$, and let ${A}_1,\dots,{A}_M \in \mathcal{F}$ be a partition of $\Omega$. Then,
\begin{align*}
	\frac{1}{M} \sum_{i=1}^M \mathbb{P}_i({A}_i^c) \ge 1 - \frac{ \frac{1}{M} \sum_{i=1}^M D_{\textrm{KL}}(\mathbb{P}_i,\mathbb{P}_0)  - \log 2}{\log M},
\end{align*}
where ${A}^c = \Omega \setminus {A}$ is the complement of ${A}$.
\end{theorem}

\def \com{\msf{aug}}
\def \mymid{\given}
\section{ Useful algorithmic subroutines from prior works}
\label{appendix:prior work}

In this section, we give the algorithm procedures of finding behavior policy $\pib$ in Algorithm~\ref{alg:RLEF}. The algorithm procedures are directly quoted from \citet{li2023minimaxoptimal}, with slight
modification.
\subsection{\texorpdfstring{Algorithm: finding behavior policy $\pib$}{Algorithm: finding behavior policy pib}}

Algorithm~\ref{alg:ex}, a component of \citet[Algorithm 1]{li2023minimaxoptimal}, aims to identify a suitable behavior policy. This is achieved by estimating the occupancy distribution $\dis^{\pi}$, which is induced by any deterministic policy $\pi$, through a meticulously designed exploration strategy. At each stage $h$, Algorithm~\ref{alg:ex} invokes Algorithm procedure~\ref{alg:sub1} to compute an appropriate exploration policy, denoted as $\pi^{\mathsf{explore},h}$, and subsequently collects $N$ sample trajectories by executing $\pi^{\mathsf{explore},h}$.
These steps facilitate the estimation of the occupancy distribution $\dis^{\pi}_{h+1}$ for the next stage $h+1$. Finally, the behavior policy $\pib\sim \mu_{\sf b}$ is computed by invoking Algorithm~\ref{alg:sub2}.

\begin{algorithm}[h]
\caption{Subroutine for computing behavior policy \citep{li2023minimaxoptimal}}\label{alg:ex}
\begin{algorithmic}[1]
\STATE\textbf{Input:} state space $\mathcal{S}$, action space $\mathcal{A}$, horizon length $H$, initial state distribution $\rho$, target success probability $1-\delta$, 
	threshold $\xi = c_{\xi}H^3S^3A^3\log(HSA/\delta)$.\label{line:ex_1}
\STATE Draw $N$ i.i.d.~initial states $s_1^{n,0} \overset{\mathrm{i.i.d.}}{\sim} \rho $ $(1\leq n\leq N)$, and define the following functions
\vspace{-1.5ex}
	\begin{equation}
		\label{eq:hat-d-1-alg}
		\widehat{\dis}_1^{\pi}(s) = \frac{1}{N}\sum_{n=1}^{N}\ind\{s_1^{n,0}=s\}, \qquad \widehat{\dis}_1^\pi (s,a)=\widehat{\dis}_1^\pi (s) \cdot\pi_1(a|s)
	\vspace{-1ex}
    \end{equation}
    for any deterministic policy $\pi: [H]\times \mathcal{S}\to\Delta(\a)$ and any $(s,a)\in \mathcal{S} \times\mathcal{A}$.
\FOR{$ h = 1,...,H-1$}
	  \STATE Call Algorithm \ref{alg:sub1} to compute an exploration policy $\pi^{\mathsf{explore},h}$. 
		\STATE Draw $N$ independent trajectories $\{s_1^{n,h},a_1^{n,h},\dots,s_{h+1}^{n,h}\}_{1\leq n\leq N}$
		using policy $\pi^{\mathsf{explore},h}$ and compute
		\vspace{-1.5ex}
		\[
		\widehat{\p}_{h}(s^{'} | s, a) = 
		\frac{\indic{N_h(s, a) > \xi}}{\max\big\{ N_h(s, a), 1 \big\} }\sum_{n = 1}^N \indic{s_{h}^{n,h} = s, a_{h}^{n,h} = a, s_{h+1}^{n,h} = s^{'}},  
		\qquad \forall (s,a,s')\in \cS\times \a\times \s, 
		\vspace{-1ex}
		\]
		where $N_h(s, a) = \sum_{n = 1}^N \indic{ s_{h}^{n,h} = s, a_{h}^{n,h} = a }$.  
		\STATE For any deterministic policy $\pi: \mathcal{S}\times[H]\to\Delta(\mathcal{A})$ and any $(s,a)\in \s\times \a$, define
		\vspace{-1.5ex}
		\begin{equation}
			\label{eq:hat-d-h-alg}
			\widehat{\dis}_{h+1}^{\pi}(s) = \big\langle \widehat{\p}_{h}(s | \cdot, \cdot), \,
			\widehat{\dis}_{h}^{\pi} (\cdot, \cdot) \big\rangle,  \qquad  
			\widehat{\dis}_{h+1}^{\pi}(s,a) = \widehat{\dis}_{h+1}^{\pi}(s) \cdot \pi_{h+1}(a| s) .
			\vspace{-1ex}
		\end{equation}
  \ENDFOR
\STATE Call Algorithm~\ref{alg:sub2} to compute a behavior policy $\pi^{\mathsf{b}}$.
\STATE \textbf{Output:} the behavior policy $\pi^{\mathsf{b}}$.  
\end{algorithmic}
\end{algorithm}

We highlight that the behavior policy distribution $\mu_{\sf b}$ output by Algorithm~\ref{alg:ex} has following property \citet{li2023minimaxoptimal}
\cblue
\begin{align}
   \sum_{h\in [H]}\sum_{(s,a)\in \cS\times \cA} \frac{\widehat{\dis}^{\pi}_h(s,a)}{\EE_{\pi'\sim\mu_{\sf b}}\brac{\widehat{\dis}^{\pi'}_h(s,a)}}\lesssim HSA,\label{eq:stopping_rule}
\end{align}
\cblack
for any deterministic policy $\pi\in \Pi^{\sf det}$.

\subsection{\texorpdfstring{Subroutine: computing exploration policy $\pi^{\mathsf{explore},h}$}{Subroutine: computing exploration policy pi mathsfexplore,h}}

We proceed to describe Algorithm~\ref{alg:sub1}, originally proposed in \citet[Algorithm 3]{li2023minimaxoptimal}, which is designed to compute the desired exploration policy $\pi^{\mathsf{explore},h}$. At a high level, this algorithm calculates the exploration policy by approximately solving the subsequent optimization sub-problem, utilizing the Frank-Wolfe algorithm:

\cblue
\begin{align}
			\widehat{\mu}^{h}\approx\arg\max_{\mu\in\Delta(\Pi)}\sum_{(s,a)\in\cS\times\cA}\log\bigg[\frac{1}{KH}+ \mathop{\mathbb{E}}_{\pi\sim\mu}\big[\widehat{\dis}_{h}^{\pi}(s,a)\big]\bigg], 
			\label{defi:target-empirical-h}
		\end{align}
\cblack

\begin{algorithm}[h]
\caption{Subroutine for solving \Eqref{defi:target-empirical-h} \citep{li2023minimaxoptimal}.}\label{alg:sub1}
\begin{algorithmic}[1]
\STATE\textbf{Initialize: } $\mu^{(0)}=\delta_{\pi_{\mathsf{init}}}$ for an arbitrary policy $\pi_{\mathsf{init}}\in\Pi$, $T_{\max}=\lfloor50SA\log(KH)\rfloor$. 
\FOR{$ t = 0,1...,T_{\max}$}
	\STATE Compute the optimal deterministic policy $\pi^{(t),\mathsf{b}}$ of the MDP $\mathcal{M}^h_{\mathsf{b}}=(\s \cup \{s_{\mathsf{aug}}\},\a,H,\widehat{\p}^{\mathsf{aug}, h},r_{\mathsf{b}}^h)$, 
		where $r_{\mathsf{b}}^h$ is defined in \Eqref{eq:reward-function-mub-h}, and $\widehat{\p}^{\mathsf{aug}, h}$ is defined in \Eqref{eq:augmented-prob-kernel-h}; 
		let $\pi^{(t)}$ be the corresponding optimal deterministic policy 
		of $\pi^{(t),\mathsf{b}}$  in the original state space. 
		
	\STATE Compute 
	\vspace{-2ex}
	\[
	\alpha_{t}=\frac{\frac{1}{SA}g(\pi^{(t)},\widehat{\dis},\mu^{(t)})-1}{g(\pi^{(t)},\widehat{\dis},\mu^{(t)})-1}, \quad\text{where}\quad 
	g(\pi, \widehat{\dis}, \mu)=\sum_{(s,a)\in \s\times\a}\frac{\frac{1}{KH}+\widehat{\dis}_{h}^{\pi}(s,a)}{\frac{1}{KH}+\mathbb{E}_{\pi\sim\mu}[\widehat{\dis}_h^{\pi}(s,a)]}.
	\vspace{-2ex}			
	\]
	Here, $\widehat{d}_h^{\pi}(s,a)$ is computed via \Eqref{eq:hat-d-1-alg} for $h=1$, and \Eqref{eq:hat-d-h-alg} for $h\geq 2$.
	\vspace{0.2ex}
	\STATE If $g(\pi^{(t)},\widehat{\dis},\mu^{(t)})\leq 2SA$ then \textsf{exit for-loop}.  \label{line:stopping-subroutine-h}
	\STATE Update
	\vspace{-0.5ex}
	\[
	\mu^{(t+1)}=\left(1-\alpha_t\right)\mu^{(t)}+\alpha_t \ind_{\pi^{(t)}}.
	\vspace{-1ex}
	\]\\
	\ENDFOR
 \STATE \textbf{Output}:  the exploration policy $\pi^{\mathsf{explore},h}=\mathbb{E}_{\pi\sim \mu^{(t)}}[\pi]$ and the weight $\widehat{\mu}^h=\mu^{(t)}$. 
\end{algorithmic}
\end{algorithm}

 Here
 $\mathcal{M}_{\mathsf{b}}^h=(\cS \cup \{s^{\com}\},\cA,H,\widehat{\p}^{\com, h},r_{\mathsf{b}}^h)$, where $s_{\com}$ is an augmented state as before, and the reward function is chosen to be

 				\begin{align} \label{eq:reward-function-mub-h}
					r_{\mathsf{b},j}^{h}(s,a)=%
 \begin{cases}
 	\frac{1}{\frac{1}{KH}+\mathbb{E}_{\pi\sim\mu^{(t)}}\big[\widehat{d}_{h}^{\pi}(s,a)\big]}\in[0,KH],\quad & \text{if }(s,a,j)\in\cS\times\cA\times\{h\};\\
 0, & 
 	\text{if } s = s_{\com} \text{ or } j\neq h.  
 \end{cases}
 		\end{align}

 		In addition, the augmented probability transition kernel $\widehat{\p}^{\com, h}$ is constructed based on $\widehat{\p}$ as follows:
		\begin{subequations}
			\label{eq:augmented-prob-kernel-h}
 		\begin{align}
			\widehat{\p}^{\com, h}_{j}(s^{'}\mymid s,a) &=    \begin{cases}
 			\widehat{\p}_{j}(s^{'}\mymid s,a),
 			& 
			\text{if }  s^{'}\in \cS \\
 			1 - \sum_{s^{'}\in \cS} \widehat{\p}_{j}(s^{'}\mymid s,a), & \text{if }  s^{'} = s_{\com}
 			\end{cases}
 			\qquad 
 			&&\text{for all }(s,a,j)\in \cS\times \cA\times [h]; \\
 			\widehat{\p}^{\com, h}_{j}(s^{'}\mymid s,a) &= \ind(s^{'} = s_{\com})
 			&&\text{if } s = s_{\com} \text{ or } j > h.
 		\end{align}
 		\end{subequations}

\subsection{\texorpdfstring{Subroutine: computing final behavior policy $\pib$}{Subroutine: computing final behavior policy pib}}
We proceed to describe Algorithm~\ref{alg:sub2}, originally proposed in \citet[Algorithm 2]{li2023minimaxoptimal}, which is designed to compute the final behavior policy $\pib$ $\pi^{\mathsf{explore},h}$, based on the estimated occupancy distributions specified in Algorithm~\ref{alg:ex}. Algorithm~\ref{alg:sub2} follows a similar fashion of Algorithm~\ref{alg:sub1}. Algorithm~\ref{alg:sub2} computes the behavior policy by approximately solving the subsequent optimization sub-problem, utilizing the Frank-Wolfe algorithm: 
\begin{align}
	\widehat{\mu}^{\sf b}  
	\approx 
	\arg\max_{\mu\in\Delta(\Pi)}\left\{ \sum_{h=1}^{H}\sum_{(s,a)\in\cS\times\cA}\log\bigg[\frac{1}{KH}+\mathbb{E}_{\pi\sim\mu}\big[\widehat{\dis}_{h}^{\pi}(s,a)\big]\bigg]\right\} .
	\label{defi:target-empirical}
\end{align}

\begin{algorithm}[H]
\caption{Subroutine for solving \Eqref{defi:target-empirical}  \citep{li2023minimaxoptimal}.}\label{alg:sub2}
\begin{algorithmic}[1]
\STATE\textbf{Initialize: }$\mu^{(0)}_{\mathsf{b}}=\delta_{\pi_{\mathsf{init}}}$ for an arbitrary policy $\pi_{\mathsf{init}}\in\Pi$, $T_{\max}=\lfloor50SAH\log(KH)\rfloor$.
\FOR{$ t = 0,1...,T_{\max}$}
		\STATE Compute the optimal deterministic policy $\pi^{(t),\mathsf{b}}$  of the MDP $\mathcal{M}_{\mathsf{b}}=(\s \cup \{s_{\mathsf{aug}}\},\cA,H,\widehat{\p}^{\mathsf{aug}},r_{\mathsf{b}})$, 
		where $r_{\mathsf{b}}$ is defined in \Eqref{eq:reward-function-mub}, and $\widehat{\p}^{\mathsf{aug}}$ is defined in \Eqref{eq:augmented-prob-kernel}; 
		let $\pi^{(t)}$ be the corresponding optimal deterministic policy 
		of $\pi^{(t),\mathsf{b}}$  in the original state space. 
		\label{line:pit-alg:sub2}
		
		\STATE Compute 
		\vspace{-1.5ex}
		\[
		\alpha_{t}=\frac{\frac{1}{SAH}g(\pi^{(t)},\widehat{\dis},\mu^{(t)}_{\mathsf{b}})-1}{g(\pi^{(t)},\widehat{\dis},\mu^{(t)}_{\mathsf{b}})-1}, \quad\text{where}\quad 	g(\pi,\widehat{\dis},\mu)=\sum_{h=1}^{H}\sum_{(s,a)\in\s\times\a}\frac{\frac{1}{KH}+\widehat{\dis}_{h}^{\pi}(s,a)}{\frac{1}{KH}+\mathbb{E}_{\pi\sim\mu}\big[\widehat{\dis}_{h}^{\pi}(s,a)\big]}.		
		\vspace{-1ex}			
		\]
		Here, $\widehat{\dis}_h^{\pi}(s,a)$ is computed via \Eqref{eq:hat-d-1-alg} for $h=1$, and \Eqref{eq:hat-d-h-alg} for $h\geq 2$.
		
		\STATE If $g(\pi^{(t)},\widehat{\dis},\mu^{(t)}_{\mathsf{b}})\leq 2HSA$ then \textsf{exit for-loop}. \label{line:stopping-alg:sub2} 
		Update 
		
		\vspace{-0.5ex}
		\[
		\mu^{(t+1)}_{\mathsf{b}} = \left(1-\alpha_t\right)\mu^{(t)}_{\mathsf{b}}+\alpha_t \mathbf{1}_{\pi^{(t)}}.
		\vspace{-1ex}
		\]
\ENDFOR	
\STATE  \textbf{Output:} the behavior policy $\pi^{\mathsf{b}}=\mathbb{E}_{\pi\sim \mu^{(t)}_{\mathsf{b}}}[\pi]$ and the associated weight $\widehat{\mu}_b={\mu}^{(t)}_{\mathsf{b}}$.
\end{algorithmic}
\end{algorithm}

Here,
			$\mathcal{M}_{\mathsf{b}}=(\cS \cup \{s_{\com}\},\cA, H,\widehat{\p}^{\com},r_{\mathsf{b}})$, where $s_{\com}$ is an augmented state and the reward function is chosen to be
				\begin{align} \label{eq:reward-function-mub}
r_{\mathsf{b},h}(s,a)=%
\begin{cases}
	\frac{1}{\frac{1}{KH}+\mathbb{E}_{\pi\sim\mu^{(t)}_{\mathsf{b}}}\big[\widehat{\dis}_{h}^{\pi}(s,a)\big]}\in[0,KH],\quad & \text{if }(s,a,h)\in\cS\times\cA\times[H];\\
0, & \text{if }(s,a,h)\in\{s_{\com}\}\times\cA\times[H].
\end{cases}
		\end{align}
		In addition, the augmented probability transition kernel $\widehat{\pp}^{\com}$ is constructed based on $\widehat{\pp}$ as follows:
		\begin{subequations}
			\label{eq:augmented-prob-kernel}
		\begin{align}
			\widehat{\pp}^{\com}_{h}(s^{'}\mymid s,a) &=    \begin{cases}
			\widehat{\pp}_{h}(s^{'}\mymid s,a),
			& 
			\text{if }  s^{'}\in \cS \\
			1 - \sum_{s^{'}\in \cS} \widehat{\pp}_{h}(s^{'}\mymid s,a), & \text{if }  s^{'} = s_{\com}
			\end{cases}
			&&\text{for all }(s,a,h)\in \cS\times \cA\times [H]; \\
			\widehat{\pp}^{\com}_{h}(s^{'}\mymid s_{\com},a) &= \ind(s^{'} = s_{\com})
			&&\text{for all }(a,h)\in  \cA\times [H] .
		\end{align}
		\end{subequations}
		
		It's evident that the augmented state behaves as an absorbing state, associated with zero immediate rewards.

\section{Relationship with existing 
metrics}\label{appendix:comparison}

In Section~\ref{sub:lindner}, we discuss the online IRL performance metric $D^L$ proposed in \citet{lindner2023active}, where we show that \RLE{} is still efficient under this metric, yet $D^L$ fails to distinguish certain pairs of reward mappings (or reward sets) that exhibit large distances under our metric.
In Section~\ref{appendix:m_b_rewards}, we briefly discuss the existing IRL performance metric $d^G_{V^\star}$ used in the simulator setting \citep{metelli2021provably,metelli2023theoretical}. In Section~\ref{appendix:proof for metric}, we provide a comparative analysis of our mapping-based metric in relation to Hausdorff-based metrics which is widely adopted by previous work \citep{metelli2021provably,metelli2023theoretical,lindner2023active}.
All proofs for this section can be found in Appendix~\ref{appendix:proof-comparison}.

\def \rt{{\cR_{\tau}}}
\def\wrt{{\cR_{\widehat{\tau}}}}
\def \mr{{\cM\cup r}}
\def \wmr{{\widehat{\cM}\cup \widehat{r}}}
\subsection{Discussion of existing metric for online IRL}\label{sub:lindner}
\citet{lindner2023active} considers a performance metric between two IRL problems $\tau=(\cM,\piE)$ and $\widehat{\tau}=(\widehat{\cM},\widehat{\pi}^{\sf E})$ instead of two reward mappings (or reward sets). Their metric $D^L$ is defined as follows:
\begin{align}
\label{eqn:metric_lindner}
D^L(\tau,\widehat{\tau})\defeq   \max&\Bigg\{\sup_{r\in \cR_{\tau}}\inf_{\widehat{r}\in \wrt}\sup_{\widehat{\pi}^\star\in \Pi^\star_\wmr}\max_{a}\abs{Q^{\pi^\star}_{1}(s_1,a;\mr)-Q_1^{\widehat{\pi}^\star}(s_1,a;\mr)},\\
&~~\sup_{\widehat{r}\in \wrt}\inf_{{r}\in \rt}\sup_{{\pi}^\star\in \Pi^\star_\mr}\max_{a}\abs{Q^{\pi^\star}_{1}(s_1,a;\mr)-Q_1^{\widehat{\pi}^\star}(s_1,a;\mr)}\Bigg\},
\end{align}
where the $\cR_{\tau},\wrt$ the set of all feasible rewards set for IRL problems $\tau,\widehat{\tau}$, respectively, $\pi^\star\in \Pi^\star_{\mr}$, $\widehat{\pi}^\star\in \Pi^\star_{\wmr}$, and $Q^{\pi}_1(\cdot|\mr)$ represent the $Q$-function induced by $\mr$ and $\pi$.
Since metric $D^{L}$ is defined between two IRL problems, we can't directly compare $D^L$ with our metrics. However, we can prove that our algorithm \RLE{} is capable of achieving the goal of attaining a small $D^L$ error.
\begin{proposition}[\RLE{} achieves small $D^L$ error]
\label{prop:can do lind}
Denote the ground truth IRL problem as $\tau=(\cM,\piE)$. Let $\widehat{\pp}$ and $\widehat{\pi}^{\sf E}$ 
 be the estimated expert policy and the estimated transition constructed by \RLE{} (Algorithm~\ref{alg:RLEF}), respectively. Define $\widehat{\tau}=\paren{\widehat{\cM},\widehat{\pi}^{\sf E}}$, where $\widehat{\cM}$ be the \MDPR{} equipped with the transition $\widehat{\pp}$. Under the same assumptions and choice of parameters as in Theorem~\ref{thm:online_main},
 for the online setting with both options in~\eqref{def_feedback}, for sufficiently small $\epsilon\le H^{-9}(SA)^{-6}$, with probability at least $1-\delta$, we can ensure $D^L(\tau,\widehat{\tau})\leq \epsilon$, as long as the total the number of episodes
\begin{align*}
   K + NH \ge \widetilde{\cO}\paren{\frac{H^4S^2A}{\epsilon^2}+\frac{H^2SA \eta}{\epsilon}}.
\end{align*}
Above, $\eta\defeq \Delta^{-1}\indic{{\rm option}~1}$, and $\tO(\cdot)$ hides ${\rm polylog}\paren{H,S,A,1/\delta}$ factors.

\end{proposition}
To achieve $D^L(\tau, \widehat{\tau})\leq \epsilon$,
the sample complexity\footnote{The original sample complexity given in \citet{lindner2023active} is $\widetilde{O}\paren{\frac{H^4SA}{\epsilon^2}}$. This is because, in the proof presented by \citet{lindner2023active}, they didn't employ the uniform convergence argument. However, the uniform convergence result is necessary for proving the sample complexity of \citet[Algorithm~1]{lindner2023active}. As a result, an $S$ factor was lost in the main term, and the burn-in term $\widetilde{\cO}\paren{\frac{H^2SA\eta}{\epsilon}}$ was neglected in their paper.}
 of \citet[Algorithm~1]{lindner2023active} is 
\begin{align*}
\widetilde{\cO}\paren{\frac{H^4S^2A}{\epsilon^2}+\frac{H^2SA\eta}{\epsilon}}
,
\end{align*}
which exactly matches the sample complexity of \RLE{}.

On the other hand, the following proposition shows that $D^L$ cannot distinguish certain cases that our $\DallTheta$ metric can.

\def \rht{r^{\widehat{\tau},\theta}}
\def \rt{r^{\tau, \theta}}
\cblue

\begin{proposition}[Example of problems distinguishable by $\DallTheta$ but not $D^L$]
\label{prop:lindner}
Let $\para=\pV\times\pA$. There exist $\tau=(\cM,\piE)$ and $\widehat{\tau}=(\widehat{\cM},\widehat{\piE})$ such that $D^L(\tau, \widehat{\tau})=0$ but $\DallTheta(\RR^{\tau}, \RR^{\widehat{\tau}})\geq 1$ where $\RR^{\tau}$ and $\RR^{\widehat{\tau}}$ are reward mappings induced by $\tau$ and $\widehat{\tau}$ respectively using definition~\eqref{def:rew-map}. 

In fact, we also have $D^L(\tau,\widehat{\tau})=0$ whenever $\piE=\widehat{\pi}^{\sf E}$ (but $\m$ and $\hat{\m}$ may differ arbitrarily, which may induce arbitrary difference between $\RR^{\tau}$ and $\RR^{\widehat{\tau}}$).
\end{proposition}

\subsection{Comparisons with existing metrics used in the simulator setting}\label{appendix:m_b_rewards}

\citet{metelli2023theoretical} consider the following metric
\begin{align}
 \label{eqn:metric-metelli}
    d^G_{V^\star}(r, \widehat{r}) =\max_{\widehat{\pi} \in \Pi^\star_{\m\cup\widehat{r}}} \max_{(h,s)\in [H]\times  \cS}\abs{V_h^\star(s; r) - V_h^{\widehat{\pi}}(s; r) }.
\end{align}
Notice the max over $(h,s)$ in~\eqref{eqn:metric-metelli}. In words, a small $d^G_{V^\star}(r, \widehat{r})$ requires $r$ and $\widehat{r}$ to induce similar value functions \emph{uniformly at all states}, which is achievable in their simulator setting and not achievable in standard offline/online settings where there may exist states that are not visitable at all by any policy in this particular \MDPR{}.

\subsection{Comparison with Hausdorff-based metrics}
\label{appendix:proof for metric}

Given a reward mapping $\RR:\pV\times \pA\mapsto \rew$, we say a reward set $\mc{R}\subset\rew$ is a feasible reward set induced by $\RR$, if $\mc{R}=\mathsf{image}(\RR)$.
For any given base metric $d$ between rewards, the Hausdorff (pre)metric $\DH$ which is given by
\[
\DH(\mathcal{R},\widehat{\mathcal{R}})\defeq
\max\big\{ \sup_{r\in \mathcal{R}}\inf_{\widehat{r}\in \widehat{\mathcal{R}}}d(r,\widehat{r}),\sup_{\widehat{r}\in \widehat{\mathcal{R}}}\inf_{r\in \mathcal{R}}d(r,\widehat{r}) \big\}.
\]
The works of \citet{metelli2021provably,metelli2023theoretical} consider finding an estimated feasible set $\widehat{\cR}$ that attains a small $\DH(\mathcal{R},\widehat{\mathcal{R}})$ using a certain base metric $d$.
\cblack

Different from our mapping-based metric (Definition~\ref{def:metric_rm}), the Hausdorff metric measures only the gap between the two sets $\mathcal{R}$ and $\widehat{\mathcal{R}}$, but cannot measure the gap between rewards for each parameter $(V, A)$ in a \emph{paired} fashion. Here we show that for any given base metric $d$, our mapping-based metric is stronger than the Hausdorff metric.
\begin{lemma}[$D^{\sf M}$ is stronger than $D^{\sf H}$]\label{lem:dm dh}
    Given an IRL problem $(\m,\pi^E)$ and a base metric $d:\rew\times \rew\mapsto \real_{\geq 0}$.
    We define the corresponding Hausdorff metric $\DH$  for any reward set pair $(\rl,\rl')$  by
    \begin{align*}
        \DH(\rl,{\rl'})\defeq 
        \max\set{ \sup_{r\in \rl}\inf_{{r}'\in {\rl'}}d(r,{r}'),\sup_{{r}'\in {\rl'}}\inf_{r\in \rl}d(r,{r}') },
    \end{align*}
    and the  mapping-based metric $\DM$ is defined for any reward mapping pair $(\RR,\RR')$  by
    \begin{align*}
        \DM(\RR,{\RR'})\defeq
        \sup_{V\in \pV,A\in \pA} d\paren{\RR(V,A),{\RR'}(V,A)},
    \end{align*}
   For any $(\RR, 
   \RR')$, let $\rl=\mathsf{image}(\RR)$ and $\rl'=\mathsf{image}(\RR')$, then we have
    \begin{align*}
     \DH(\rl,{\rl'})\leq \DM(\RR,{\RR'}).
    \end{align*}
\end{lemma}

\cblue
We then present the following lemma which demonstrates that for some $d$, $\DM$ is \emph{strictly} stronger than $\DH$.

\begin{lemma}[$D^{\sf M}$ is stronger than $D^{\sf H}$]
\label{lem:h_metric_bad}
There exists a base metric $d$ defined on rewards such that for any IRL problem $(\cM,\piE)$, there exists another IRL problem $(\widehat{\cM},\widehat{\pi}^{\sf E})$ such that $D^{\msf{H}}(\cR^\star,\widehat{\cR})=0$, but $D^{\msf{M}}(\sRR,
\hRR)\geq 1/2$, where $\DH$ and $D^{\sf M}$ are the Hausdorff metric and mapping-based metric induced by $d$, respectively;  $\cR^\star$ and $\widehat{\cR}$ are the feasible sets of $(\cM,\piE)$ and $(\widehat{\cM},\widehat{\pi}^{\sf E})$, respectively; $\sRR$ and $\hRR$ are the reward mappings induced by $(\cM,\piE)$ and $(\widehat{\cM},\widehat{\pi}^{\sf E})$, respectively.
    
\end{lemma}

 \cblack

\subsection{Proofs for Section~\ref{appendix:comparison}}
\label{appendix:proof-comparison}

\begin{proof}[Proof of Proposition~\ref{prop:can do lind}]
For any $\theta=(V,A)$, 
we first define the error $C^\theta_{h}(s,a)$ as follows:
    \begin{align}
        C^\theta_h(s,a)\defeq b^\theta_h(s,a)+\abs{A_{h}(s,a)\cdot\paren{\indic{a\in \supp\paren{\piE_h(\cdot|s)}}-\indic{a\in \supp\paren{\widehat{\pi}^{\sf E}_h(\cdot|s)}}}},
    \end{align}
    where $b^\theta_h(s,a)$ is defined in \Eqref{rlp:eq_8}.
    Let $\RR^{\tau}$ and $\RR^{\widehat{\tau}}$ be the ground truth reward mappings induced by $\tau$ and $\widehat{\tau}$. We consider the concentration event $\cE$ defined in Lemma~\ref{lem:concentration_2}.
   Conditioning on $\cE$, we next prove the following result:
    \begin{align}
       C^\theta_h(s,a) \geq \max\set{\abs{\rt_h(s,a)-\rht(s,a)}, \abs{\brac{\paren{\pp_h-\widehat{\pp}_h}V_{h+1}}(s,a)}},\label{eq:linp_claim_0}
    \end{align}
    holds for any $\theta\in \pV\times \pA$ and any $(h,s,a)\in [H]\times \cS\times \cA$,
    where $\rt=\RR^{\tau}(V,A)$, $\rht=\RR^{\widehat{\tau}}(V,A)$.

 By Lemma~\ref{lem:concentration_2}, under event $\cE$, we directly have 
\begin{align}
    C^\theta_h(s,a)\geq  b^\theta_h(s,a)\geq \abs{\brac{\paren{\pp_h-\widehat{\pp}_h}V_{h+1}}(s,a)},\label{eq:linp_claim_1}
\end{align}
holds for any $\theta\in \pV\times \pA$ and any $(h,s,a)\in [H]\times \cS\times \cA$. 
For the second part of \Eqref{eq:linp_claim_0}, by definition of $\rt$ and $\rht$, we have 
\begin{align}
    &\abs{\rt_h(s,a)-\rht(s,a)}
    =\Big|-A_{h}(s,a)\cdot\indic{a\notin \supp\paren{\piE_h(\cdot|s)}}+V_{h}(s)-\brac{\p_hV_{h+1}}(s,a)\notag\\
     &+A_{h}(s,a)\cdot\indic{a\notin \supp\paren{\widehat{\pi}^{\sf E}_h(\cdot|s)}}-V_h(s)+\brac{\widehat{\p}_hV_{h+1}}(s,a)\Big|\notag\\
     &\leq \abs{A_{h}(s,a)\cdot \paren{\indic{a\notin \supp\paren{\piE_h(\cdot|s)}}-\indic{a\notin \supp\paren{\widehat{\pi}^{\sf E}_h(\cdot|s)}}}}+\abs{\paren{\pp_h-\widehat{\pp}_h}V_{h+1}(s,a)}\notag\\
     &\leq \abs{A_{h}(s,a)\cdot \paren{\indic{a\notin \supp\paren{\piE_h(\cdot|s)}}-\indic{a\notin \supp\paren{\widehat{\pi}^{\sf E}_h(\cdot|s)}}}}+b^\theta_h(s,a)\notag\\
     &=C^\theta_h(s,a)\label{eq:linp_claim_2}
     ,
\end{align}
where the last line is by Lemma~\ref{lem:concentration_2}.
Combining \Eqref{eq:linp_claim_1} and \Eqref{eq:linp_claim_2}, we compelete the proof for \Eqref{eq:linp_claim_0}.
When \Eqref{eq:linp_claim_0} holds, by \citet[Lemma~20]{lindner2023active}, there exists a policy $\pi^{L}$ (see the proof of \citet[Lemma~20]{lindner2023active} for the construction of $\pi^L$) such that
\begin{align}
    D^L(\tau,\widehat{\tau})&\lesssim \sup_{\theta\in \pV\times \pA}\sum_{h\in [H]}\sum_{s,a}d^{\pi^L}_h(s,a)\cdot C^\theta_h(s,a)\notag\\
    &=\sup_{\theta\in \pV\times \pA}\sum_{h\in [H]}\sum_{s,a}d^{\pi^L}_h(s,a)\cdot \abs{A_{h}(s,a)\cdot \paren{\indic{a\notin \supp\paren{\piE_h(\cdot|s)}}-\indic{a\notin \supp\paren{\widehat{\pi}^{\sf E}_h(\cdot|s)}}}}\notag\\
    &+\sup_{\theta\in \pV\times \pA}\sum_{h\in [H]}\sum_{s,a}d^{\pi^L}_h(s,a)\cdot b^\theta_h(s,a),\label{eq:linp_end_1}
\end{align}
where $\set{d^{\pi^L}_h(\cdot)}_{h\in [H]}$ is the state-action visitation distribution induced by $\pp$ and $\pi^L$. 
Following the proof of Theorem~\ref{thm:online_main}, we can prove that under $\cE$
\begin{align}
    &\sup_{\theta\in \pV\times \pA}\sum_{h\in [H]}\sum_{s,a}d^{\pi^L}_h(s,a)\cdot \abs{A_{h}(s,a)\cdot \paren{\indic{a\notin \supp\paren{\piE_h(\cdot|s)}}-\indic{a\notin \supp\paren{\widehat{\pi}^{\sf E}_h(\cdot|s)}}}}\lesssim \epsilon,\notag\\
    &\sup_{\theta\in \pV\times \pA}\sum_{h\in [H]}\sum_{s,a}d^{\pi^L}_h(s,a)\cdot b^\theta_h(s,a)\lesssim \epsilon,\label{eq:linp_end_2}
\end{align}
hold, provided that
\begin{align*}
K\geq\widetilde{\cO}\paren{\frac{H^4S^2A}{\epsilon^2}+\frac{H^2SA\eta}{\epsilon}}, \quad KH \geq N \geq \widetilde{\cO}\paren{{\sqrt{H^9S^7A^7K}}}.
\end{align*}
Combining \Eqref{eq:linp_end_1} and \Eqref{eq:linp_end_2}, we complete the proof.

\end{proof}

\begin{proof}[Proof of Proposition~\ref{prop:lindner}]
To begin with,
we prove a stronger result: 
  for any $\tau=(\cM,\piE)$ and any $\widehat{\tau}=(\widehat{\cM},\widehat{\pi}^{\sf E})$, if $\piE=\widehat{\pi}^{\sf E}$, then $D^L(\tau,\widehat{\tau})=0$.
  
    Let $\RR^{\tau}$ and $\RR^{\widehat{\tau}}$ be the reward mappings induced by $\tau$ and $\widehat{\tau}$, respectively. For any $\theta=(V,A)$, we define $\rt=\RR^{\tau}(V,A)$ and $\rht=\RR^{\widehat{\tau}}(V,A)$.
    By the construction of reward mappings and the definition of optimal policies, we have that
 $\pi\in \Pi^\star_{\cM\cup r}$ is equivalent to 
 \begin{align}
    A_h(s,a)\cdot \indic{\piE_h(a|s)=0}=0, \qquad \forall (h,s,a)~\textrm{s.t.}~\pi_h(a|s)\neq 0.
\end{align}
Similarly, $\pi\in \Pi^\star_{\widehat{\cM}\cup \rht}$ is equivalent to 
\begin{align}
    A_h(s,a)\cdot \indic{\widehat{\pi}^{\sf E}_h(a|s)=0}= A_h(s,a)\cdot \indic{\piE_h(a|s)=0}=0, \qquad \forall (h,s,a)~\textrm{s.t.}~\pi_h(a|s)\neq 0.
\end{align}
Hence, we can conclude that $\Pi^\star_{{\cM}\cup {r}^{\theta}}=\Pi^\star_{\widehat{\cM}\cup \rht}$.
Notice that $\rt=\set{\rt\mid \theta=(V,A)}$ and $\rt=\set{\rht\mid \theta=(V,A)}$, we then have
\begin{align}\label{eq:lineder}
&\sup_{r\in \cR_{\tau}}\inf_{\widehat{r}\in \wrt}\sup_{\widehat{\pi}^\star\in \Pi^\star_\wmr}\max_{a}\abs{Q^{\pi^\star}_{1}(s_1,a;\mr)-Q_1^{\widehat{\pi}^\star}(s_1,a;\mr)}\notag\\
&=
\sup_{\theta\in \para}\inf_{\theta'\in \para}\sup_{\widehat{\pi}^\star\in \Pi^\star_{\widehat{\cM}\cup {r}^{\widehat{\tau},\theta'}}}\max_{a}\abs{Q^{\pi^\star}_{1}(s_1,a;\cM\cup \rt)-Q_1^{\widehat{\pi}^\star}(s_1,a;\cM\cup \rt)}\notag\\
&=\sup_{\theta\in \para}\sup_{\widehat{\pi}^\star\in \Pi^\star_{\widehat{\cM}\cup \rht}}\max_{a}\abs{Q^{\pi^\star}_{1}(s_1,a;\cM\cup \rt)-Q_1^{\widehat{\pi}^\star}(s_1,a;\cM\cup \rt)}=0,
\end{align}
where the last line is due to $\Pi^\star_{{\cM}\cup {r}^{\theta}}=\Pi^\star_{\widehat{\cM}\cup \rht}$.
Follow the same proof of \Eqref{eq:lineder}, we have
\begin{align}\label{eq:lineder_1}
    \sup_{\widehat{r}\in \wrt}\inf_{{r}\in R_\tau}\sup_{{\pi}^\star\in \Pi^\star_\mr}\max_{a}\abs{Q^{\pi^\star}_{1}(s_1,a;\mr)-Q_1^{\widehat{\pi}^\star}(s_1,a;\mr)}=0.
\end{align}
Combining \Eqref{eq:lineder} and \Eqref{eq:lineder_1}, we conclude that $D^L(\tau,\widehat{\tau})=0$.

We then construct $\tau$ and $\widehat{\tau}$, respectively. We set $\cS=\set{1,2,\ldots, S}, \cA=\set{1,2,\ldots,A}$, and $H\geq 2$. Design the transitions $\pp$ and $\widehat{\pp}$ as follows:  
\begin{align}
    \pp_h(1|s,a)=1, \qquad \widehat{\pp}_h(2|s,a)=1\qquad \forall (h,s,a)\in [H]\times \cS\times \cA.
\end{align}
Let $\cM=(\cS, \cA, H, \pp)$, $\widehat{\cM}=(\cS, \cA, H, \widehat{\pp})$.
Define $\piE$ and $(\bar{V},\bar{A})\in \pV\times \pA$ by
\begin{align}
    &\piE_h(1|s)=1, \qquad\forall (h,s)\in [H]\times \cS\times \cA\notag\\
    &\bar{V}_h(s)=\indic{s=1},\qquad \bar{A}\equiv \mathbf{0}, \qquad\forall h\in [H].
\end{align}
Set $\tau=(\cM, \piE)$ and $\widehat{\tau}=(\widehat{\cM}, \piE)$. By the result we proved at first, we have
\begin{align}
    D^L(\tau, \widehat{
    \tau
    })=0.
\end{align}
Let $1$ be the initial state i.e., $\pp(s_1=1)=1$.
By definition of $\DpiTheta$, we obtain that 
\begin{align}
    \DallTheta(\RR^\tau, \RR^{\widehat{\tau}})\geq d^{\pi}\paren{\RR^\tau\paren{\bar{V},\bar{A}}, \RR^{\widehat{\tau}}\paren{\bar{V},\bar{A}}}
    \geq \EE_{\pi}\brac{\abs{V^\pi_2\paren{s;\RR^\tau\paren{\bar{V},\bar{A}}}-V^\pi_2\paren{s;\RR^{\widehat{\tau}}\paren{\bar{V},\bar{A}}}}}
    =\abs{V_2(1)-V_2(2)}=1,
\end{align}
where the send last equality is due to the construction of $\pp$ and $\widehat{\pp}$.
\end{proof}

\begin{proof}[Proof of Lemma~\ref{lem:dm dh}]
    Since $\rl$ and $\rl'$ are induced by $\RR$ and $\RR'$, then for any $r\in \rl$ and $r'\in \rl$, there exist $V$, $V'\in \pV$, $A$, $A\in \pA$ such that
    \[
    r=\RR(V,A), \qquad r'=\RR'(V',A').
    \]
    Then, we have
    \begin{align*}
         \sup_{r\in \rl}\inf_{{r'}\in {\rl'}}d(r,{r'})&
         =\sup_{V\in \pV,A\in \pA}\inf_{V'\in \pV,A'\in \pA}d(\RR(V,A),\RR'(V',A'))\\
         &\leq \sup_{V\in \pV,A\in \pA}d(\RR(V,A),\RR'(V,A))=\DM(\RR,{\RR'}).
    \end{align*}
    Similarly, we obtain 
    \[
    \sup_{{r}'\in {\rl'}}\inf_{r\in \rl}d(r,{r}')\leq \DM(\RR,{\RR'}).
    \]
    Hence, we conclude that 
    \[
     \DH(\rl,{\rl'})\leq \DM(\RR,{\RR'}).
    \]
\end{proof}

\begin{proof}[Proof of Lemma~\ref{lem:h_metric_bad}]
Fix a $(\bar{s},\bar{a})\in \cS\times \cA$, we define metric $d$ by
\begin{align}
    d(r,r')\defeq \abs{r_1(\bar{s},\bar{a})-r'_1(\bar{s},\bar{a})}.
\end{align}
Give an IRL  problem $(\cM,\piE)$,
let $\pp$ be the transition dynamics of $(\cM,\piE)$. 
Let $s^\star\defeq\argmin_{s\in \cS} \pp_1(s|\bar{s},\bar{a})$. By 
the Pigeonhole Principle, we have $\pp_1(s^\star|\bar{s},\bar{a})\leq 1/S\leq 1/2$. We construct transition $\pp'$ by
\begin{align}
    \pp'_1(s^\star|\bar{s},\bar{a})=1.
\end{align}
Let $\widehat{\cM}=(\cS,\cA,H,\pp')$, $\widehat{\pi}^{\sf E}=\piE$, and $\widehat{\RR}$ be the reward mapping induced by $(\widehat{\cM},\widehat{\pi}^{\sf E})$. 
For any $(V,A)\in \bar{\cV}\times \bar{\cA}$, we define $({V}',A')\in \bar{\cV}\times \bar{\cA}$ by 
\begin{align}
    \begin{cases}
        V'_2(s^\star)=\brac{\pp_1 V_2}(\bar{s}, \bar{a})&{}\\
       V'_h(s)=V_h(s)&(h,s)\neq (2, s^\star), 
    \end{cases}
    ~~\text{and}~~ A'=A.
\end{align}
Then we have
\begin{align}\label{eq:lem_c_5_1}
d\paren{\RR(V,A), \hRR(V',A')}&=\abs{\brac{\RR(V,A)}_1(\bar{s}, \bar{a})-\brac{\widehat{\RR}(V,A)}_1(\bar{s}, \bar{a})}\\
&=\Big|-A_1(\bar{s}, \bar{a})\cdot \indic{\bar{a}\in \supp\paren{\piE_1(\cdot|\bar{s})}}+V_1(s)-\brac{\pp_1 V_2}(\bar{s},\bar{a})\notag\\
&-\set{-A'_1(\bar{s}, \bar{a})\cdot \indic{\bar{a}\in \supp\paren{\widehat{\pi}^{\sf E}_1(\cdot|\bar{s})}}+V'_1(s)-\brac{\pp'_1 V'_2}(\bar{s},\bar{a})}\Big|\notag\\
&=\abs{\brac{\pp'_1 V'_2}(\bar{s},\bar{a})-\brac{\pp_1 V_2}(\bar{s},\bar{a})}\notag\\
&=\abs{V'_2(s^\star)-\brac{\pp_1 V_2}(\bar{s},\bar{a})}=0
\end{align}
On the other hand, for any $(V',A')\in\bar{\cV}\times \bar{\cA}$, we set $(V,A)\in\bar{\cV}\times \bar{\cA}$ by
\begin{align}
    \begin{cases}
        V_2(s)=V_2'(s^\star), & s\in \cS,\\
        V_h(s)=V'_h(s)&h\neq 2, \\
    \end{cases} 
\end{align}
which implies that 
\begin{align}
\brac{\pp_1 V_2}(\bar{s},\bar{a})=V'_2(s^\star)=\brac{\pp'_1 V'_2}(\bar{s},\bar{a}).
\end{align}

Hence, we have
\begin{align}\label{eq:lem_c_5_2}
d\paren{\RR(V,A), \hRR(V',A')}&=\abs{\brac{\RR(V,A)}_1(\bar{s}, \bar{a})-\brac{\hRR(V,A)}_1(\bar{s}, \bar{a})}\\
&=\Big|-A_1(\bar{s}, \bar{a})\cdot \indic{\bar{a}\in \supp\paren{\piE_1(\cdot|\bar{s})}}+V_1(s)-\brac{\pp_1 V_2}(\bar{s},\bar{a})\notag\\
&-\set{-A'_1(\bar{s}, \bar{a})\cdot \indic{\bar{a}\in \supp\paren{\widehat{\pi}^{\sf E}_1(\cdot|\bar{s})}}+V'_1(s)-\brac{\pp'_1 V'_2}(\bar{s},\bar{a})}\Big|\notag\\
&=\abs{\brac{\pp'_1 V'_2}(\bar{s},\bar{a})-\brac{\pp_1 V_2}(\bar{s},\bar{a})}=0
\end{align}
Combining \Eqref{eq:lem_c_5_1} and \Eqref{eq:lem_c_5_2}, we have
$\DH(\cR,\widehat{\cR})=0$. 

Next, we lower bound $\DM(\RR,\widehat{\RR})$. First, we define a parameter $(\widetilde{V},\widetilde{A})\in \bar{\cV}\times \bar{\cA}$ as follows:
\begin{align}
\begin{cases}
 \widetilde{V}_2(s^\star)=H-1,&{}\\
 \widetilde{V}_h(s)=0, &(h,s)\neq (2,s^\star),
\end{cases}
\qquad \widetilde{A}\equiv \mathbf{0}.
\end{align}
Then we have 
\begin{align}
    \DM(\RR,\widehat{\RR})&\geq d\paren{\RR(\widetilde{V}, \widetilde{A}),\hRR(\widetilde{V}, \widetilde{A})}=\abs{\brac{\pp'_1 \widetilde{V}_2}(\bar{s},\bar{a})-\brac{\pp_1 \widetilde{V}_2}(\bar{s},\bar{a})}\notag\\
    &=\abs{\paren{H-1}\paren{\pp_1(s^\star|\bar{s},\bar{a})-1}}\geq \frac{H-1}{2}\geq \frac{1}{2},
\end{align}
where the last line is due to $\pp_1(s^\star|\bar{s},\bar{a})\leq 1/2$.

\end{proof}

\subsection{Proof of Proposition~\ref{prop:mon_use_reward}}
\label{app:proof-mon_use_reward}

\begin{proof}[Proof of Proposition~\ref{prop:mon_use_reward}]
    Since $\widehat{\pi}$ is an $\epsilon$-optiaml policy in $\cM\cup{\widehat{r}}$, we have
    \begin{align}\label{eq_prop_mon_1}
        \epsilon'+V^{\widehat{\pi}}(s_1;\widehat{r})\ge V^{{\pi}}(s_1;\widehat{r}).
    \end{align}
    In the same way, ${\pi}$ is an $\bar{\epsilon}$-optiaml policy in $\cM\cup{{r}}$, and therefore, we obtain that
    \begin{align}\label{eq_prop_mon_2}
        \bar{\epsilon}+V^{{\pi}}(s_1;{r})\ge V^{{\widehat{\pi}}}(s_1;{r}).
    \end{align}
    And by $\widehat{r}_h(s,a)\le r_h(s,a)$ for all $(h,s,a)\in [H]\times\cS\times \cA$, we have 
    \begin{align}\label{eq_prop_mon_3}
        V^{{\pi}}(s_1;{r})\geq V^{{\pi}}(s_1;\widehat{r}),\qquad  V^{\widehat{\pi}}(s_1;{r})\geq V^{\widehat{\pi}}(s_1;\widehat{r}).
    \end{align}
    Combining \Eqref{eq_prop_mon_1}, \Eqref{eq_prop_mon_2} and \Eqref{eq_prop_mon_3}, we conclude that
    \begin{align}\label{eq_prop_mon_4}
        \epsilon'+\bar{\epsilon}+V^{{\pi}}(s_1;{r})\geq \epsilon'+V^{\widehat{\pi}}(s_1;{r})\geq \epsilon'+V^{\widehat{\pi}}(s_1;\widehat{r})\geq V^{{\pi}}(s_1;\widehat{r}).
    \end{align}
    Hence, we have
    \begin{align}
        |V^{{\pi}}(s_1;{r})-V^{\widehat{\pi}}(s_1;{r})|&\leq |\epsilon'+\bar{\epsilon}+V^{{\pi}}(s_1;{r})-\paren{\epsilon'+V^{\widehat{\pi}}(s_1;{r})}|+\bar{\epsilon}\notag\\
        &\leq |\epsilon'+\bar{\epsilon}+V^{{\pi}}(s_1;{r})-V^{{\pi}}(s_1;\widehat{r})|+\bar{\epsilon}\notag\\
        &\leq 2\bar{\epsilon}+\epsilon'+|V^{{\pi}}(s_1;{r})-V^{{\pi}}(s_1;\widehat{r})|\leq \epsilon+\epsilon'+2\bar{\epsilon},
    \end{align}
    where the first and last line is by triangle inequality and the second is by \Eqref{eq_prop_mon_4}.
\end{proof}

\cblack

\def \wellp{\eta}
\def \dt{\mrm{dim}\paren{\Theta}}
\def \hsa{[H]\times \cS\times \cA}

\section{Proofs for Section ~\ref{sec:offline learning}}\label{appendix:proof for offline learning}

\subsection{Some lemmas}

\begin{lemma}[Concentration event]\label{lem:concentration_1}
    Under the assumption of Theorem~\ref{thm:offline_main}, there exists absolute constant $C_1,~C_2$ such that the concentration event $\mc{E}$ holds with probability at least $1-\delta$, where
    \begin{align}
          \mc{E} \defeq \Bigg\{
          \text{(i):} &~  \abs{\brac{(\p_h-\widehat{\p}_h) V_{h+1}}(s, a)}\leq 
          b^\theta_h(s,a)~~~\forall~\theta=(V,A)\in  \para,~ (h,s,a)\in \hsa,
       \\
       \text{(ii):} &~ \frac{1}{{N}_{h}(s, a)\vee 1} \le \frac{C_1\iota}{Kd^{\pib}_h(s, a)} ~~~\forall 
      (h,s,a)\in[H]\times\mc{S}\times\mc{A},
        \\\text{(iii):} &~ {N}^e_{h}(s,a)\geq 1 ~~~\forall (s,a)\in \cS\times \cA~\textrm{s.t.}~ d_h^{\pib}(s,a)\geq \frac{ C_2\wellp\iota}{ K}, a\in \supp\paren{\piE_h(\cdot|s)}
         \Bigg\},
    \end{align}
    where
    $b_h(s,a)$ is defined in \Eqref{rlp:eq_8}, $C^\star$ is specified in Definition \ref{ass:off_1}, and ${N}^e_{h}(s,a),\eta$ are given by
    \[
    {N}^e_{h}(s,a)\defeq
    \begin{cases}
        \sum_{(s_h,a_h, e_h)\in \cD}\indic{(s_h, e_h)=(s,a)}\qquad &\textrm{in option $1$},\\
        N^b_h(s,a)\qquad &\textrm{in option $2$},
    \end{cases}
  \qquad
   \wellp\defeq
   \begin{cases}
       \frac{1}{\wellpi}\qquad &\textrm{in option $1$},\\
        1\qquad &\textrm{in option $2$.}
   \end{cases}
    \]

\end{lemma}
\begin{proof}
When ${N}_{h}(s, a)=0$, then $\widehat{\p}_{h}(\cdot|s, a)=0$, as a result, claim (i) holds trivially.
We then consider the case where ${N}_{h}(s, a)\geq 1$. 
For any $h\in [H]$, we define $\mathcal{N}_{\epsilon,h}$ as an $\epsilon/H$-net with respect to $\|\cdot\|_{\infty}$ norm for $\pV^{\Theta}_h$. By definition of $  \mathcal{N}(\Theta; \eps/H)$, we have
\[
\log\abs{\mathcal{N}_{\epsilon,h}}\leq  \log\mathcal{N}(\Theta; \eps/H).
\]

For fixed $\widetilde{V}_{h+1}\in \cN_{\eps,h+1}$, $(h,s,a)\in [H]\times\cS\times \cA$, by the empirical Bernstein inequality \citep[Theorem 4]{maurer2009empirical}, there exists some absolute constant $c>0$ such that

 	\begin{align*}
 		\abs{\brac{(\p_h-\widehat{\p}_h) \widetilde{V}_{h+1}}(s, a)} &\leq  \sqrt{\frac{c}{{N}_h^b(s,a)\vee 1}\brac{\widehat{\v}_h\widetilde{V}_{h+1}}(s,a)\log\frac{3HSA\cdot\abs{\cN_{\eps,h+1}}}{\delta}}\notag\\
   &+\frac{cH}{{N}_h^b(s,a)\vee 1}\log\frac{3HSA\cdot\abs{\mathcal{N}_{\epsilon,h+1}}}{\delta}\\
   &\lesssim  \sqrt{\frac{c\cn\iota}{{N}_h^b(s,a)\vee 1}\brac{\widehat{\v}_h\widetilde{V}_{h+1}}(s,a)}+\frac{cH\cn\iota}{{N}_h^b(s,a)\vee 1}
 \end{align*}
	with probability at least $1-\delta/(3HSA\abs{\cN_{\epsilon,h}})$. Here $\lesssim$ hides absolute constants. 
	Taking the union bound over all $\widetilde{V}_{h+1}\in\mathcal{N}_{\epsilon,h+1}$ and $(h,s,a)\in \hsa$,
	we know that with probability at least $1-\delta/3$, 
	\[
		\abs{\brac{(\p_h-\widehat{\p}_h) \widetilde{V}_{h+1}}(s, a)} \lesssim \sqrt{\frac{c\cn\iota}{{N}_h^b(s,a)\vee 1}\brac{\widehat{\v}_h\widetilde{V}_{h+1}}(s,a)}+\frac{cH\cn\iota}{{N}_h^b(s,a)\vee 1}
	\]
\def \tV{\widetilde{V}}	
	holds simultaneously for all $\widetilde{V}\in\mathcal{N}_{\epsilon,h}$ and $(h,s,a)\in \hsa$.

   For any $(V,A)\in \para$ and $h\in [H]$, there exists a $\tV_{h}\in \pV^{\Theta}_h $ such that $\Vert V_h-\widetilde{V}_h\Vert_{\infty}\leq \epsilon/H$. Denote $(\tV_1,\ldots,\tV_H)$ as $\tV$.
   By applying the triangle inequality, we deduce that
	\def \lab{eq:good_event_online}
\begin{align}
 & \abs{\brac{(\p_h-\widehat{\p}_h) V_{h+1}}(s, a)} \leq\left|\brac{\big(\widehat{\p}_{h}-\p_{h}\big){\tV_{h+1}}}(s,a)\right|+2\big\Vert\widetilde{V}-V\big\Vert_{\infty}\notag\\
 & \lesssim \sqrt{\frac{c\cn\iota}{{N}_h^b(s,a)\vee 1}\brac{\widehat{\v}_h\widetilde{V}_{h+1}}(s,a)}+\frac{cH\cn\iota}{{N}_h^b(s,a)\vee 1}+\frac{\epsilon}{H}\notag\\
 &\leq \sqrt{\frac{c\cn\iota}{{N}_h^b(s,a)\vee 1}\brac{\widehat{\v}_h{V}_{h+1}}(s,a)}+
 \sqrt{\frac{c\cn\iota}{{N}_h^b(s,a)\vee 1}\brac{\widehat{\v}_h\paren{\tV_{h+1}-{V}_{h+1}}}(s,a)}\notag\\
 &+
 \frac{cH\cn\iota}{{N}_h^b(s,a)\vee 1}+\frac{\epsilon}{H}\notag\\
 &\leq\sqrt{\frac{c\cn\iota}{{N}_h^b(s,a)\vee 1}\brac{\widehat{\v}_h{V}_{h+1}}(s,a)}+
 \sqrt{\frac{c\cn\iota\epsilon^2}{H^2\cdot{N}_h^b(s,a)\vee 1}}+
 \frac{cH\cn\iota}{{N}_h^b(s,a)\vee 1}+\frac{\epsilon}{H}\notag\\
 &= \sqrt{\frac{c\cn\iota}{{N}_h^b(s,a)\vee 1}\brac{\widehat{\v}_h{V}_{h+1}}(s,a)}+
 \frac{cH\cn\iota}{{N}_h^b(s,a)\vee 1}+\frac{\epsilon}{H}\paren{1+\sqrt{\frac{c\cn\iota}{{N}_h^b(s,a)\vee 1}}}
 \label{eq:good_event_online_1}
\end{align}
holds with probability at least $1-\delta/3$ for all $\theta=(V,A)\in \para$ and $(h,s,a)\in \hsa$. Here, 
the second inequality is by $\sqrt{a+b}\leq \sqrt{a}+\sqrt{b}$ and 
the last inequality in is {by $\brac{\widehat{\v}_h\paren{\tV_{h+1}-{V}_{h+1}}}(s,a)\leq \frac{\epsilon^2}{H^2}$.} On the other hand, by $\abs{V_h}_\infty\leq H-h+1$, 
we obtain that
\begin{align}
\abs{\brac{\paren{\p_h-\widehat{\p}_h}V_{h+1}}(s,a)}\leq 2(H-h+1)\leq 2H,\label{\lab_2}
\end{align}
for all $V\in \para$ and $(h,s,a)\in \hsa$.
Recall that, $b^\theta_h(s,a)$ is given by
\begin{align}
b^\theta_h(s,a)&=C\cdot\min\Bigg\{\sqrt{\frac{\cn\iota}{{N}_h^b(s,a)\vee 1}\brac{\widehat{\v}_h{V}_{h+1}}(s,a)}+
 \frac{H\cn\iota}{{N}_h^b(s,a)\vee 1}\notag\\
& +\frac{\epsilon}{H}\paren{1+\sqrt{\frac{\cn\iota}{{N}_h^b(s,a)\vee 1}}}, H\Bigg\},\label{eq:def_b_theta}
\end{align}
for some absolute constant $C$.
Combining \Eqref{\lab_1} and \Eqref{\lab_2}, it turns out that Claims (ii) holds.\\

For claim (ii), notice that ${N}_h(s, a)\sim \Bin(K, d^{\pib}_h(s, a))$. Applying Lemma~\ref{lemma:binomial-concentration} yields that
\begin{align*}
    \frac{1}{N_h(s, a)\vee 1} \le \frac{8}{K \cdot d^{\pib}_h(s_h, a_h^E)}\cdot \log(\frac{3HSA}{\delta}) \le \frac{C_1\iota}{Kd^{\pib}_h(s,a)}
\end{align*} for some absolute constant $C_1$,
with probability at least $1-\delta/(3HSA)$. Taking the union bound yields claim (ii) over all $(h,s, a)$ with probability at least $1-\delta/3$.\\

For claim (iii), in option 2, for any $(h,s,a)\in [H]\times\cS\times \cA$ such that $a\in \supp\paren{\piE_h\paren{\cdot|s}}$ and $d_h^{\pib}(s,a)\geq \frac{ C_2\wellp\iota}{ K}$, we have ${N}^e_{h}(s, a)\sim \Bin\paren{K, d^{\pib}_h(s, a)\cdot\piE_h(a|s)}$.
By direct computing, we obtain that
\begin{align*}
\p\brac{{N}^e_{h}(s, a)=0}&=(1-d^{\pib}_h(s, a)\cdot\piE_h(a|s))^{K}\leq \paren{1-\wellpi\cdot d^{\pib}_h(s, a)}^{K}\\
&= \brac{1-\paren{\frac{\delta}{3HSA}}^{1/K}+\paren{\frac{\delta}{3HSA}}^{1/K}-\wellpi\cdot d^{\pib}_h(s, a)}^{K}\\
&\leq \brac{\paren{\frac{\delta}{3HSA}}^{1/K}+\underbrace{1-\paren{\frac{\delta}{3HSA}}^{1/K}-\wellpi\cdot d^{\pib}_h(s, a)}_{\leq 0}}^{K}\\ 
&\leq \paren{\frac{\delta}{3HSA}}^{1/K\cdot K}=\frac{\delta}{3HSA},
\end{align*}

where the second line follows from the well-posedness condition: $\piE(a|s)\geq \wellpi$ and the last inequality is valid since 
\begin{align*}
    1-\paren{\frac{\delta}{3HSA}}^{1/K}=1-\exp(-\frac{1}{K}\log\frac{\delta}{3HSA})\leq -\frac{\widetilde{C}_2}{K}\log\frac{\delta}{3HSA}\leq \frac{{C}_2\iota}{K}\leq \wellpi\cdot d^{\pib}_h(s, a),
\end{align*}
where $\widetilde{C}_2$ and $C_2$ are absolute constants and the last inequality comes from $d_h^{\pib}({s,a})\geq  \frac{ C_2\wellp\iota}{K}=\frac{ C_2\iota}{\wellpi\cdot K}$. Hence, it holds that
\begin{align*}
    {N}^e_h(s,a)\geq 1,
\end{align*}
with probability at least $1-\delta/(3HSA)$.
Taking the union bound over all $(h,s,a)\in \hsa$ yields that
\[
{N}^e_h(s,a)\geq 1
\]
holds with probability at least $1-\delta/3$
for all $(s,a)\in \cS\times \cA~\textrm{s.t.}~d_h^{\pib}(s,a)\geq \frac{ C_2\wellp\iota}{ K}, a\in \supp\paren{\piE_h(\cdot|s)}$,
which implies that claim (iii) holds.

In option $1$, notice that $\pp\brac{{N}^e_h(s,a)=0}=\paren{1-d^{\pib(s,a)}_h(s,a)}^K$, with a similar argument, we can prove the claim (iii) in option $1$.

Further, we can conclude that the concentration event $\cE$ holds with probability at least $1-\delta$.
\end{proof}

\begin{lemma}[Performance decomposition for \RLE]\label{lem:offine_performance_decomposition}
 For any $\theta=(V,A)\in \para$, let $r^\theta=\sRR(V,A)$ and $\widehat{r}^\theta=\hRR(V,A)$, where $\sRR$ is the ground truth reward mapping and $\hRR$ is the estimated reward mapping outputted by \RLP{}. On the event defined in Lemma~\ref{lem:concentration_1}, for any $\theta\in \para$ and any $h\in [H]$, we have
  \begin{align*}
      \dpival\paren{r^\theta,\widehat{r}^\theta}\lesssim \frac{C^\star H^2S\eta\iota}{K}+\sum_{h\in[H]}\sum_{(s,a)\in \s\times\a}d^{\pival}_{h}(s,a)b^\theta_h(s,a),
  \end{align*}
  where $C^\star$ is defined in Assumption~\ref{ass:off_1} and $\eta$ is specified in Lemma~\ref{lem:concentration_1}
\end{lemma}
\begin{proof}
Fix a tuple $(h, s, a)\in \hsa$.
When $a\notin\supp\paren{\piE_h(\cdot|s)}$, by definition of ${N}^e_h(s,a)$, we have ${N}^e_h(s,a)= 0$  By construction of $\widehat{\pi}^{\sf E}_h(a|s)$ in Algorithm~\ref{alg:RLPfull}, we deduce that $\widehat{\pi}^{\sf E}_h(a|s)=0$, and therefore 
\[
\abs{\indic{a\notin \supp{\widehat{\pi}^{\sf E}_h(\cdot|s)}}-\indic{a\notin\supp\paren{\piE_h(\cdot|s)}}}= \abs{0-0}=0.
\]
\def \oiat{\frac{C_2\wellp\iota}{K}}
When $a\in\supp\paren{\piE_h(\cdot|s)}$ and $d^{\pib}_h(s,a)< \frac{C_2\wellp\iota}{K}$, then 
\[
\abs{\indic{a\notin \supp\paren{\widehat{\pi}^{\sf E}_h(\cdot|s)}}-\indic{a\notin\supp\paren{\piE_h(\cdot|s)}}}\leq 2.
\]

If $a\in\supp\paren{\piE_h(\cdot|s)}$ and $d^{\pib}_h(s,a)\geq \oiat$, then by concentration event $\mc{E}$ (iii), ${N}^e_h(s,a)\geq 1$ which implies that $\widehat{\pi}^{\msf{E}}_h(a|s)>0$. Hence, we obtain that
\[
\abs{\indic{a\notin \supp\paren{\widehat{\pi}^{\sf E}_h(\cdot|s)}}-\indic{a\notin\supp\paren{\piE_h(\cdot|s)}}}= \abs{1-1}=0.
\]
Thus we can conclude that
\begin{align}   
\abs{\indic{a\notin \supp\paren{\widehat{\pi}^{\sf E}_h(\cdot|s)}}-\indic{a\notin\supp\paren{\piE_h(\cdot|s)}}}\leq 2\cdot \indic{d^{\pib}_h(s,a)<\oiat, a\in \supp\paren{\piE_h(\cdot|s)}}.\label{eq:offline_decomposition_indic}
\end{align}
We then bound the $\abs{\brac{r^\theta_{h}-\widehat{r}^\theta_{h}}(s,a)}$ for all $(h,s,a)\in \hsa$:
\begin{align}
    &\abs{\brac{r^\theta_{h}-\widehat{r}^\theta_{h}}(s,a)}=\Big|-A_{h}(s,a)\cdot\indic{a\notin \supp\paren{\piE_h(\cdot|s)}}+V_{h}(s)-\brac{\p_hV_{h+1}}(s,a)\notag\\
    &~+A_{h}(s,a)\cdot\indic{a\notin \supp\paren{\widehat{\pi}^{\sf E}_h(\cdot|s)}}-V_h(s)+\brac{\widehat{\p}_hV_{h+1}}(s,a)+b^\theta_h(s,a)\Big|\notag\\
    &\leq A_{h}(s,a)\cdot\abs{\indic{a\notin \supp\paren{\widehat{\pi}^{\sf E}_h(\cdot|s)}}-\indic{a\notin\supp\paren{\piE_h(\cdot|s)}}}+\abs{\brac{(\p_h-\widehat{\p}_h)V_{h+1}}(s,a)}+b^\theta_h(s,a)\notag\\
    &\leq 2H\cdot \indic{d^{\pib}_h(s,a)<\oiat, a\in \supp\paren{\piE_h(\cdot|s)}}+\abs{\brac{(\p_h-\widehat{\p}_h)V_{h+1}}(s,a)}+b^\theta_h(s,a)\notag\\
    &\leq  2H\cdot \indic{d^{\pib}_h(s,a)<\oiat, a\in \supp\paren{\piE_h(\cdot|s)}}+2b^\theta_h(s,a),
    \label{eq:offline_decomposition_bound_r}
\end{align}
where the second line follows from the triangle inequality, the third line comes from \Eqref{eq:offline_decomposition_indic}, the second last line follows from $\|A_h\|_\infty\leq H$, the last line comes from the concentration event $\mc{E}$ (i). Finally, we give the bound of $ \EE_{\pival}|V_{h}^{\pival}(s;r^\theta)-V_{h}^{\pival}(s;\widehat{r}^\theta)|$. 
By definition of the $V$ function, we have
\begin{align*}
&\EE_{\pival}\left|V_{h}^{\pival}(s;r^\theta)-V_{h}^{\pival}(s;\widehat{r}^\theta)\right|=\sum_{s\in \s}d_h^{\pival}(s)\cdot \left|V_{h}^{\pival}(s;r^\theta)-V_{h}^{\pival}(s;\widehat{r}^\theta)\right|\\
&= \sum_{s'\in \s}d^{\pival}_h(s')\cdot\left|\sum_{h'\geq h}\sum_{(s,a)\in \s\times\a}d^{\pival}_{h'}(s_{h'}=s,a_{h'}=a|s_h=s')\cdot\brac{r^\theta_{h'}-\widehat{r}^\theta_{h'}}(s,a) \right|\\
&\leq \sum_{h'\geq h}\sum_{(s,a)\in \s\times\a}\set{\sum_{s\in \s}d^{\pival}_h(s)\cdot d^{\pival}_{h'}(s,a|s_h=s)}\cdot\abs{\brac{r^\theta_{h'}-\widehat{r}^\theta_{h'}}(s,a)} \\
&\overset{(i)}{\leq} \sum_{h'\geq h}\sum_{(s,a)\in \s\times\a}d^{\pival}_{h'}(s,a)\cdot\abs{\brac{r^\theta_{h'}-\widehat{r}^\theta_{h'}}(s,a)} \\
&\overset{(ii)}{\leq}\sum_{h'\geq h}\sum_{(s,a)\in \s\times\a}d^{\pival}_{h'}(s,a)\cdot\brac{2H\cdot \indic{d^{\pib}_h(s,a)<\oiat, a\in \supp\paren{\piE_h(\cdot|s)}}+2b^\theta_h(s,a)}\\ 
&\leq \sum_{h\in [H]}\sum_{(s,a)\in \s\times\a}\frac{2Hd^{\pival}_{h}(s,a)}{d^{\pib}_{h}(s,a)} d^{\pib}_{h}(s,a)\cdot \indic{d^{\pib}_h(s,a)<\oiat, a\in \supp\paren{\piE_h(\cdot|s)}}\\
&+\sum_{h\geq 1}\sum_{(s,a)\in \s\times\a}2d^{\pival}_{h}(s,a)b^\theta_h(s,a)\\
&\leq 2H\cdot\oiat\cdot\sum_{h\in [H]}\sum_{(s,a)\in \s\times\a}\frac{d^{\pival}_{h}(s,a)}{d^{\pib}_{h}(s,a)}
+\sum_{h\geq 1}\sum_{(s,a)\in \s\times\a}2d^{\pival}_{h}(s,a)b^\theta_h(s,a)\\
&\overset{(iii)}{\lesssim} \frac{C^\star H^2S\eta\iota}{K}+\sum_{h\in[H]}\sum_{(s,a)\in \s\times\a}d^{\pival}_{h}(s,a)b^\theta_h(s,a),
\end{align*}
where $d^{\pival}_{h'}(s_{h'}=s,a_{h'}=a|s_h=s)=\pp_h(s_{h'}=s,a_{h'}=a|s_h=s)$, (i) is due to $d^{\pival}_{h'}(s,a)=\sum_{s\in \s}d^{\pival}_h(s)\cdot d^{\pival}_{h'}(s_{h'}=s,a_{h'}=a|s_h=s)$, (ii) follows from \Eqref{eq:offline_decomposition_bound_r} and (iii) comes from definition of $C^\star$-concentrability. This completes the proof.
\end{proof}

\subsection{Proof of Theorem \ref{thm:offline_main}}
\label{app:proof-offline_main}

\begin{proof}
    By Lemma~\ref{lem:offine_performance_decomposition}, we have
    \begin{align}
D^{\pival}_\para(\sRR,{\hRR})\lesssim\frac{C^\star H^2S\eta\iota}{K}+\underbrace{\sup_{\theta\in \para}\sum_{h\in[H]}\sum_{(s,a)\in \s\times\a}d^{\pival}_{h}(s,a)b^\theta_h(s,a)}_{(\mathrm{I})}\label{offline_main_d_decom}.
    \end{align}
  
   Plugging \Eqref{eq:def_b_theta} into \Eqref{offline_main_d_decom}, we obtain that
    \begin{align}
        (\mathrm{I})&=\sum_{h\in [H]}\sum_{(s,a)\in \s\times\a}d^{\pival}_{h}(s,a)b_h(s,a)\notag\\
        &\lesssim \sum_{h\in [H]}\sum_{(s,a)\in \s\times\a}d^{\pival}_{h}(s,a)\cdot \Bigg\{\sqrt{\frac{\cn\iota}{{N}_h^b(s,a)\vee 1}\brac{\widehat{\v}_h{V}_{h+1}}(s,a)}\\
        &+
 \frac{H\cn\iota}{{N}_h^b(s,a)\vee 1}+\frac{\epsilon}{H}\paren{1+\sqrt{\frac{\cn\iota}{{N}_h^b(s,a)\vee 1}}}\Bigg\}\notag\\
        &\leq \underbrace{\sum_{h\in [H]}\sum_{(s,a)\in \s\times\a}d^{\pival}_{h}(s,a)\cdot \sqrt{\frac{\cn\iota}{{N}_h^b(s,a)\vee 1}\brac{{\v}_h{V}_{h+1}}(s,a)}}_{(\mathrm{I.a})}\notag\\
        &+\underbrace{\sum_{h\in [H]}\sum_{(s,a)\in \s\times\a}d^{\pival}_{h}(s,a)\cdot \sqrt{\frac{\cn\iota}{{N}_h^b(s,a)\vee 1}\brac{\paren{\widehat{\v}_h-\v_h}{V}_{h+1}}(s,a)}}_{(\mathrm{I.b})}\notag\\
        &+
 \underbrace{\sum_{h\in [H]}\sum_{(s,a)\in \s\times\a}d^{\pival}_{h}(s,a)\cdot\frac{H\cn\iota}{{N}_h^b(s,a)\vee 1}}_{(\mathrm{I.c})}\notag\\
 &+\underbrace{\epsilon\cdot\sum_{h\in [H]}\sum_{(s,a)\in \s\times\a}d^{\pival}_{h}(s,a)\cdot\paren{\frac{1}{H}+\sqrt{\frac{\cn\iota}{H^2\cdot{N}_h^b(s,a)\vee 1}}}}_{(\mathrm{I.d})}
   \label{offline_main_term_I_decom}   
   \end{align}
   where the last inequality comes from the triangle inequality.
We study the four terms separately.\\
For the term (I.a), on the concentration event $\mc{E}$, we have

\begin{align}
    &(\mathrm{I.a})=\sum_{h\in [H]}\sum_{(s,a)\in \s\times\a}d^{\pival}_{h}(s,a)\cdot \sqrt{\frac{\cn\iota}{{N}_h^b(s,a)\vee 1}\brac{{\v}_h{V}_{h+1}}(s,a)}\notag\\
    &~\lesssim \sum_{h\in [H]}\sum_{(s,a)\in \s\times\a}d^{\pival}_{h}(s,a)\cdot \sqrt{\frac{\cn\iota}{Kd^{\pib}(s,a)}\brac{{\v}_h{V}_{h+1}}(s,a)}
    \notag\\
    &~\leq \sqrt{\frac{H^2\cn\iota}{K}}\cdot\sum_{h\in [H]}\sum_{(s,a)\in \s\times\a}\sqrt{d^{\pival}_{h}(s,a)}\cdot \sqrt{\frac{d^{\pival}_{h}(s,a)}{d^{\pib}_{h}(s,a)}}
    \notag\\
     &~\leq\sqrt{\frac{H^2\cn\iota}{K}}\cdot \sqrt{\sum_{h\in [H]}\sum_{(s,a)\in \s\times\a}\frac{d^{\pival}_{h}(s,a)}{d^{\pib}_h(s,a)}}\cdot\sqrt{\sum_{h\in [H]}\sum_{(s,a)\in \s\times\a}d^{\pival}_{h}(s,a)}
     \tag{by Cauchy-Schwarz inequality}
     \\
     &~\leq\sqrt{\frac{C^\star H^4S\cn\iota}{K}}\label{eq:offline_main_term_I_a},
    \end{align}
    where the second line comes from concentration event $\mc{E}$(ii), the third line is valid since $\|V_{h+1}\|_{\infty}\leq H$ and the last is by thw definition of $C^\star$-concentrability.

Next, we study the term (I.b).
For any $(h,s,a)$, we have
\begin{align}
& \abs{\brac{\paren{\widehat{\v}_h-\v_h}V_{h+1}}\paren{s,a}}
\notag
\\
&= \brac{(\widehat{\p}_{h} {V}_{h+1})^2 - (\widehat{\p}_{h} V_{h+1})^2 - \paren{\p_{h} ({V}_{h+1})^2 - (\p_{h} {V}_{h+1})^2}}(s,a)
\notag
\\
&\leq  \abs{\brac{(\widehat{\p}_{h} - \p_{h})({V}_{h+1})^2}(s,a)} + \abs{\brac{(\widehat{\p}_{h} + \p_{h}){V}_{h+1} \cdot (\widehat{\p}_{h} - \p_{h}){V}_{h+1}}(s,a)}
\notag
\\
&\leq  2H\abs{\brac{(\widehat{\pp}_h-\pp_h)(V_{h+1})}(s,a)} + 2 H \abs{\brac{(\widehat{\p}_{h} - \p_{h}){V}_{h+1}}(s,a)}
\notag
\\
&\lesssim~  c \sqrt{\frac{H^4 \iota}{{N}_h^b(s,a) \vee 1}},\label{eq:spvar2popvar}
\end{align}
\def \dtt{\mrm{dim}^{1/2}\paren{\para}}
where the second last inequality is by $\|{V_{h+1}}\|_{\infty}\le H$ and the last inequality follows from the Azuma-Hoeffding inequality. By applying \Eqref{eq:spvar2popvar}, we can obtain the bound for the term (I.b):
\begin{align}
    &(\mathrm{I.b})=\sum_{h\in [H]}\sum_{(s,a)\in \s\times\a}d^{\pival}_{h}(s,a)\cdot \sqrt{\frac{\cn\iota}{{N}_h^b(s,a)\vee 1}\brac{\paren{\widehat{\v}_h-\v_h}{V}_{h+1}}(s,a)}
\notag
\\
&~\leq\sum_{h\in [H]}\sum_{(s,a)\in \s\times\a}d^{\pival}_{h}(s,a)\cdot \sqrt{\frac{\cn\iota}{{N}_h^b(s,a)\vee 1}\cdot \sqrt{\frac{H^4 \iota}{\widehat{N}_h^b(s,a) \vee 1}}}\notag
\\
&~=\paren{\cn}^{1/2}\cdot\sum_{h\in [H]}\sum_{(s,a)\in \s\times\a}d^{\pival}_{h}(s,a)\cdot 
\frac{H\iota^{3/4}}{\set{{N}_{h}^{{b}}\left(s,a\right)\vee 1}^{3/4}}
\notag
\\
&~\leq
\underbrace{\paren{\cn}^{1/2}\cdot\sum_{h\in [H]}\sum_{(s,a)\in \s\times\a}d^{\pival}_{h}(s,a)\cdot 
\sqrt{\frac{1}{{N}_{h}^{{b}}\left(s,a\right)\vee 1}}}_{(\mathrm{I.b.1})}\notag\\
&+
\underbrace{\paren{\cn}^{1/2}\cdot\sum_{h\in [H]}\sum_{(s,a)\in \s\times\a}d^{\pival}_{h}(s,a)\cdot 
{\frac{H^2\iota^{3/2}}{{N}_{h}^{{b}}\left(s,a\right)\vee 1}}}_{(\mathrm{I.b.2})},
\end{align}
where the last line is from AM-GM inequality.
For the term (I.b.1), on the concentration event $\mc{E}$, we have
\begin{align}
&(\mathrm{I.b.1})=\paren{\cn}^{1/2}\cdot\sum_{h\in [H]}\sum_{(s,a)\in \s\times\a}d^{\pival}_{h}(s,a)\cdot 
\sqrt{\frac{1}{{N}_{h}^{{b}}\left(s,a\right)\vee 1}}
\notag
\\
 &~\leq \paren{\cn}^{1/2}\cdot\sum_{h\in [H]}\sum_{(s,a)\in \s\times\a}\sqrt{d^{\pival}_{h}(s,a)}\cdot 
 \sqrt{\frac{d^{\pival}_{h}(s,a)}{Kd^{\pib}_h(s,a)}}
 \notag
 \\
 &~\leq\paren{\cn}^{1/2}\cdot\sqrt{\frac{1}{K}}\cdot\sqrt{\sum_{h\in [H]}\sum_{(s,a)\in \s\times\a}\frac{d^{\pival}_{h}(s,a)}{d^{\pib}_h(s,a)}}\cdot\sqrt{\sum_{h\in [H]}\sum_{(s,a)\in \s\times\a}d^{\pival}_{h}(s,a)}
 \notag
 \\
 &~\leq \sqrt{\frac{C^\star HS\cn}{K}}\cdot\sqrt{\sum_{h\in [H]}\sum_{(s,a)\in \s\times\a}d^{\pival}_{h}(s,a)}\notag\\
 &~=\sqrt{\frac{C^\star H^2S\cn}{K}}
 \label{eq:offline_main_term_I_a_1},
\end{align}
where the second line is directly from concentration event $\mc{E}(ii)$, the third line follows from Cauchy-Schwarz inequality and the second last line comes from the definition of $C^\star$-concentrability.
For the term (I.b.2), on the concentration event $\mc{E}$, we obtain
\begin{align}
   &(\mathrm{I.b.2})= \paren{\cn}^{1/2}\cdot\sum_{h\in [H]}\sum_{(s,a)\in \s\times\a}d^{\pival}_{h}(s,a)\cdot 
{\frac{H^2\iota^{3/2}}{{N}_{h}^{{b}}\left(s,a\right)\vee 1}}\notag\\
&~\leq \paren{\cn}^{1/2}\cdot\sum_{h\in [H]}\sum_{(s,a)\in \s\times\a}d^{\pival}_{h}(s,a)\cdot 
{\frac{H^2\iota^{5/2}}{Kd^{\pib}_h(s,a)}}\notag
 \\
 &~\leq \paren{\cn}^{1/2}\cdot{\frac{H^2\iota^{5/2}}{K}\sum_{h\in [H]}\sum_{(s,a)\in \s\times\a} 
{\frac{d^{\pival}_{h}(s,a)}{d^{\pib}_h(s,a)}}}
\notag
 \\
&~= 
{\frac{C^\star H^3S\cn\iota^{5/2}}{K}}
\label{eq:offline_main_term_I_a_2},
\end{align}
where the second line comes from concentration event $\mc{E}$\text{(ii)}, the third line follows from the definition of $C^\star$-concentrability.\\
Combining \Eqref{eq:offline_main_term_I_a_1} and \Eqref{eq:offline_main_term_I_a_2}, the term (I.b) can be bounded as follows:
\begin{align}
   (\mathrm{I.b})\lesssim \sqrt{\frac{C^\star H^2S\cn}{K}}+{\frac{C^\star H^3S\cn\iota^{5/2}}{K}}.\label{eq:offline_main_term_I_b}
\end{align}
For the term (I.c), observe that
\[
(\mathrm{I.c})=(\mathrm{I.b.2})/(H\iota^{3/2}).
\]
Hence, by \Eqref{eq:offline_main_term_I_a_2}, we deduce that
\begin{align}
   (\mathrm{I.c})\leq {\frac{C^\star H^2S\cn\iota}{K}}
    \label{eq:offline_main_term_I_c}
\end{align}
For the term (I.d),
\begin{align}
   (\mathrm{I.d})&=\epsilon\cdot\sum_{h\in [H]}\sum_{(s,a)\in \s\times\a}d^{\pival}_{h}(s,a)\cdot\paren{\frac{1}{H}+\sqrt{\frac{\cn\iota}{H^2\cdot{N}_h^b(s,a)\vee 1}}} \notag\\
   &=\epsilon+\epsilon\cdot \sum_{h\in [H]}\sum_{(s,a)\in \s\times\a}d^{\pival}_{h}(s,a)\cdot\sqrt{\frac{\cn\iota}{H^2\cdot{N}_h^b(s,a)\vee 1}}\notag\\
   &=\epsilon+\epsilon\cdot \sum_{h\in [H]}\sum_{(s,a)\in \s\times\a}d^{\pival}_{h}(s,a)\cdot\sqrt{\frac{\cn\iota}{H^2Kd^{\pib}_h(s,a)}}\notag\\
   &=\epsilon+\epsilon\sqrt{\frac{\cn\iota}{H^2K}}\cdot \sum_{h\in [H]}\sum_{(s,a)\in \s\times\a}\sqrt{d^{\pival}_{h}(s,a)}\sqrt{\frac{d^{\pival}_{h}(s,a)}{d^{\pib}_{h}(s,a)}}\notag\\
   &\leq \epsilon+\epsilon\sqrt{\frac{\cn\iota}{H^2K}}\sqrt{\sum_{h\in [H]}\sum_{(s,a)\in \s\times\a}{d^{\pival}_{h}(s,a)}}\cdot\sqrt{\sum_{h\in [H]}\sum_{(s,a)\in \s\times\a}\frac{d^{\pival}_{h}(s,a)}{d^{\pib}_{h}(s,a)}}\notag\\
   &\leq \epsilon\cdot(1+\sqrt{\frac{C^\star S\cn\iota}{K}}) ,\label{eq:offline_main_term_I_d} 
   \end{align}
   where the second last line is by Cauchy-Schwarz inequality and the last line is by definition of $C^\star$-concentrablity.\\
   Combining \Eqref{eq:offline_main_term_I_a}, \Eqref{eq:offline_main_term_I_b}, \Eqref{eq:offline_main_term_I_c} and \Eqref{eq:offline_main_term_I_d}, we deduce that
   \begin{align*}
  (\mathrm{I})&\lesssim (\mathrm{I.a})+(\mathrm{I.b})+(\mathrm{I.c})+(\mathrm{I.d})\notag
   \\
   &\lesssim   \sqrt{\frac{C^\star H^4S\cn\iota}{K}}+\sqrt{\frac{C^\star H^2S\cn}{K}}+{\frac{C^\star H^3S\cn\iota^{5/2}}{K}}\\
   &+{\frac{C^\star H^2S\cn\iota}{K}}
   +\epsilon\cdot(1+\sqrt{\frac{C^\star S\cn\iota}{K}})\\
   &\lesssim \sqrt{\frac{C^\star H^4S\cn\iota}{K}}+{\frac{C^\star H^3S\cn\iota^{5/2}}{K}}+\epsilon.
   \end{align*}
   Finally, plugging into \Eqref{offline_main_d_decom}, the final bound is given by
   \begin{align*}
   D^{\pival}_\para(\sRR,{\hRR})&\lesssim \frac{C^\star H^2S\eta\iota}{K}+\sqrt{\frac{C^\star H^4S\cn\iota}{K}}+{\frac{C^\star H^3S\cn\iota^{5/2}}{K}}+\epsilon\\
   \end{align*}
   The right-hand-side is upper bounded by $2\eps$ as long as
   \[
   K\geq\widetilde{\cO}\paren{\frac{C^\star H^4S\cn}{\epsilon^2}+\frac{C^\star H^2S \eta}{\epsilon}}.
   \]
   Here $poly \log\paren{H,S,A,1/\delta}$ are omitted.

\end{proof}
\subsection{Proof of Theorem~\ref{coro:pi=piE}}
\label{app:proof-pi=piE}
\def \bd{\bar{d}}

In this section, we will consider the case that $\pival=\piE$. We first introduce the following concentration event which is slightly different from the concentration event defined in Lemma~\ref{lem:concentration_1}.
\begin{lemma}[Concentration event ]\label{lem:concentration_11}
    Under the setting of Theorem~\ref{thm:offline_main}, there exists an absolute constant $C_1,~C_2$ such that the concentration event $\mc{E}$ holds with probability at least $1-\delta$, where
    \begin{align}
          \mc{E} \defeq \Bigg\{
          \text{(i):} &~  \abs{\brac{(\p_h-\widehat{\p}_h) V_{h+1}}(s, a)}\leq 
          b^\theta_h(s,a)~~~\forall~\theta=(V,A)\in  \para,~ (h,s,a)\in \hsa,
       \\
       \text{(ii):} &~ \frac{1}{{N}_{h}(s, a)\vee 1} \le \frac{C_1\iota}{Kd^{\pib}_h(s, a)}, ~~~\forall 
      (h,s,a)\in[H]\times\mc{S}\times\mc{A},
        \\\text{(iii):} &~ {N}^e_{h}(s,a)\geq 1 ~~~\forall (h,s,a)\in \cS\times \cA~\textrm{s.t.}~ \bd_h(s,a)\geq \frac{ C_2\iota}{ K}, a\in \supp\paren{\piE_h(\cdot|s)}
         \Bigg\},
    \end{align}
    where
    $b^\theta_h(s,a)$ is defined in \Eqref{rlp:eq_8}, $C^\star$ is specified in Definition~\ref{ass:off_1}, and ${N}^e_{h}(s)$ is given by
    \begin{align}
    &{N}^e_{h}(s,a)\defeq
    \begin{cases}
        \sum_{(s_h,a_h, e_h)\in \cD}\indic{(s_h, e_h)=(s,a)}\qquad &\textrm{in option~} 1,\\
        N^b_h(s,a)\qquad &\textrm{in option}~2,
    \end{cases}\notag\\
   &\bd_h(s,a)\defeq
   \begin{cases}
      {d^{\pib}_h(s)\cdot \piE_h(a|s)}\qquad &\textrm{in option~} 1,\\
        d^{\pib}_h(s,a)\qquad &\textrm{in option~}2. 
   \end{cases}
   \end{align}

\end{lemma}
\begin{proof}
    Repeating the arguments in the proof of Lemma~\ref{lem:concentration_1}, we can prove that claim (i), (ii) holds with probability at least $1-\frac{2\delta}{3}$.\\
    For claim (iii), for any $(h,s,a)\in [H]\times \cS\times \cA$ such that $a\in \supp\paren{\piE_h{\cdot|s}}$ and $\bd_h(s,a)\geq \frac{ C_2\iota}{ K}$, ${N}^e_{h}(s, a)\sim \Bin\paren{K, \bd_h(s, a)}$.
By direct computing, we obtain that
\begin{align*}
\p\brac{{N}^e_{h}(s, a)=0}&=\paren{1-\bd_h(s, a)}^{K}
= \brac{1-\paren{\frac{\delta}{3HSA}}^{1/K}+\paren{\frac{\delta}{3HSA}}^{1/K}-\bd_h(s, a)}^{K}\\
&\leq \brac{\paren{\frac{\delta}{3HSA}}^{1/K}+\underbrace{1-\paren{\frac{\delta}{3HSA}}^{1/K}-\bd_h(s, a)}_{\leq 0}}^{K}\\ 
&\leq \paren{\frac{\delta}{3HSA}}^{1/K\cdot K}=\frac{\delta}{3HSA},
\end{align*}

where the last inequality is valid since 
\begin{align*}
    1-\paren{\frac{\delta}{3HSA}}^{1/K}=1-\exp(-\frac{1}{K}\log\frac{\delta}{3HSA})\leq -\frac{\widetilde{C}_2}{K}\log\frac{\delta}{3HSA}\leq \frac{{C}_2\iota}{K}\leq \bd_h(s,a),
\end{align*}
where $\widetilde{C}_2$ and $C_2$ are absolute constants. Hence, it holds that
\begin{align*}
    {N}^e_h(s,a)\geq 1,
\end{align*}
with probability at least $1-\delta/(3HSA)$.
Taking the union bound over all $(h,s,a)\in \hsa$ yields that
\[
{N}^e_h(s,a)\geq 1
\]
holds with probability at least $1-\delta/3$
for all $(h,s,a)\in \cS\times \cA~\textrm{s.t.}~\bd_h(s,a)\geq \frac{ C_2\iota}{ K}, a\in \supp\paren{\piE_h(\cdot|s)}$,
which implies that claim (iii) holds. 
Further, we conclude that the concentration event $\cE$ holds with probability at least $1-\delta$.
\end{proof}
\begin{proof}[Proof of Corollary~\ref{coro:pi=piE}]
Recall that $r^\theta=\sRR(V,A)$ and $\widehat{r}^\theta=\hRR(V,A)$ for any $\theta=(V,A)\in \para$.
    When $\pival=\piE$, repeating the arguments in Lemma~\ref{lem:offine_performance_decomposition}, we have following decomposition:
    \begin{align*}
        \dpival\paren{r^\theta,\widehat{r}^\theta}&\leq 2H\cdot\sum_{h\in [H]}\sum_{(s,a)\in \s\times\a} d^{\piE}_h(s,a)\cdot \indic{\bd_h(s,a)<\frac{C_2\iota}{K},a\in \supp\paren{
        \piE_h(\cdot|s)
        }}\\
        &+\sum_{h\in [H]}\sum_{(s,a)\in \s\times\a}2d^{\pival}_{h}(s,a)b^\theta_h(s,a)\notag\\
        &\leq 2H\cdot\sum_{h\in [H]}\sum_{(s,a)\in \s\times\a} \frac{d^{\piE}_h(s,a)}{\bd_h(s,a)}\cdot\bd_h(s,a)\cdot \indic{\bd_h(s,a)<\frac{C_2\iota}{K},a\in \supp\paren{
        \piE_h(\cdot|s)
        }}\\
        &+\sum_{h\in [H]}\sum_{(s,a)\in \s\times\a}2d^{\pival}_{h}(s,a)b^\theta_h(s,a)\notag\\
        &\lesssim \frac{H\iota}{K}\cdot \sum_{h\in [H]}\sum_{(s,a)\in \s\times\a} \frac{d^{\piE}_h(s,a)}{\bd_h(s,a)}+\sum_{h\in [H]}\sum_{(s,a)\in \s\times\a}d^{\pival}_{h}(s,a)b^\theta_h(s,a)\notag\\
        &\leq \frac{C^\star H^2SA\iota}{K}+\sum_{h\in [H]}\sum_{(s,a)\in \s\times\a}d^{\pival}_{h}(s,a)b^\theta_h(s,a).
    \end{align*}
    where the second last line is valid since
    \begin{align*}
        \sum_{h\in [H]}\sum_{(s,a)\in \s\times\a} \frac{d^{\piE}_h(s,a)}{\bd_h(s,a)}&=\sum_{h\in [H]}\sum_{(s,a)\in \s\times\a} \frac{d^{\piE}_h(s)\cdot \piE_h(a|s)}{d^{\pib}_h(s)\cdot \piE_h(a|s)}\notag\\
        &=A\cdot \sum_{h\in [H]} \sum_{s\in \s} \frac{d^{\piE}_h(s)}{d^{\pib}_h(s)}\notag\\
        &=A\cdot  \sum_{h\in [H]}\sum_{s\in \s} \frac{\sum_{a\in \cA}d^{\piE}_h(s,a)}{\sum_{a\in \cA}d^{\pib}_h(s,a)}\notag\\
        &\leq A\cdot  \sum_{h\in [H]}\sum_{s\in \s} \max_{a\in \cA}\frac{d^{\piE}_h(s,a)}{d^{\pib}_h(s,a)}\notag\\
        &\leq A\cdot \sum_{h\in [H]}\sum_{(s,a)\in \s\times\a}\frac{d^{\piE}_h(s,a)}{d^{\pib}_h(s,a)}\notag\\
        &\leq C^\star HSA.
    \end{align*}

Similar as \Eqref{offline_main_d_decom}, we can decompose 
$D^{\pival}_\para(\sRR,{\hRR})$ as follows:
\begin{align}
    D^{\piE}_\para(\sRR,{\hRR})\lesssim \frac{C^\star H^2S\eta\iota}{K}+\underbrace{\sup_{\theta\in \para}\sum_{(s,a)\in \s\times\a}d^{\pival}_{h}(s,a)b^\theta_h(s,a)}_{\text{(I)}}\label{offline_coro_d_decom}
\end{align}
We can decompose terms (I) into four terms (I.a), (I.b), (I.c), and (I.d) as in  \Eqref{offline_main_term_I_decom}.  Since we don't use claim (iii) in the proof of bounding (I.b), (I.c), and (I.d), \Eqref{eq:offline_main_term_I_b}, \Eqref{eq:offline_main_term_I_c} and \Eqref{eq:offline_main_term_I_d} still holds on the concentration event $\cE$ defined in Lemma~\ref{lem:concentration_11}. In the following, we will prove an improved bound of the term (I.a):
\begin{align}
       (\mathrm{I.a})&=\sum_{h\in [H]}\sum_{(s,a)\in \s\times\a}d^{\piE}_{h}(s,a)\cdot 
\sqrt{\frac{\cn\iota}{{N}_h^b(s,a)\vee 1}\brac{{\v}_h{V}_{h+1}}(s,a)}\notag\\
    &\leq \sum_{h\in [H]}\sum_{(s,a)\in \s\times\a}d^{\piE}_{h}(s,a)\cdot 
\sqrt{\frac{\cn\iota}{Kd^{\pib}(s,a)}\brac{{\v}_h{V}_{h+1}}(s,a)}
    \notag\\
    &= \sqrt{\frac{\cn\iota}{K}}\cdot\sum_{h\in [H]}\sum_{(s,a)\in \s\times\a}\sqrt{d^{\piE}_{h}(s,a)\cdot\brac{{\v}_h{V}_{h+1}}(s,a)}\cdot \sqrt{\frac{d^{\piE}_{h}(s,a)}{d^{\pib}_{h}(s,a)}}\notag\\
    &\leq \sqrt{\frac{\cn\iota}{K}}\cdot \sqrt{\sum_{h\in [H]}\sum_{(s,a)\in \s\times\a}d^{\piE}_{h}(s,a)\cdot\brac{{\v}_h{V}_{h+1}}(s,a)}\cdot \sqrt{\sum_{h\in [H]}\sum_{(s,a)\in \s\times\a}\frac{d^{\piE}_{h}(s,a)}{d^{\pib}_{h}(s,a)}}\notag\\
    & \sqrt{\frac{C^\star HS\cn\iota}{K}}\cdot\sqrt{\sum_{h\in [H]}\sum_{(s,a)\in \s\times\a}d^{\piE}_{h}(s,a)\cdot\brac{{\v}_h{V}_{h+1}}(s,a)} 
     \label{eq:offline_coro_term_I_b_exp_pre}
    \end{align}
    We then give a sharp bound of $\sum_{h\in [H]}\sum_{(s,a)\in \s\times\a}{d^{\piE}_{h}(s,a)}\cdot \brac{{\v}_h{V}_{h+1}}\left(s,a\right)$. 
    \begin{align}
 &\sum_{h\in [H]}\sum_{(s,a)\in \s\times\a}{d^{\piE}_{h}(s,a)}\cdot \brac{{\v}_h{V}_{h+1}}\left(s,a\right)\notag\\
 &= \sum_{h = 1}^{H} \E_{\piE} \left[\Var_{\piE}\left[  V_{h + 1}(s_{h + 1}) \middle| s_h, a_h \right]\right]
\notag\\
&\overset{\text{(i)}}{=} \sum_{h = 1}^{H} \E_{{\piE}} \left[\E\left[  \left(V_{h + 1}(s_{h + 1})+A_h(s_h,a_h)\cdot \indic{a_h\notin \supp\paren{\piE(\cdot|s_h)}} + r^\theta_h(s_h,a_h) -  V_h(s_h)\right)^2 \middle| s_h, a_h \right]\right]\notag\\
&{=}\sum_{h = 1}^{H} \EE_{{\piE}} \left[ \left(V_{h + 1}(s_{h + 1}) + A_h(s_h,a_h)\cdot \indic{a_h\notin \supp\paren{\piE(\cdot|s_h)}}+r^\theta_h(s_h,a_h) -  V_h(s_h)\right)^2 \right]
\notag\\
 &\overset{\text{(ii)}}{=} \sum_{h = 1}^{H} \EE_{{\piE}} \left[ \left(V_{h + 1}(s_{h + 1}) + r^\theta_h(s_h,a_h) -  V_h(s_h)\right)^2 \right]
\notag\\
&=  \E_{{\piE}} \left[ \paren{\sum_{h = 1}^{H}\left(V_{h + 1}(s_{h + 1}) + r^\theta_h(s_h,a_h) -  V_h(s_h)\right)}^2 \right]
\notag\\
& + 2 {\sum_{1 \leq h < h' \leq H} \EE_{{\pi^E}} \left[ \left(V_{h + 1}(s_{h + 1}) + r^\theta_h(s_h,a_h) -  V_h(s_h)\right) \cdot \left( V_{h' + 1}(s_{h' + 1}) + r^\theta(s_h',a_h') -  V_h(s_h')\right) \right]}
\notag\\
&\overset{\text{(iii)}}{=}  \EE_{{\piE}} \left[ \left(\sum_{h = 1}^{H} \left(V_{h + 1}(s_{h + 1}) + r^\theta_h(s_h,a_h) -  V_h(s_h)\right)\right)^2 \right]
\notag\\
&=  \E_{{\piE}} \left[ \left(\sum_{h = 1}^{H} r^\theta_h(s_h,a_h) + \sum_{h = 1}^{H} \left(V_{h + 1}(s_{h + 1}) -  V_h(s_h)\right)\right)^2 \right]
\notag\\
&=  \EE_{{\piE}} \left[ \left(\sum_{h = 1}^{H} r^\theta_h(s_h,a_h)  -  V_1(s_1)\right)^2 \right]
\notag\\
&\overset{\text{(iv)}}{= } \Var_{{\piE}}\paren{ \sum_{h=1}^H r^\theta_h(s_h, a_h) } \le H^2.\label{eq:total_var},
\end{align}
where (i) is by definition of reward mapping $r^{\theta}_h(s,a)=-A_h(s,a)\cdot\indic{a\in \supp\paren{\piE_h(\cdot|s)}}+V_h(s)-\brac{\pp_hV_{h+1}}(s,a)$, (ii) comes from
\begin{align*}
   \indic{a_h\notin \supp\paren{\piE_h(\cdot|s_h)}}=0
\end{align*}
for any $(s_h,a_h)\in \supp\paren{d^{\piE}_h(\cdot)}$,
(iii) is valid since 
\begin{align}
&\EE_{{\pi^E}} \left[ \left(V_{h + 1}(s_{h + 1}) + r^\theta_h(s_h,a_h) -  V_h(s_h)\right) \cdot \left( V_{h' + 1}(s_{h' + 1}) + r^\theta(s_h',a_h') -  V_h(s_h')\right) \right]\notag\\
&=
\EE_{\piE}\brac{\left(V_{h + 1}(s_{h + 1}) + r^\theta_h(s_h,a_h) -  V_h(s_h)\right)\EE_{d^{\piE}}[V_{h' + 1}(s_{h' + 1}) -  V_{h'}(s_{h'}) + r^\theta_{h'}(s_{h'},a_{h'}) | \mc{F}_{h+1}]} = 0,
\end{align}
($\cF_{h+1}$)
and
(iv) is by $\para\in \parab$.
Plugging \Eqref{eq:total_var} into \Eqref{eq:offline_coro_term_I_b_exp_pre}, we deduce that
\begin{align}
(\mathrm{I.a})\leq \sqrt{\frac{C^\star H^3S\cn\iota}{K}} \label{eq:offline_coro_term_I_b}.
\end{align}
Combining \Eqref{eq:offline_coro_term_I_b}, \Eqref{eq:offline_main_term_I_b}, \Eqref{eq:offline_main_term_I_c} and \Eqref{eq:offline_main_term_I_d}, we have

\begin{align*}
   (\mathrm{I})&\lesssim (\mathrm{I.a})+(\mathrm{I.b})+(\mathrm{I.c})+(\mathrm{I.d})\notag
   \\
   &\lesssim   \sqrt{\frac{C^\star H^3S\cn\iota}{K}}+\sqrt{\frac{C^\star H^2S\cn}{K}}+{\frac{C^\star H^3S\cn\iota^{5/2}}{K}}\\
   &+{\frac{C^\star H^2S\cn\iota}{K}}
   +\epsilon\cdot(1+\epsilon\sqrt{\frac{C^\star S\cn\iota}{K}})\\
   &\lesssim \sqrt{\frac{C^\star H^3S\cn\iota}{K}}+{\frac{C^\star H^3S\cn\iota^{5/2}}{K}}+\epsilon.
   \end{align*}
  Pligging into \Eqref{offline_coro_d_decom}, the final bound is given by
   \begin{align*}
   D^{\pival}_\para(\sRR,{\hRR})&\lesssim \frac{C^\star H^2SA\iota}{K}+\sqrt{\frac{C^\star H^3S\cn\iota}{K}}+{\frac{C^\star H^3S\cn\iota^{5/2}}{K}}+\epsilon\\
   \end{align*}
   The right-hand-side is upper bounded by $2\eps$ as long as
   \[
   K\geq\widetilde{\cO}\paren{\frac{C^\star H^3S\cn}{\epsilon^2}+\frac{C^\star H^2SA}{\epsilon}}.
   \]
   Here $poly \log\paren{H,S,A,1/\delta}$ are omitted.

\end{proof}

\cblue
\subsection{Framework for offline inverse reinforcement learning}\label{app:frmaework_offline}
\begin{algorithm}[t]
\caption{\textsc{Framework for offline inverse reinforcement learning}}\label{alg:framework_rlp}
 \small
 \begin{algorithmic}[1]
  \STATE\textbf{Input:} Dataset $\cD$ collected by executing $\pi^{\msf{b}}$ in $\m$.
  \STATE Recover the transition dynamics  
  $\widehat{\pp}:[H]\times\cS\times \cA\to\Delta{\cS}$ and expert policy $\widehat{\pi}^\msf{E}=\set{\widehat{\pi}^\msf{E}_h:\cS\times \Delta(\cS)}$ and design the bonus $b:[H]\times\cS\times \cA\times \para\to \real_{\geq 0}$ .
      \STATE Compute $\widehat{\RR}$ by
      \begin{align}\label{eq:fram_exp}
      [\widehat{\RR}(V,A)]_h(s,a)=
      -A_h(s, a)\cdot \indic{a\notin {\rm supp}\paren{\widehat{\pi}^\msf{E}_h(\cdot|s)}} + V_h(s) - [\widehat{\p}_hV_{h+1}](s,a)-b^\theta_h(s,a)
   \end{align}
   \vspace{-.5em}
 \STATE \textbf{Output}: Estimated reward mapping $\widehat{\RR}$.
\end{algorithmic}
 \end{algorithm}
\paragraph{Pessimism} As shown in \Eqref{eq:fram_exp},  that estimator reward mapping involves a penalty term $b^\theta_h(s,a)$. The reason for introducing the penalty term $b^\theta_h(s,a)$ is to ensure that our reward satisfies the monotonicity condition: $\brac{\hRR(V,A)}_h(s,a)\leq \brac{\hRR(V, A)}_h(s, a)$, which is crucial for the guarantee of the performance of RL algorithms with learned rewards, as demonstrated in Proposition \ref{prop:mon_use_reward} and Corollary~\ref{coro:perform_RL_offline}.

\begin{condition}
\label{ass:frmwork} 
With probability at least $1-\delta$, we have
      $
          \sup_{(V,A)\in \para}\abs{\brac{(\p_h-\widehat{\p}_h) V_{h+1}}(s, a)}\leq 
          b^\theta_h(s,a)
      $
      and
$\supp\paren{\widehat{\pi}^{\msf{E}}_h(\cdot|s)}\subset \supp\paren{{\pi}^{\msf{E}}_h(\cdot|s)}$ for all $(h,s)\in [H]\times\cS $ and all $(V,A)\in \para$. 
\end{condition}

 \begin{theorem}[Learning bound for Algorithm~\ref{alg:framework_rlp}]\label{thm:framwork}
    Suppose that Condition~\ref{ass:frmwork} holds. With probability at least $1-\delta$, we have $\brac{\hRR(V,A)}_h(s,a)\leq \brac{\sRR(V, A)}_h(s, a)$ for all $(h,s,a)\in [H]\times \cS\times \cA$, and 
\begin{align}
\Dpivalt\paren{\sRR,\hRR}\leq \sup_{\theta\in \para} \Bigg\{&H\cdot\sum_{h\in [H]} \EE_{(s,a)\sim d_h^{\pival}}\brac{\indic{a\in \supp\paren{\piE_h(\cdot|s)},a\notin \supp\paren{\widehat{\pi}^{\msf{E}}_h(\cdot|s)}}}\notag\\
&+2\sum_{h\in [H]}\EE_{(s,a)\sim d_h^{\pival}}\brac{b^\theta_h(s,a)}\Bigg\}.
\end{align}
 \end{theorem}
 \begin{proof}
    When 
      $
          \sup_{(V,A)\in \para}\abs{\brac{(\p_h-\widehat{\p}_h) V_{h+1}}(s, a)}\leq 
          b^\theta_h(s,a)
      $
      and
$\supp{\widehat{\pi}^{\msf{E}}_h(\cdot|s)}\subset \supp{{\pi}^{\msf{E}}_h(\cdot|s)}$ holds for all $(h,s)\in [H]\times\cS $ and all $(V,A)\in \para$ hold, we have 
\begin{align}
&\brac{\hRR(V,h)}_h(s,a)-\brac{\sRR(V,h)}_h(s,a)\notag\\
&
=-A_h(s, a)\cdot \brac{\indic{a\notin {\rm supp}\paren{\widehat{\pi}^\msf{E}_h(\cdot|s)}}- \indic{a\notin {\rm supp}\paren{{\pi}^\msf{E}_h(\cdot|s)}}}-[\paren{\widehat{\p}_h-{\p}_h}V_{h+1}](s,a)-b_h^\theta(s,a)\notag\\
&=\underbrace{-A_h(s, a)\cdot {\indic{a\in \supp\paren{\piE_h(\cdot|s)},a\notin \supp\paren{\widehat{\pi}^{\msf{E}}_h(\cdot|s)}}}}_{\leq 0}
\underbrace{-[\paren{\widehat{\p}_h-{\p}_h}V_{h+1}](s,a)-b_h^\theta(s,a)}_{\leq 0}\leq 0,
\end{align}
where the second line is by $\supp\paren{\widehat{\pi}^{\msf{E}}_h(\cdot|s)}\subset \supp\paren{{\pi}^{\msf{E}}_h(\cdot|s)}$ and $\sup_{(V,A)\in \para}\abs{\brac{(\p_h-\widehat{\p}_h) V_{h+1}}(s, a)}\leq 
          b^\theta_h(s,a)$.
Further, by triangle inequality, we obtain that
\begin{align}\label{eq:framwork_1}
    &\abs{\brac{\hRR(V,h)}_h(s,a)-\brac{\sRR(V,h)}_h(s,a)}\notag\\
    &\leq A_h(s, a)\cdot {\indic{a\in \supp\paren{\piE_h(\cdot|s)},a\notin \supp\paren{\widehat{\pi}^{\msf{E}}_h(\cdot|s)}}}+
{\abs{[\paren{\widehat{\p}_h-{\p}_h}V_{h+1}](s,a)}+b_h^\theta(s,a)}\notag\\
&\leq H\cdot{\indic{a\in \supp\paren{\piE_h(\cdot|s)},a\notin \supp\paren{\widehat{\pi}^{\msf{E}}_h(\cdot|s)}}}+2b_h^\theta(s,a),
\end{align}
where the last line is due to $A_h(s,a)\leq H$ and $\abs{[\paren{\widehat{\p}_h-{\p}_h}V_{h+1}](s,a)}\leq b_h^\theta(s,a)$.
Similar to the proof of Lemma~\ref{lem:offine_performance_decomposition}, we have
\begin{align}\label{eq:framwork_2}
\dpival\paren{\hRR(V,A),\sRR(V,A)}\leq \sum_{h\in [H]}\EE_{(s,a)\sim d_h^{\pival}}\brac{\abs{\brac{\hRR(V,h)}_h(s,a)-\brac{\sRR(V,h)}_h(s,a)}}.
\end{align}
Combining \Eqref{eq:framwork_1} and \Eqref{eq:framwork_2}, we obtain that
\begin{align}
    \dpival\paren{\hRR(V,A),\sRR(V,A)}&\leq H\cdot\sum_{h\in [H]} \EE_{(s,a)\sim d_h^{\pival}}\brac{\indic{a\in \supp\paren{\piE_h(\cdot|s)},a\notin \supp\paren{\widehat{\pi}^{\msf{E}}_h(\cdot|s)}}}\notag\\
    &+2\sum_{h\in [H]}\EE_{(s,a)\sim d_h^{\pival}}\brac{b^\theta_h(s,a)}.
\end{align}
By the definition of $\Dpivalt$: $\Dpivalt\paren{\sRR,\hRR}=\sup_{\theta\in \para}\dpival\paren{\hRR(V,A),\sRR(V,A)}$, we complete the proof.
 \end{proof}
 By Theorem~\ref{thm:framwork},
 all we need to do is design $b_h^\theta$ and learn $\widehat{\pp},\widehat{\pi}^{\msf{E}}$ from the data to satisfy Condition~\ref{ass:frmwork}, thereby obtaining an IRL algorithm. The crux of the problem lies in the design of $b_h^\theta,\widehat{\pp}$ and $\widehat{\pi}^{\msf{E}}$. In \RLP{}, we employ the pessimism technique from offline RL, and the construction of $b_h^\theta$ and $\widehat{\pi}^{\msf{E}}$ using pessimism in \RLP{} satisfies Condition~\ref{ass:frmwork}, as illustrated in the proof of Theorem~\ref{thm:offline_main}.

\cblack

\section{Proofs for Section \ref{sec:online learning}}\label{appendix:proof for online learning}
\subsection{Full description of \RLEfull}
We propose a meta-algorithm, named \RLEfull{} (\RLE). The pseudocode of \RLE{} is presented in Algorithm~\ref{alg:RLEF}, where the algorithm contains the
following three main components:
\begin{itemize}
    \item Exploring the unknown environment: This segment involves computing a desired behavior policy $\pib=\E_{\pi\sim\mu^{\sf b}}\brac{\pi}$, which takes the form of a finite mixture of deterministic policies. To achieve this, we need to collect $NH$ episodes of samples.
  We then execute this policy to gather a total of $K$ episodes worth of samples. Our exploration approach is based on leveraging the exploration scheme outlined in \citet[Algorithm 1]{li2023minimaxoptimal}. A comprehensive description of this exploration method is postponed and will be provided in Section~\ref{appendix:prior work}. 
    \item Subsampling: For the sake of theoretical simplicity, we apply subsampling. For each $(h,s,a)\in \hsa$, we populate the new dataset with $\min\set{\widehat{N}_h^{\msf{b}}(s,a),N_h(s, a)}$ sample transitions. Here, $\widehat{N}_h^{\msf{b}}(s,a)$, as defined in \Eqref{eq:def:hatN}, acts as a lower bound on the total number of visits to $(h,s, a)$ among these $K$ sample episodes, with high probability. 
  \item Computing estimated reward mapping: With the previously collected dataset at hand,  we then utilize the offline IRL algorithm \RLP{} to compute the desired reward mapping.
\end{itemize}

We remark that our algorithm \RLE{} follows a similar approach to that of \citet[Algorithm 1]{li2023minimaxoptimal}. We begin by computing a desired behavior policy $\pib$, then proceed to collect data, and finally compute results through the invocation of an offline algorithm. In contrast to the offline setting, we have the flexibility to select the desired behavior. In the following, we will observe that the behavior policy $\pib$ exhibits concentrability with \emph{any} deterministic policy, as shown in \Eqref{eq:stopping_rule}. This property enables us to achieve our learning goal within the online setting.

\subsection{Proof of Theorem~\ref{thm:online_main}}
\label{app:proof-online_main}

\begin{lemma}[\citet{li2023minimaxoptimal}] \label{lemma:occupancy}
	Recall that $\xi=c_{\xi}H^3S^3A^3 \log \frac{10HSA}{\delta}$ for some large enough constant $c_{\xi}>0$ (see line~\ref{line:ex_1} in Algorithm~\ref{alg:ex}). Then,
	with probability at least $1-\delta$,  
	the estimated occupancy distributions specified in \Eqref{eq:hat-d-1-alg} and \eqref{eq:hat-d-h-alg} of Algorithm~\ref{alg:ex} satisfy
	\begin{align}
		\frac{1}{2}\widehat{d}_{h}^{\pi}(s, a) - \frac{\xi}{4N} \le d_{h}^{\pi}(s, a) \le 2\widehat{d}_{h}^{\pi}(s, a) + 2e_{h}^{\pi}(s, a) + \frac{\xi}{4N} \label{eq:d-error}
	\end{align}
	simultaneously for all $(h,s,a)\in[H]\times \mathcal{S}\times\mathcal{A} $ and all deterministic Markov policy $\pi\in\Pi^{\det}$, 
	provided that 
	\begin{align}
		KH \geq N \geq C_N\sqrt{H^9S^7A^7K}\log\frac{10HSA}{\delta} 
		\qquad \text{and}  \qquad 
		K\geq C_K HSA
		\label{eq:N-K-lower-bound-lemma2}
	\end{align}
	for some large enough constants $C_N,C_K>0$,  
	where, $\{e_{h}^{\pi}(s, a)\in \real_+\}$ satisfies that
	\begin{align}
		\sum_{(s, a)\in \cS\times \cA} e_{h}^{\pi}(s, a) \leq \frac{2SA}{K} + \frac{13SAH\xi}{N}  \lesssim \sqrt{\frac{SA}{HK}}
		\qquad \forall h\in [H],\pi\in \Pi^{\det}\label{eq:e-sum-zeta}
	\end{align}
\end{lemma}
Notice that \Eqref{eq:d-error} only holds for $\pi\in \Pi^{\det}$, however, we will show a similar result also holds for any stochastic policy. \\
For any stochastic policy $\pi=\EE_{\pi'\sim\mu}[\pi']$ ($\mu\in \Delta(\Pi^{\det})$), the visitation distribution $\set{d^{\pi}_h}$ can be expressed as
\begin{align*}
    d^{\pi}_h(s,a)=\EE_{\pi'\sim\mu}\brac{{d}^{\pi'}_h(s,a)},\qquad\forall (h,s,a)\in [H]\times\cS\times \cA,
\end{align*}
.
We can define $\widehat{d}^{\pi}$ as
\[
\widehat{d}^{\pi}_h(s,a)=\EE_{\pi'\sim\mu}\brac{\widehat{d}^{\pi'}_h(s,a)},\qquad\forall (h,s,a)\in [H]\times\cS\times\cA,
\]
where $\set{d^{\pi'}_h}$ are the estimated occupancy distributions in Algorithm~\ref{alg:ex}.

By \Eqref{eq:d-error}, we have
\begin{align}
    &\frac{1}{2}\widehat{d}_{h}^{\pi}(s, a) - \frac{\xi}{4N} \le d_{h}^{\pi}(s, a)=\EE_{\pi'\sim\mu}\brac{d_{h}^{\pi'}(s, a)} \le 2\widehat{d}_{h}^{\pi}(s, a) + 2\EE_{\pi'\sim\mu}\brac{e_{h}^{\pi'}(s, a)} + \frac{\xi}{4N}\notag\\
    &\frac{1}{2}\widehat{d}_{h}^{\pi}(s, a) - \frac{\xi}{4N} \le d_{h}^{\pi}(s, a)=\EE_{\pi'\sim\mu}\brac{d_{h}^{\pi'}(s, a)} \le 2\widehat{d}_{h}^{\pi}(s, a) + 2\EE_{\pi'\sim\mu}\brac{e_{h}^{\pi'}(s, a)} + \frac{\xi}{4N}\tag{$e_{h}^{\pi}(s, a)\defeq\EE_{\pi'\sim\mu}\brac{e_{h}^{\pi'}(s, a)}$}.
\end{align}
 We also have
\[
\sum_{(s, a)\in \cS\times \cA} e_{h}^{\pi}(s, a)=\sum_{(s, a)\in \cS\times \cA}\EE_{\pi'\sim\mu}\brac{e_{h}^{\pi'}(s, a)} \leq \frac{2SA}{K} + \frac{13SAH\xi}{N}  \lesssim \sqrt{\frac{SA}{HK}},
\]
provided \Eqref{eq:N-K-lower-bound-lemma2}.

Different from previous sections, we set $\iota=\log \frac{10HSA}{\delta}$.
\begin{lemma}[Concentration event]\label{lem:concentration_2}
    Suppose \Eqref{eq:N-K-lower-bound-lemma2}. Under the setting of Theorem~\ref{thm:online_main}, there exists an absolute constants $C_1$, $C_2\geq 2$ such that the concentration event $\mc{E}$ holds with probability at least $1-\delta$, where
    \begin{align*}
          \mc{E} \defeq \Bigg\{
          \text{(i):} &~  \abs{\brac{(\p_h-\widehat{\p}_h) V_{h+1}}(s, a)}\leq 
          b^\theta_h(s,a)~~~\forall~\theta=(V,A)\in  \para,~ (h,s,a)\in \hsa,
       \\
       \text{(ii):} &~ \frac{1}{2}\widehat{d}_{h}^{\pi}(s, a) - \frac{\xi}{4N} \le d_{h}^{\pi}(s, a) \le 2\widehat{d}_{h}^{\pi}(s, a) + 2e_{h}^{\pi}(s, a) + \frac{\xi}{4N}  ~~~\forall 
      (h,s,a)\in[H]\times\mc{S}\times\mc{A},\pi \in \Pi, 
        \\
        \text{(iii):} &~ \widehat{N}_{h}^{b}(s, a)\leq {N}_{h}^{b}(s, a)~~~\forall 
(h,s,a)\in[H]\times\mc{S}\times\mc{A},\\
\text{(iv):} &~ 
\widehat{N}^e_{h}(s,a)\geq 1 ~~~\forall (s,a)\in \cS\times \cA~\textrm{s.t.}~ 
\widehat{N}^b_h(s,a)\geq \max\set{{C_2\wellp\iota},1}\Bigg\}
    \end{align*}
    where $b^\theta_h(s,a)$ is defined in \Eqref{rlp:eq_8}, ${N}^{ b}_h(s,a)$ $\widehat{N}^{ b}_h(s,a)$ is defined in \Eqref{eq:def:hatN}, $\wellp$ are specified in Lemma~\ref{lem:concentration_1} and $\widehat{N}^e_{h}(s,a)$ is given by 
    \[
     \widehat{N}^e_{h}(s,a)\defeq
     \begin{cases}
         \sum_{(s_h,a_h, e_h)\in \cD^{\sf trim}}\indic{(s_h, e_h)=(s,a)}\qquad &\textrm{in option $1$},\\
         \widehat{N}^b_h(s,a)\qquad &\textrm{in option $2$},
     \end{cases}
  \]
\end{lemma}
\begin{proof}

 First, we observe that Claim (i) can be proved to hold with probability at least $1-\delta/10$ by repeating a similar argument as in Lemma~\ref{lem:concentration_1}. By Lemma~\ref{lemma:occupancy}, Claim (ii) holds with probability at least $1-\delta/10$. Claim (iii) has been shown to hold with probability at least $1-\delta/10$ in the proof of \citet[Theorem 2]{li2023minimaxoptimal}.

Next, we focus on (iv).
 For claim (iv), in option 1, we have
 \begin{align*}
     \pp\paren{\widehat{N}^e_h(s,a)=0}=\paren{1-\piE_h(a|s)}^{\widehat{N}^b_h(s,a)}
     &\leq \exp\paren{\widehat{N}^b_h(s,a)\log\paren{1-\wellp}}\\
     &\leq \exp\paren{\log\frac{\delta}{4HSA}}
     =\frac{\delta}{4HSA},
 \end{align*}
 for all $(h,s,a)\in \hsa$. The last line is valid since
 \[
 \widehat{N}^b_h(s,a)\log\paren{1-\wellp}\leq C_2\log\frac{\delta}{HSA}\cdot\frac{\log\paren{1-\wellp}}{\wellp}\leq \log\frac{\delta}{4HSA},
 \]
holds for sufficiently large constant $C_2$.
 In option 2, we have
 \[
 \widehat{N}^e_h(s,a)=\widehat{N}^b_h(s,a)\geq \max\set{C_2\wellp\iota,1}\ge 1,
 \]
 for all $(h,s,a)\in \hsa$.
This completes the proof.

\end{proof}
\subsection{Proof of Theorem~\ref{thm:online_main}}
Define 
\begin{align}
\mc{I}_h=\set{(s,a)\in \s\times \a\mid\EE_{\pi'\sim\mu_{\sf b}}\brac{\widehat{d}^{\pi'}_h(s,a)}\ge \frac{\xi}{N}+\frac{4\paren{C_2\eta+3}\iota}{K}},\label{eq:def_ICH}
\end{align}
for all $h\in [H]$.
Then for $(s,a)\in \mc{I}_h$, we have
\begin{align}
   \widehat{N}^b_h(s,a)\geq \frac{K}{4}\EE_{\pi'\sim\mu_{\sf b}}\brac{\widehat{d}^{\pi'}_h(s,a)}-\frac{K\xi}{8N}-3\iota
    \geq C_2\wellp\iota.
\end{align}
By concentration event $\cE$ (iv), we have
\[
\widehat{N}^e_h(s,a)\geq 1,
\]
By construction of $\widehat{\pi}^{\sf E}$ in Algorithm~\ref{alg:RLPfull}, we deduce that
\begin{align}
    \abs{\indic{a\in \supp\paren{\piE_{h}(\cdot|s)}}-\indic{a\in \supp\paren{\widehat{\pi}^{\sf E}_{h}(\cdot|s)}}}=0.
\end{align}
for all $(s,a)\in \cI_h$.\\
With $\cI_h$ at hand, we can decompose the $d^{\pi}\paren{r^{\theta},\widehat{r}^\theta}$ for any $\pi$ and $\theta\in \Theta$ as follows:
\begin{align}\label{eq:off_upper_bound_6}
    &d^{\pi}\paren{r_h^{\theta},\widehat{r}_h^\theta}\leq \sum_{(h,s,a)\in \hsa}d^{\pi}_h(s,a)\cdot\abs{r_h^{\theta}(s,a)-\widehat{r}_h^\theta(s,a)}\notag\\
    &\leq \underbrace{\sum_{h\in[H]}\sum_{(s,a)\notin \mc{I}_h}d^{\pi}_h(s,a)\cdot \abs{r^{\theta}_h(s,a)-\widehat{r}^\theta_h(s,a)}}_{\mathrm{(\mathrm{I})}}
    +\underbrace{\sum_{h\in[H]}\sum_{(s,a)\in \mc{I}_h}d^{\pi}_h(s,a)\cdot \abs{r^{\theta}(s,a)-\widehat{r}_h^\theta(s,a)}}_{(\mathrm{II})},
\end{align}
where the first line follows the same argument in the proof of Lemma~\ref{lem:offine_performance_decomposition}.
We then study the terms (I) and (II) separately. 
For the term (I), by the construction of Algorithm~\ref{alg:RLPfull}, we obtain that

\begin{align}
    &{(\mathrm{I})}=\sum_{h\in[H]}\sum_{(s,a)\notin \mc{I}_h}d^{\pi}_h(s,a)\cdot \abs{r^{\theta}(s,a)-\widehat{r}^\theta(s,a)}\notag\\
    &=\sum_{h\in[H]}\sum_{(s,a)\notin \mc{I}_h}d^{\pi}_h(s,a)\cdot|-A_{h}(s,a)\paren{\indic{a\in \supp\paren{\piE_h(\cdot|s)}}-\indic{a\in \supp\paren{\widehat{\pi}^{\sf E}_h(\cdot|s)}}}
    -\brac{\paren{\p_h-\widehat{\p}_h}V_{h+1}}(s,a)
    -b^\theta_h(s,a)|\notag\\
    &\leq  \sum_{h\in[H]}\sum_{(s,a)\notin \mc{I}_h}d^{\pi}_h(s,a)\cdot\Big\{\abs{{A_{h}(s,a)\cdot\paren{\indic{a\in \supp\paren{\widehat{\pi}^{\sf E}_h(\cdot|s)}}-\cdot\indic{a\in\supp\paren{\piE_h(\cdot|s)}}}}}\notag\\
    &+\abs{\brac{(\p_h-\widehat{\p}_h)V_{h+1}}(s,a)}+b^\theta_h(s,a)\Big\}\tag{by triangle inequality}\\
    &\overset{\text{(i)}}{\lesssim}H\cdot \sum_{h\in[H]}\sum_{(s,a)\notin \mc{I}_h}d^{\pi}_h(s,a)\notag\\
    &\overset{\text{(ii)}}{\lesssim}
    H\cdot\sum_{h\in[H]}\sum_{(s,a)\notin \mc{I}_h}\paren{2\widehat{d}_{h}^{\pi}(s, a) + 2e_{h}^{\pi}(s, a) + \frac{\xi}{4N}}\notag\\
    &\overset{\text{(iii)}}{\lesssim}H\cdot\sum_{h\in[H]}\sum_{(s,a)\notin \mc{I}_h}\widehat{d}^{\pi}_h(s,a)+
    \frac{\xi H^2SA}{N}+\sqrt{\frac{HSA}{K}}\notag\\
    &=H\cdot\sum_{h\in[H]}\sum_{(s,a)\notin \mc{I}_h}\frac{\widehat{d}^{\pi}_h(s,a)}{\E_{\pi'\sim \mu^{\sf b}}\brac{\widehat{d}_h^{\pi'}(s,a)}+\frac{1}{KH}}\cdot\paren{\E_{\pi'\sim \mu^{\sf b}}\brac{\widehat{d}_h^{\pi'}(s,a)}+\frac{1}{KH}}+
    \frac{\xi H^2SA}{ N}+\sqrt{\frac{HSA}{K}}\notag\\
    &\overset{\text{(iv)}}{\lesssim} \paren{\frac{\xi H}{N}+\frac{4H\paren{C_2\eta+3}\iota}{K}+\frac{1}{K}}\sum_{h\in[H]}\sum_{(s,a)\notin \mc{I}_h}\frac{\widehat{d}^{\pi}_h(s,a)}{\E_{\pi'\sim \mu^{\sf b}}\brac{\widehat{d}_h^{\pi'}(s,a)}+\frac{1}{KH}}+
    \frac{\xi H^2SA}{N}+\sqrt{\frac{HSA}{K}}\notag\\
    &\lesssim \paren{\frac{H\xi }{N}+\frac{4H\paren{C_2\eta+3}\iota}{K}+\frac{1}{K}}\cdot HSA+
    \frac{\xi H^2SA}{ N}+\sqrt{\frac{HSA}{K}}\notag\\
    &\asymp \frac{\xi H^2SA}{N}+\frac{H^2SA\eta\iota}{K}+\frac{HSA}{K}+\sqrt{\frac{HSA}{K}},\label{eq:off_upper_bound_7}
\end{align}
where \text{(i)} is by $\|A_h\|_{\infty}$, $\|V_{h+1}\|_{\infty},b^\theta_h(s,a)\leq H$, (ii) comes from concentration $\cE$(ii), (iii) comes from \Eqref{eq:d-error}, and (iv) is by definition of $\cI_h$.
For the term (I), conditioning on the concentration event $\mc{E}$, we have
\begin{align}
    (\mathrm{II})
    &=\sum_{h\in[H]}\sum_{(s,a)\in \mc{I}_h}d^{\pi}_h(s,a)\cdot \abs{r^\theta_h(s,a)-\widehat{r}^\theta_h(s,a)}\notag\\
    &\leq \sum_{h\in[H]}\sum_{(s,a)\in \mc{I}_h}d^{\pi}_h(s,a)\cdot\abs{\brac{(\p_h-\widehat{\p}_h) V_{h+1}}(s, a)-b^\theta_h(s,a)}\notag\\
    &\leq 2\sum_{h\in[H]}\sum_{(s,a)\in \mc{I}_h}d^{\pi}_h(s,a)\cdot{b^\theta_h(s,a)}\notag\\
    &\leq \sum_{h\in[H]}\sum_{(s,a)\in \mc{I}_h}\paren{4\widehat{d}_{h}^{\pi}(s, a) + 4e_{h}^{\pi}(s, a) + \frac{\xi}{2N}}\cdot b^\theta_h(s,a)\notag\\
    &\lesssim \sum_{h\in[H]}\sum_{(s,a)\in \mc{I}_h}\widehat{d}_{h}^{\pi}(s, a)\cdot b^\theta_h(s,a) +
    H\cdot\sum_{h\in[H]}\sum_{(s,a)\in \mc{I}_h}\paren{\frac{\xi}{N}+e_{h}^{\pi}(s, a)}\notag\\
    &\lesssim  \frac{\xi H^2SA}{N}+\sqrt{\frac{HSA}{K}}+\sum_{h\in[H]}\sum_{(s,a)\in \mc{I}_h}{\widehat{d}_{h}^{\pi}(s, a)\cdot b^{\theta}_h(s,a)},\label{eq:off_upper_bound_8}
\end{align}

where the second line is by construction of Algorithm~\ref{alg:RLPfull}, the second last line is by $b^\theta_h(s,a)\lesssim H$, the last follows from \eqref{eq:d-error}. Further, we decompose the second term of \Eqref{eq:off_upper_bound_8} for any $\theta\in \para$ by
\begin{align}
    &\sum_{h\in[H]}\sum_{(s,a)\in \mc{I}_h}{\widehat{d}_{h}^{\pi}(s, a)\cdot b^{\theta}_h(s,a)}\notag\\
    &=\sum_{h\in[H]}\sum_{(s,a)\in \mc{I}_h}\widehat{d}_{h}^{\pi}(s, a)\cdot\min\Bigg\{\sqrt{\frac{\cn\iota}{\widehat{N}_h^b(s,a)\vee 1}\brac{\widehat{\v}_h{V}_{h+1}}(s,a)}+
 \frac{H\cn\iota}{\widehat{N}_h^b(s,a)\vee 1}\notag\\
 &+\frac{\epsilon}{H}\paren{1+\sqrt{\frac{\cn\iota}{\widehat{N}_h^b(s,a)\vee 1}}}, H\Bigg\}\notag\\
    &\overset{\text{(i)}}{\leq} \sum_{h\in[H]}\sum_{(s,a)\in \mc{I}_h}\widehat{d}_{h}^{\pi}(s, a)\cdot\Bigg\{\min\set{\sqrt{\frac{\cn\iota}{\widehat{N}_h^b(s,a)\vee 1}\brac{\widehat{\v}_h{V}_{h+1}}(s,a)}, H}\notag\\
    &+
 \frac{H\cn\iota}{\widehat{N}_h^b(s,a)\vee 1}+\frac{\epsilon}{H}\paren{1+\sqrt{\frac{\cn\iota}{\widehat{N}_h^b(s,a)\vee 1}}}\Bigg\}\notag\\
&\overset{\text{(ii)}}{\leq}  \sum_{h\in[H]}\sum_{(s,a)\in \mc{I}_h}\widehat{d}_{h}^{\pi}(s, a)\cdot\Bigg\{{\sqrt{\frac{\cn\iota\brac{\widehat{\v}_h{V}_{h+1}}(s,a)+H}{\widehat{N}_h^b(s,a)\vee 1+1/H}}}+
 \frac{H\cn\iota}{\widehat{N}_h^b(s,a)\vee 1}\notag\\
 &+\frac{\epsilon}{H}\paren{1+\sqrt{\frac{\cn\iota}{\widehat{N}_h^b(s,a)\vee 1}}}\Bigg\}\notag\\
&\overset{\text{(iii)}}{=}\underbrace{\sum_{h\in[H]}\sum_{(s,a)\in \mc{I}_h}\widehat{d}_{h}^{\pi}(s, a)\cdot{\sqrt{\frac{\cn\iota\brac{\widehat{\v}_h{V}_{h+1}}(s,a)+H}{K\EE_{\pi'\sim\mu^{\sf b}}\brac{\widehat{d}^{\pi'}_h\paren{s,a}}+1/H}}}}_{(\mathrm{II.a})}\notag\\
& +
\underbrace{\sum_{h\in[H]}\sum_{(s,a)\in \mc{I}_h}\widehat{d}_{h}^{\pi}(s, a)\cdot
 \frac{H\cn\iota}{K\EE_{\pi'\sim\mu^{\sf b}}\brac{\widehat{d}^{\pi'}_h\paren{s,a}}+1/H}}_{(\mathrm{II.b})}\notag\\
 &+
 \underbrace{\frac{\epsilon}{H}\sum_{h\in[H]}\sum_{(s,a)\in \mc{I}_h}\widehat{d}_{h}^{\pi}(s, a)\cdot\paren{1+\sqrt{\frac{\cn\iota}{K\EE_{\pi'\sim\mu^{\sf b}}\brac{\widehat{d}^{\pi'}_h\paren{s,a}}+1/H}}}}_{(\mathrm{II.c})}
\end{align}
where the (i) is by inequality $\min\set{a+b,c}\leq \min\set{a,c}+b$~($a,b,c\ge 0$), (ii) comes from inequality $\min\set{\frac{x}{y},\frac{z}{w}}\leq \frac{x+z}{y+w}$ and (iii) is valid since 
\[
\widehat{N}_h^b(s,a)=\brac{\frac{K}{4}, \mathop{\mathbb{E}}_{\pi\sim \mu^{\msf{b}}}[\widehat{d}_{h}^{\pi}(s,a)] - \frac{K\xi}{8N} - 3\log\frac{HSA}{\delta}}_+\gtrsim K\EE_{\pi'\sim\mu^{\sf b}}\brac{\widehat{d}^{\pi'}_h\paren{s,a}}+1/H
\]
holds for all $(s,a)\in \cI_h$ according to definition of $\cI$.
We then study the three terms separately. For the term (II.a), by the Cauchy-Schwarz inequality, we have
\begin{align*}
(\mathrm{II.a})&\leq \underbrace{\set{\sum_{h\in [H]}\sum_{(s,a)\in \s\times\a}\widehat{d}_{h}^{\pi}(s, a)\cdot \brac{\cn\iota\brac{{\v}_h{V}_{h+1}}\left(s,a\right)+H}}^{1/2}}_{(\mathrm{II.a.1})}\\
&~~\times \underbrace{\set{\sum_{h\in [H]}\sum_{(s,a)\in \s\times\a}\frac{\widehat{d}_{h}^{\pi}(s, a)}{K\EE_{\pi'\sim\mu^{\sf b}}\brac{\widehat{d}^{\pi'}_h\paren{s,a}}+1/H}}^{1/2}}_{(\mathrm{II.a.2})}.
\end{align*}
Observe that $\|V_{h+1}\|_{\infty}\leq H$, then the term (II.a.1) can be upper bounded by 
\begin{align}\label{eq:online_term_II_a_1}
    (\mathrm{II.a.1})&=\sqrt{\sum_{h\in [H]}\sum_{(s,a)\in \s\times\a}\widehat{d}_{h}^{\pi}(s, a)\cdot \brac{\cn\iota\brac{{\v}_h{V}_{h+1}}\left(s,a\right)+H}}\notag\\
    &\leq \sqrt{\brac{H^2\cn\iota+H}\cdot \sqrt{\sum_{h\in [H]}\sum_{(s,a)\in \s\times\a}\widehat{d}_{h}^{\pi}(s, a)}}\asymp \sqrt{H^3\cn\iota}.
\end{align}
For the term (II.a.2), we have
\begin{align}\label{eq:online_term_II_a_2}
    (\mathrm{II.a.2})&=\sqrt{\sum_{h\in [H]}\sum_{(s,a)\in \s\times\a}\frac{\widehat{d}_{h}^{\pi}(s, a)}{{K\EE_{\pi'\sim\mu^{\sf b}}\brac{\widehat{d}^{\pi'}_h\paren{s,a}}+1/H}}}\notag\\
    &=\sqrt{\frac{1}{K}\cdot\sum_{h\in [H]}\sum_{(s,a)\in \s\times\a}\frac{\widehat{d}_{h}^{\pi}(s, a)}{{\EE_{\pi'\sim\mu^{\sf b}}\brac{\widehat{d}^{\pi'}_h\paren{s,a}}+1/KH}}}\notag\\
    &\lesssim \sqrt{\frac{HSA}{K}},
\end{align}
which the last line comes from \Eqref{eq:stopping_rule}.
Combining \Eqref{eq:online_term_II_a_1} and \eqref{eq:online_term_II_a_2}, we conclude that
\begin{align}\label{eq:online_term_II_a}
    (\mathrm{II.a)}\lesssim \sqrt{\frac{H^4SA}{K}}.
\end{align}
For the term (II.b), by \Eqref{eq:stopping_rule}, we have
\begin{align}\label{eq:online_term_II_b}
   (\mathrm{II.b})&= \sum_{h\in [H]}\sum_{(s,a)\in \s\times\a}\widehat{d}_{h}^{\pi}(s, a)\cdot \frac{H\cn\iota}{K\EE_{\pi'\sim\mu^{\sf b}}\brac{\widehat{d}^{\pi'}_h\paren{s,a}}+1/H}\notag\\
   &=\frac{1}{K}\cdot\sum_{h\in [H]}\sum_{(s,a)\in \s\times\a}\widehat{d}_{h}^{\pi}(s, a)\cdot \frac{H\cn\iota}{\EE_{\pi'\sim\mu^{\sf b}}\brac{\widehat{d}^{\pi'}_h\paren{s,a}}+1/KH}\notag\\
   &\lesssim \frac{H^2SA\cn\iota}{K}.
\end{align}
For the term (II.c), we have
\begin{align}
    \mathrm{(II.c)}&=\frac{\epsilon}{H}\sum_{h\in[H]}\sum_{(s,a)\in \mc{I}_h}\widehat{d}_{h}^{\pi}(s, a)\cdot\paren{1+\sqrt{\frac{\cn\iota}{K\EE_{\pi'\sim\mu^{\sf b}}\brac{\widehat{d}^{\pi'}_h\paren{s,a}}+1/H}}}\notag\\
    &=\epsilon+\frac{\epsilon}{H}\sum_{h\in[H]}\sum_{(s,a)\in \mc{I}_h}\sqrt{\widehat{d}_{h}^{\pi}(s, a)}\cdot\sqrt{\frac{\widehat{d}_{h}^{\pi}(s, a)\cn\iota}{K\EE_{\pi'\sim\mu^{\sf b}}\brac{\widehat{d}^{\pi'}_h\paren{s,a}}+1/H}}\notag\\
    &\leq \epsilon+\frac{\epsilon}{H}
    \sqrt{\sum_{h\in[H]}\sum_{(s,a)\in \mc{I}_h} \widehat{d}_{h}^{\pi}(s, a)}\cdot\sqrt{\sum_{h\in[H]}\sum_{(s,a)\in \mc{I}_h}\frac{\widehat{d}_{h}^{\pi}(s, a)\cn\iota}{K\EE_{\pi'\sim\mu^{\sf b}}\brac{\widehat{d}^{\pi'}_h\paren{s,a}}+1/H}}\notag\\
    &\leq \epsilon(1+\sqrt{\frac{SA\cn\iota}{K}}),\label{eq:online_term_II_c} 
\end{align}
where the second last line is by the Cauchy-Schwarz inequality and the last line is due to \Eqref{eq:online_term_II_a_2}.

Then combining \Eqref{eq:off_upper_bound_8}, \Eqref{eq:online_term_II_a}, \Eqref{eq:online_term_II_b}, and \Eqref{eq:online_term_II_c}, we obtain the bound for the term (II)
\begin{align}\label{eq:online_term_II}
(\mathrm{II})&\lesssim(\mathrm{II.a})+(\mathrm{II.b})+(\mathrm{II.c})\notag\\
&\lesssim\sqrt{\frac{H^4SA\cn\iota}{K}} +\frac{H^2SA\cn\iota}{K}+\epsilon(1+\sqrt{\frac{SA\cn\iota}{K}})\notag\\
&\lesssim\sqrt{\frac{H^4SA\cn\iota}{K}} +\epsilon,
\end{align}
where the last line is from $\epsilon<1$. Finally, combining \Eqref{eq:off_upper_bound_7} and \eqref{eq:online_term_II_a}, we get the final bound 
\begin{align*}
    \DallTheta\paren{\RR^\star,\hRR}&= \sup_{\pi,\theta\in \para}d^{\pi}\paren{r_h^{\theta},\widehat{r}_h^\theta}\leq \mathrm{I}+\mathrm{II}\notag\\
    &\lesssim \frac{\xi H^2SA}{N}+\frac{H^2SA\eta\iota}{K}+\sqrt{\frac{H^4SA\cn\iota}{K}} +\epsilon.
\end{align*}
Hence, we can guarantee $ \DallTheta\paren{\RR^\star,\hRR}\leq 2\epsilon$, provided that
\begin{align*}
K\geq\widetilde{\cO}\paren{\frac{H^4SA\cn}{\epsilon^2}+\frac{H^2SA\eta}{\epsilon}}, \qquad KH \geq N \geq \widetilde{\cO}\paren{{\sqrt{H^9S^7A^7K}}}.
\end{align*} 
 Here $poly \log\paren{H,S,A,1/\delta}$ are omitted.

Suppose that $\epsilon\leq {H^{-9}(SA)^{-6}}$.
We set 
\begin{align}
N=\widetilde{\cO}\paren{H^9 S^7 A^7 K},\qquad K=\widetilde{\cO}\paren{{\frac{H^4SA\cn}{\epsilon^2}+\frac{H^2SA\eta}{\epsilon}}}.\label{eq:refine_rate_1}
\end{align}
When $\epsilon\leq {H^{-9}(SA)^{-6}}$, we have
\begin{align}
    KH&\geq \sqrt{K}H\cdot\widetilde{\cO}\paren{\sqrt{\frac{H^4SA\cn}{\epsilon^2}}}\notag\\
    &\geq \sqrt{K}H\cdot \widetilde{\cO}\paren{H^9S^6A^6\sqrt{{H^4SA}}}\notag\\
     &\geq \sqrt{K}H\cdot \widetilde{\cO}\paren{\sqrt{{H^9S^7A^7}}}\notag\\
     &\geq \widetilde{\cO}\paren{{\sqrt{H^9S^7A^7K}}}.\label{eq:refine_rate_2}
\end{align}
Combining \Eqref{eq:refine_rate_1} and \Eqref{eq:refine_rate_2}, we have
\begin{align}
    KH \geq N \geq \widetilde{\cO}\paren{{\sqrt{H^9S^7A^7K}}}.
\end{align}
Then, the total sample complexity is 
\begin{align}
    K+NH\ge 
    \widetilde{\cO}\paren{{\frac{H^4SA\cn}{\epsilon^2}+\frac{H^2SA\eta}{\epsilon}}+ \sqrt{\frac{H^{15}S^{8}A\cn}{\epsilon^2}}+\sqrt{\frac{H^{13}S^8A^8\eta}{\epsilon}}}.\label{eq:refine_rate_3}
\end{align}
When $\epsilon\leq {H^{-9}(SA)^{-6}}$, we have
\begin{align}
    \widetilde{\cO}\paren{\frac{H^4SA\cn}{\epsilon^2}}&\geq \widetilde{\cO}\paren{\frac{H^{13}S^7A^7\cn}{\epsilon}}\notag\\
    &=\widetilde{\cO}\paren{\sqrt{\frac{H^{26}S^{14}A^{14}\cn^2}{\epsilon^2}}}\tag{$\cn\geq 1$}\\
    &\geq \sqrt{\frac{H^{15}S^{8}A\cn}{\epsilon^2}}\label{eq:refine_rate_4}
\end{align}
and 
\begin{align}
   \widetilde{\cO}\paren{ \frac{H^2SA\eta}{\epsilon}}&\geq \widetilde{\cO}\paren{ \frac{H^2SA\eta}{\epsilon}}\notag\\
   &=\widetilde{\cO}\paren{ \sqrt{\frac{H^4S^2A^2\eta^2}{\epsilon^2}}}\notag\\
   &\geq \widetilde{\cO}\paren{ \sqrt{\frac{H^{13}S^8A^8\eta^2}{\epsilon}}}\notag\\
   &\ge 
   \widetilde{\cO}\paren{ \sqrt{\frac{H^{13}S^8A^8\eta}{\epsilon}}},\label{eq:refine_rate_5}
\end{align}
where the last line is due to $\eta\in \set{0}\cup[1,\infty)$.
Combining \Eqref{eq:refine_rate_3}, \Eqref{eq:refine_rate_4}, and \Eqref{eq:refine_rate_5}, we obtain that
\begin{align}
     K+NH\ge 
    \widetilde{\cO}\paren{{\frac{H^4SA\cn}{\epsilon^2}+\frac{H^2SA\eta}{\epsilon}}}
\end{align}
holds when $\epsilon\leq {H^{-9}(SA)^{-6}}$

\def \wav{\mathbf{w}\overset{a}{\leftarrow} v}
\def \waz{\mathbf{w}\overset{a}{\leftarrow} 0}
\def \waw{\mathbf{w}\overset{a}{\leftarrow} w}
\def \wavv{\mathbf{w}\overset{a}{\leftarrow} v^\star}
\def \gap{{\sf gap}}
\def \sst{s_{\textrm{start}}}
\def \srt{s_{\textrm{root}}}
\def \barw{\lcW}
\section{{Lower bound in the online setting}}\label{appendix:proof for lower bound online learning}
\def \stt{\mrm{start}}
\def \tep{\widetilde{\epsilon}}

\subsection{Lower bound of online IRL problems}
We focus on the case where $\para=\pV\times\pA$. In this case $\cn=\widetilde{\cO}(S)$, the upper bound of the sample complexity of Algorithm~\ref{alg:RLEF} becomes $\widetilde{\cO}\paren{{H^4S^2A}/{\epsilon^2}}$ (we hide the burn-in term).

Similar to the offline setting, we define \emph{$(\epsilon,
\delta)$-PAC algorithm for online IRL problems} for all $\epsilon,\delta\in (0,1)$ as follows.

\begin{definition}
Fix a parameter set $\para$, we say an online IRL algorithm $\mathfrak{A}$ is a $(\epsilon,
\delta)$-PAC algorithm for online IRL problems, if for any IRL problem $(\m,\piE)$, with probability $1-\delta$, 
$\mathfrak{A}$ outputs a reward mapping $\hRR$
such that
\[
D^{\sf all}_\para(\hRR,\sRR)\leq\epsilon.
\]
\end{definition}
\cblue
\begin{theorem}[Lower bound for online IRL problems]\label{thm:lower_bound_online}
Fix parameter set $\para=\pV\times\pA$ and let $\mathfrak{A}$ be an $(\epsilon,\delta)$-PAC algorithm for online IRL problems, where $\delta\leq 1/3$.
Then, there exists  an IRL problem $(\m,\piE)$ such that, if $H\geq 4,S\ge 130, A\ge 2$, there exists an absolute constant $c_0$ such that the expected sample complexity $N$ is lower bounded by

\[
{N}\geq \frac{c_0 H^3SA\min\set{S,A}}{\epsilon^2},
\]
where $0<\epsilon\leq \paren{H-2}/{1024}$;
\end{theorem}
\cblack
Note that when $S\leq A$, the lower bound scales with ${\Omega}\paren{S^2A}$, matching the $S^2A$ factor dependence observed in the upper bound (Theorem \ref{thm:online_main}).

\subsection{Hard instance construction}

\paragraph{Hard Instance Construction}
Our construction is a modification of the hard instance constructed in the proof of \citet[Theorem B.3]{metelli2023theoretical}.
We construct the hard instance with $2S+1$ states, $A+1$ actions, and $2H+2$ stages for any $H,~S,~A>0$. (This rescaling only affects $S,~H$ by at most a multiplicative constant and thus does not affect our result).
We then define an integer $K$ by
\$
K\defeq \min\set{S,A}.
\$
Each MDP $\m_{\mathbf{v}}$ is indexed by a vector $\mathbf{w}=\paren{w_h^{(i,j,k)}}_{h\in [H],i\in [K],j\in [S],k\in [K]}\in \real^{HSKA}$ and is specified as follows:
\begin{itemize}
    \item State space: $\s=\{\sst,\srt,s_{1},\ldots,s_S,\bar{s}_{1},\ldots,\bar{s}_S\}$.
    \item Action space: $\a=\set{a_0,a_1,...,a_A}$.
    \item Initial state: $s_{\mrm{start}}$, that is 
    \[
    \pp\paren{s_1=s_{\stt}}=1.
    \]
    \item Transitions: 
    \begin{itemize}
        \item At stage 1, $\sst$ can only transition to itself or $s_i$.
          The transition probabilities are given by 
          \begin{align*}
        \begin{cases}
         &\p_1(\sst\given s_{\stt},a_0)=1\\
            &\p_1(s_i\given s_{\stt},a_i)=1\qquad \text{for all}~i\in [K],\\
            &\p_1(s_j\given s_{\stt}, a_k)=\frac{1}{S} \qquad\text{for all}~j\in[S],~k\geq K+1,
         \end{cases}
     \end{align*}
    \item At each stage $h\in \set{2,\ldots,H+1}$, $\sst$ can only transition to itself or $s_i$, $s_i$ can only transition to absorbing state $\bar{s}_j$. The transition probabilities are given by
    \begin{align}\label{eq:def_hard_mdp}
        \begin{cases}
            &\p_h(\sst\given \sst,a_0)=1,\\
            &\p_h(s_i\given \sst,a_i)=1\qquad \text{for all}~i\in [K],\\
            &\p_h(s_j\given \sst, a_k)=\frac{1}{S} \qquad\text{for all}~j\in[S],~k\geq K+1,\\
            &\p_h(\bar{s}_j\given s_i,a_0)=\frac{1}{S}\qquad\text{for all}~i\ge K+1,~j\in[S],\\
            &\p_h(\bar{s}_j\given s_i,a_k)=\frac{1+\epsilon'\cdot w^{\paren{i,j,k}}_{h-1}}{S}\qquad\text{for all}~i\in [K],~j\in[S],~k\in[A],\\
            &\p_h(\bar{s}_j\given \bar{s}_j,a_k)=1\qquad \text{for all}~j\in [S],~k\ge 0.
        \end{cases}
    \end{align}
    \item At each stage $h\in \set{H+1,\ldots,2H+2}$ and  $\sst$ can only transition to $s_i$ and $s_i$ can only transition to absorbing state $\bar{s}_j$. The transition probabilities are given by
    \begin{align*}
        \begin{cases}
            &\p_h(s_i\given \sst,a_0)=\frac{1}{S}\qquad\text{for all}~i\in [S],\\
            &\p_h(s_i\given \sst,a_i)=1\qquad \text{for}~i\in [K],\\
            &\p_h(s_j\given \sst, a_k)=\frac{1}{S} \qquad\text{for all}~j\in[S],~k\geq K+1,~\\
            &\p_h(\bar{s}_j\given s_i,a_k)=\frac{1}{S}\qquad\text{for all}~i\in [K],~j\in[S],~k\geq 0,\\
            &\p_h(\bar{s}_j\given \bar{s}_j,a_k)=1\qquad \text{for all}~i\in [S],~k\geq 0.
        \end{cases}
    \end{align*}
    
    \end{itemize}
  \item Expert policy: expert policy $\piE$ plays action $a_0$ at every stage $h\in[H]$ and state $s\in \s$. That is
  \begin{align}
  \piE_h(a_0|s)=1,\qquad \text{for all}~h\in [2H+2],~s\in \s. \label{eq:low_b_exp_1}
 \end{align}
\end{itemize}
\def \bfw{\mathbf{w}}
In this case, $\Delta$ can be $1$, which means our lower bound is not derived from a large $\Omega\paren{1/\Delta}$ in our proof.
To ensure the definition of $\m_{\mathbf{w}}$ is valid, we enforce the following condition:
\[
\sum_{j\in [S]} w_{h}^{(i,j,k)}=0,
\]
for any $h\in [H]$, $i\in [K]$, $k\in [A]$. We define a vector space $\mc{W}$ by
\[
\mc{W}\defeq \set{{w}=(w_j)_{j\in [S]}\in \set{1,-1}^{S}:\sum_{j\in[S]}w_j=0}.
\]
Let $\mathcal{I}$ denote $[H]\times [K]\times [A]$, the \Eqref{eq:low_b_exp_1} is equivalent to 
\[
\bfw\in \cW^{\cI}.
\]
\def \lcW{{\overline{\cW}}}
\def \pbfW{{\paren{\bfw}}}
Further, we let $\p^{\pbfW}=\set{\p^{\pbfW}_h}_{h\in [H]}$ to be the transition kernel of \MDPR{} $\m_{\mathbf{w}}$. In addition, Given $\bfw\in \mathcal{W}^{\cI}, w\in \cW$ and index $a\in \mathcal{I}$, we use the notation $\mathbf{w}\overset{a}{\leftarrow} w$ to represent vector obtained by replacing $a$ component of $\mathbf{w}$ with $w$. For example, 
let $\mathbf{w}=(w^{(i,j,k)}_h)_{h\in [H],i\in[K],j\in [S],k\in [K]}$, $w=(w_j)_{j\in [S]}$, $a=(h_a,i_a,j_a)$ and $\overline{\mathbf{w}}=\mathbf{w}\overset{a}{\leftarrow} w$ and then $\overline{\mathbf{w}}$ can be expressed as follows:
\begin{align}
    \overline{w}^{(i,j,k)}_h=\begin{cases}
        w_j & (h,i,k)=(h_a,i_a,k_a),\\
{w}^{(i,j,k)}_h& \textrm{otherwise}.
    \end{cases}
\end{align}

By \citet[Lemma E.6]{metelli2023theoretical}, there exists a $\lcW\subseteq\mathcal{W}$ such that
   \#\label{eq:online_lb_corr}
   \sum_{i\in[n]}(w_i-v_i)^2\geq \frac{S}{8},\qquad \forall v,w\in \bar{\mc{W}},\qquad \log \abs{\lcW}\geq \frac{S}{10}.
   \#
   \paragraph{Notations.}
   To distinguish with different \MDPR{}s, we denote $V^{\pi}_h\paren{\cdot; r, \pp^{\pbfW}}$ be the value function of $\pi$ in MDP $\m_{\bfw}\cup r$. Given two rewards $r$ $r'$, we define $\dall(r,r';\pp^{\pbfW})$ to be the $\dall$ metric evaluated in $\m_{\bfw}$:
   \[
   \dall(r,r';\pp^{\pbfW})\defeq \sup_{\pi,h\in[H]}\EE_{\pp^{\pbfW},\pi} \abs{V^\pi_h(s_h;r,\pp^{\pbfW})-V^\pi_h(s_h;{r'},\pp^{\pbfW})}.
   \]
   Correspondingly, given a parameter set $\para$, two reward mappings $\RR$, $\RR'$, we can define $\DallTheta(\RR,\RR';\pp^{\pbfW})$ by
   \[
   \DallTheta(\RR,\RR';\pp^{\pbfW})\defeq
   \sup_{(V,A)\in \para} \dall\paren{\RR\paren{V,A},\RR'\paren{V,A};\pp^{\pbfW}}.
   \]
   In the following, we always assume that $\mathbf{w}\in \bar{\mathcal{W}}$. We then present the following lemma which shows the difference between two \MDPR{}s~ $\m_{\mathbf{w}\overset{a}{\leftarrow} v}$ and $\m_{\mathbf{w}\overset{a}{\leftarrow} w}$ for any ${\mathbf{w}}\in \lcW^{\cI}$ and $v\neq w\in \lcW$.
   \def \cWI{\lcW^{\cI}}
   \def \lep{\DallTheta\paren{\RR^{\paren{\waw}},\RR;\p^{\paren{\waw}}}}
   \begin{lemma}\label{lem_low_1}
   
   Given any $\mathbf{w}\in \cWI$, $w\neq v\in \lcW$, and index $a=(h_a, i_a, k_a)\in \mc{I}$, let 
   $\RR^{\paren{\waw}}$, $\RR^{\paren{\wav}}$ be the ground truth reward mapping induced by ${\m_{\waw}}$ , ${\m_{\wav}}$, respectively. Set $\para=\pV\times \pA$. For any $\epsilon'\in (0,1/2]$ and any reward mapping $\RR:\pV\times \pA\to \rew$, we have
   \$
   7\DallTheta\paren{\RR^{\paren{\waw}},\RR;\p^{\paren{\waw}}}+\DallTheta\paren{\RR^{\paren{\wav}},\RR;\pp^{\paren{\wav}}}\geq \frac{H\epsilon' }{16},
   \$
   where $\epsilon'$ is specified in \Eqref{eq:def_hard_mdp}.
   \end{lemma}

   \def  \pitest{\pi^{\msf{test},\paren{1}}}
   \def  \pitestt{\pi^{\sf {test},\paren{2}}}
   \def \hard{\msf{bad}}
   \def \pig{\pi^\msf{g}}
   \begin{proof}
   {\bf Step 1: Construct the bad parameter $(V^{\hard},A^{\hard})$.} We construct the bad parameter $(V^{\hard},A^{\hard})\in \pV\times \pA $ as follows:
   \begin{itemize}
       \item We set $A^{\hard}_h(s,a)=0$ for all $(h,s,a)\in [2H+2]\times\cS\times \cA$.
       \item We set $V^{\hard}_h$ by
\begin{align}
V^{\hard}_h(s)\defeq
          \begin{cases}
          \frac{(2H+2-h)\cdot\paren{w_i-v_i}}{2}&\text{if}~s=\bar{s}_i, h=h_{a}+2,\\
           0&\text{other}.
       \end{cases}
\end{align}
   \end{itemize}

   Directly by the construction of $(V^{\hard}, A^{\hard})$, we obtain that
   \begin{align}
       \sum_{i\in [S]}{\paren{w_i-v_i}\cdot V^{\hard}_{h_a+2}(\bar{s}_i)}=\sum_{i\in [S]}\frac{\paren{2H-h_a}\paren{w_i-v_i}^2}{2}\geq \frac{H\ \paren{w_i-v_i}^2}{2} \geq \frac{H S}{16},\label{eq:online_lb_eq_1}
   \end{align}
   where the last inequality is due to \Eqref{eq:online_lb_corr}. We then denote $\RR^{\paren{\waw}}\paren{V^{\hard},A^{\hard}}$, $\RR^{\paren{\wav}}\paren{V^{\hard},A^{\hard}}$ as $r^{\hard}_w$, $r^{\hard}_v$, respectively. 
   
   Since $A^{\hard}\equiv \mathbf{0}$,  any policy $\pi\in \Pi^{\star}_{\m_{\waw}\cup r^{\hard}_w}, \Pi^{\star}_{\m_{\wav}\cup r^{\hard}_v}$. More explicitly, any policy is optimal in $\m_{\waw}\cup r^{\hard}_w$ and $\m_{\wav}\cup r^{\hard}_v$.

   {\bf Step 2: Construct test policies $\pitest$, $\pitestt$.} Let $r=\RR\paren{V^\textrm{bad},A^\textrm{bad}}$. Let $\pi^{\sf g}\in \Pi^{\det}$ be a optimal policy of $\m_{\waw}\cup r$. By Lemma~\ref{lemma:metric_1}, there exist a pair $\paren{V,A}\in \pV\times \pA$ such that
   \begin{align}
        r_h(s,a)= -A_h(s, a)\cdot\indic{a\notin \supp\paren{\pig_h(\cdot\given s)} } + V_h(s) - \brac{\p^{\paren{\waw}}_h V_{h+1}}(s,a),\label{eq:lower_offline_rep_r}
   \end{align}
   
   We then construct test policy $\pitest$ by
   \begin{align*}
   \begin{cases}
       \pitest_h(a_0\given s_{\textrm{start}})=1&h\leq h_{a}-1\\
        \pitest_h(a_{i_a}\given s_{\textrm{start}})=1&h=h_{a}\\
        \pitest_h(a_{k_a}\given s_{i_a})=1&h=h_{a}+1\\
        \pitest_h=\pig_h&h\ge h_a+2
    \end{cases}
   \end{align*}
   which implies that at stage $h\leq h_a-1$, $\pitest$ always plays $a_0$, at stage $h_a$, $\pitest$ plays~$a_{i_a}$, then transition to $s_{i_a}$, at stage $h_a+1$, $\pitest$ plays $a_{k_a}$, then at stage $h\geq h_a+2$, $\pitest$ is equal to the greedy policy $\pig$. By construction, we can conclude that
   \begin{align}
d^{\pitest}_{h_a+1}(s_{i_a};\pp^{(\waw)})=1,\qquad
V^{\pitest}_{h_a+2}(\cdot|r,\pp^{\paren{\waw}})=V_{h_a+2}(\cdot).\label{eq:pitest_1_occ},
   \end{align}
   the second equality is due to $\pitest_h=\pi^{\sf g}_h$ for any $h\geq h_a+2$.

   Further, we have
   \begin{align}
      &V_{h_a+1}^{\pitest}(s_{i_a};r,\pp^{\paren{\waw}})
      =r_{h_a+1}(s_{i_a}, a_{k_a})+\brac{\pp^{\paren{\waw}}_{h_a+1}V^{\pitest}_{h_a+2}(\cdot|r,\pp^{\paren{\waw}})}\paren{s_{i_a}, a_{k_a}}\notag\\
      &=-A_{h_a+1}(s_{i_a}, a_{k_a})\cdot\indic{ a_{k_a}\notin \supp\paren{\pig_{h_a+1}(\cdot\given s_{i_a})}}\notag\\
      &+ V_{h_a+1}(s) - \brac{\p^{\paren{\waw}}_{h_a+1} V_{h_a+2}}(s_{i_a},a_{k_a})+\brac{\pp^{\paren{\waw}}_{h_a+1}V_{h_a+2}}\notag\\
      &=V_{h_a+1}(s_{i_a})-\gap,\label{eq:lowbol_lem_1_1}
   \end{align}
   where the first line is by the Bellman equation, the second line is due to \Eqref{eq:lower_offline_rep_r} and \Eqref{eq:pitest_1_occ}.
   Here $\gap$ is the advantage function at $(h_{a}+1,s_{i_a},a_{k_a})$, i.e, $\gap\defeq A_{h_a+1}(s_{i_a}, a_{k_a})\cdot\indic{ a_{k_a}\in \supp\paren{\pig_{h_a+1}(\cdot\given s_{i_a})}}$.
   Then by definition of $\DallTheta(\RR^{\paren{\waw}},\RR;\pp^{(\waw)})$, we can obtain that
   \begin{align}\label{eq:online_lb_eq_3}
&\DallTheta\paren{\RR^{\paren{\waw}},\RR;\p^{\paren{\waw}}}\notag\\
&\geq
\dall\paren{\RR^{\paren{\waw}}\paren{V^\hard,A^\hard},\RR\paren{V^\hard,A^\hard};\p^{\paren{\waw}}}\notag\\
&=\dall\paren{r^\hard_w, r;\p^{\paren{\waw}}}\notag\\
&\geq\EE_{\pp^{\paren{\waw}},\pitest}\abs{V_{h_a+1}^{\pitest}(s;r^{\hard}_w,\pp^{\paren{\waw}})-V_{h_a+1}^{\pitest}(s;r,\pp^{\paren{\waw}})}\notag\\
       &=\abs{V_{h_a+1}^{\pitest}(s_{i_a};r_w^{\hard},\pp^{\paren{\waw}})-V_{h_a+1}^{\pitest}(s_{i_a};r,\pp^{\paren{\waw}})}\notag\\
       &=\abs{V_{h_a+1}^{\hard}(s_{i_a})-V_{h_a+1}(s_{i_a})+\gap},
   \end{align}
   where the second last line is due to \Eqref{eq:pitest_1_occ} and  the last line is by \Eqref{eq:lowbol_lem_1_1} and $\pitest\in \Pi^\star_{\m_{\waw}\cup r^{\hard}_w} \Rightarrow V_{h_a+1}^{\pitest}(s_{i_a};r_w^{\hard},\pp^{\paren{\waw}})=V_{h_a+1}^{\hard}(s_{i_a})$.
   
   Next, we construct another test policy $\pitestt$ as follows:
   \begin{align*}
   \begin{cases}
       \pitestt_h(a_0\given s_{\textrm{start}})=1&h\leq h_{a}-1\\
        \pitestt_h(a_{i_a}\given s_{\textrm{start}})=1&h=h_{a}\\
        \pitestt_h=\pig_h&h\ge h_a+1.
    \end{cases}
   \end{align*}
   The difference between $ \pitestt$ and $ \pitest$ is that at stage $h_a$ $ \pitestt$ play the $\pig_{h_a+1}(s_{i_a})$ instead of $a_{k_a}$. Similar to \Eqref{eq:pitest_1_occ}, we have
   \begin{align}\label{eq:online_lb_eq_4}
       d^{\pitestt}_{h_a+1}(s_{i_a};\pp^{(\waw)})=1,\qquad V_{h_a+1}^{\pitestt}(s_{i_a};r,\pp^{\paren{\waw}})=V_{h_a+1}(s_{i_a})
   \end{align}
   where the seconed equality is valid since $\pitestt_h=\pig_h$ for any $h\ge h_a+1$.

  Similar to \Eqref{eq:online_lb_eq_3}, we have
  \begin{align}\DallTheta\paren{\RR^{\paren{\waw}},\RR;\p^{\paren{\waw}}}&\geq\dall\paren{r^\hard_w, r;\p^{\paren{\waw}}}\notag\\
      &\geq \EE_{\pp^{\paren{\waw}},\pitestt}\abs{V_{h_a+1}^{\pitestt}(s;r^{\hard}_w,\pp^{\paren{\waw}})-V_{h_a+1}^{\pitestt}(s;r,\pp^{\paren{\waw}})}\notag\\
       &=\abs{V_{h_a+1}^{\pitestt}(s_{i_a};r^{\hard}_w,\pp^{\paren{\waw}})-V_{h_a+1}^{\pitestt}(s_{i_a};r,\pp^{\paren{\waw}})}\notag\\
       &=\abs{V_{h_a+1}^{\hard}(s_{i_a})-V_{h_a+1}(s_{i_a})},\label{eq:lower_bound_online_vv}
  \end{align}
  where the second last is due to \Eqref{eq:online_lb_eq_4}, the last line follows from $\pitestt\in \Pi^\star_{\m_{\waw}\cup r^{\hard}_w}$: $V_{h_a+1}^{\pitestt}(s_{i_a};r^{\hard}_w,\pp^{\paren{\waw}})=V_{h_a+1}^{\hard}(s_{i_a})$.
  Combing \Eqref{eq:online_lb_eq_3} and \Eqref{eq:lower_bound_online_vv}, we have 
  \begin{align}
      2\lep&\geq \abs{V_{h_a+1}^{\hard}(s_{i_a})-V_{h_a+1}(s_{i_a})}+\abs{V_{h_a}^{\hard}(s_{i_a})-V_{h_a+1}(s_{i_a})+\gap}\notag\\
      &\geq \gap,\label{eq:online_lb_eq_6}
  \end{align}
  where the second line comes from the triangle inequality.\\
   {{\bf Step 3: lower bound}
    $\DallTheta\paren{\RR^{\paren{\wav}},\RR; \p^{\paren{\wav}}}$.}
  We still use the test policy $\pitest$ in $\m_{(\wav)}$. Since $\p^{\paren{\wav}}_h=\p^{\paren{\waw}}_h$ for any $h\ge h_a+2$, we have
  \begin{align}
      V^{\pitest}_{h_a+2}(\bar{s}_i|r,\pp^{\paren{\wav}})=V^{\pitest}_{h_a+2}(\bar{s}_i|r,\pp^{\paren{\waw}})=V_{h_a+2}(\bar{s}_i),\qquad \text{for all}~i\in [S],
  \end{align}
  where the second equality comes from \Eqref{eq:pitest_1_occ}.

  By the definition of $\DallTheta\paren{\RR^{\paren{\wav}},\RR;\p^{\paren{\wav}}}$, we have
  \begin{align}\label{eq:online_lb_eq_11}
&\DallTheta\paren{\RR^{\paren{\wav}},\RR;\p^{\paren{\wav}}}\notag\\
&\ge \dall\paren{r^{\hard}_v,r;\pp^{\paren{\wav}}}\notag\\
&\geq \EE_{\pp^{\paren{\wav}},\pitest}\abs{V_{h_a+1}^{\pitest}(s;r^{\hard}_v,\pp^{\paren{\wav}})-V_{h_a+1}^{\pitest}(s;r,\pp^{\paren{\wav}})}\notag\\
      &=\abs{V_{h_a+1}^{\pitest}(s_{i_a};r^{\hard}_v,\pp^{\paren{\wav}})-V_{h_a+1}^{\pitest}(s_{i_a};r,\pp^{\paren{\wav}})}\tag{by construction of policy $\pitest$.}\\
      &=\abs{V_{h_a+1}^{\hard}(s_{i_a})-V_{h_a+1}^{\pitest}(s_{i_a};r,\pp^{\paren{\wav}})}\notag\\
      &\overset{\textrm{(i)}}{=}\abs{V_{h_a+1}^{\hard}\paren{s_{i_a}}-r_{h_a+1}\paren{s_{i_a},a_{k_a}}-\brac{\p^{{\paren{\wav}}}_{h_a+1}V_{h_a+2}}\paren{s_{i_a},a_{k_a}}}\notag\\
      &{=}\abs{V_{h_a+1}^{\hard}\paren{s_{i_a}}-r_{h_a+1}\paren{s_{i_a},a_{k_a}}-\brac{\p^{{\paren{\waw}}}_{h_a+1}V_{h_a+2}}\paren{s_{i_a},a_{k_a}}-\brac{\paren{\p_{h_a+1}^{{\paren{\wav}}}-\p^{{\paren{\waw}}}_{h_a+1}}V_{h_a+2}}\paren{s_{i_a},a_{k_a}}}\notag\\
      &\geq \abs{\brac{\paren{\p_{h_a+1}^{{\paren{\wav}}}-\p^{{\paren{\waw}}}_{h_a+1}}V_{h_a+2}}\paren{s_{i_a},a_{k_a}}}-
      \abs{V_{h_a+1}^{\hard}\paren{s_{i_a}}-r_{h_a+1}\paren{s_{i_a},a_{k_a}}-\brac{\p^{{\paren{\waw}}}_{h_a+1}V_{h_a+2}}\paren{s_{i_a},a_{k_a}}}\tag{by triangle inequality}\\
     &\overset{\textrm{(ii)}}{=}\abs{\brac{\paren{\p_{h_a+1}^{{\paren{\wav}}}-\p^{{\paren{\waw}}}_{h_a+1}}V_{h_a+2}}\paren{s_{i_a},a_{k_a}}}-\abs{V^{\hard}_{h_a+1}(s_{i_a})-V_{h_a+1}(s_{i_a})+\gap}\notag\\
     &\geq \abs{\brac{\paren{\p_{h_a+1}^{{\paren{\wav}}}-\p^{{\paren{\waw}}}_{h_a+1}}V_{h_a+2}}\paren{s_{i_a},a_{k_a}}}-\abs{V^{\hard}_{h_a+1}(s_{i_a})-V_{h_a+1}(s_{i_a})}-\gap\notag\\
     &\overset{\textrm{(iii)}}{\geq} \abs{\brac{\paren{\p_{s_a+1}^{{\paren{\wav}}}-\p^{{\paren{\waw}}}_{h_a+1}}V_{h_a+2}}\paren{s_{i_a},a_{k_a}}}-3\lep\label{eq:finalbound_lb_offline_1},
  \end{align}
  where (i) is by the Bellman equation, (ii) is valid since 
  \begin{align}
  &r_{h_a+1}(s_{i_a}, a_{k_a})+\brac{\p^{{\paren{\waw}}}_{h_a+1}V_{h_a+2}}\paren{s_{i_a},a_{k_a}}\notag\\
  &~=-A_{h_a+1}(s_{i_a}, a_{k_a})\cdot \indic{a_{k_a}\in \supp\paren{\pi^g_{h_a+1}\paren{\cdot|s_{i_a}}}}+V_{h_a+1}(s_{i_a})\notag\\
  &-\brac{\p^{{\paren{\waw}}}_{h_a+1}V_{h_a+2}}\paren{s_{i_a},a_{k_a}}+\brac{\p^{{\paren{\waw}}}_{h_a+1}V_{h_a+2}}\paren{s_{i_a},a_{k_a}}\tag{by \Eqref{eq:lower_offline_rep_r}}\\
  &~=-\gap+V_{h_a+1}(s_{i_a})\notag
   \end{align}
   and (iii) is due to \Eqref{eq:lower_bound_online_vv} and \Eqref{eq:online_lb_eq_6}.
  We next analyse $\abs{\brac{\paren{\p_{h_a+1}^{{\paren{\wav}}}-\p_{h_a+1}^{{\paren{\waw}}}}V_{h_a+2}}\paren{s_{i_a},a_{k_a}}}$.
   We move back to $\pitest$. By the construction of $\pitest$ and the transition probabilities of $\m_{\waw}$, we have
   \begin{align}\label{eq:online_lb_eq_7}
       d^{\pitest}_{h_a+2}(\bar{s}_i;\pp^{(\waw)})=\frac{1+\epsilon' w_{i}}{S},\qquad V_{h_a+2}^{\pitest}(\bar{s}_i;r,\pp^{(\waw)})= V_{h_a+2}(\bar{s}_i),\qquad \forall i\in[S].
   \end{align}
   By definition of $\DallTheta(\RR^{\paren{\waw}},\RR;\pp^{\paren{\waw}})$, we have
   \begin{align}\label{eq:online_lb_eq_8}
&\DallTheta\paren{\RR^{\paren{\waw}},\RR;\pp^{\paren{\waw}}}\notag\\
&\geq \EE_{\pp^{\paren{\waw}},\pitest}\abs{V_{h_a+2}^{\pitest}(s;r^{\hard}_w, \pp^{\paren{\waw}})-V_{h_a+2}^{\pitest}(s;r,\pp^{\paren{\waw}})}\notag\\
       &\geq\sum_{i\in[S]} d^{\pitest}_{h_a+2}(\bar{s}_i;\pp^{(\waw)})\cdot \abs{V_{h_a+2}^{\pitest}(\bar{s}_{i};r^{\hard}_w,\pp^{\paren{\waw}})-V_{h_a+2}^{\pitest}(\bar{s}_{i};r,\ \pp^{\paren{\waw}})}\notag\\  
      &=\sum_{i\in[S]} \frac{1+\epsilon' w_{i}}{S}\cdot \abs{V_{h_a+2}^{\hard}(\bar{s}_{i})-V_{h_a+2}(\bar{s}_{i})}\notag\\
       &\geq  \sum_{i\in[S]} \frac{1}{2S}\cdot \abs{V_{h_a+2}^{\hard}(\bar{s}_{i})-V_{h_a+2}(\bar{s}_{i})},
   \end{align}
   where the last second is by \Eqref{eq:online_lb_eq_7} and the last line comes from $\epsilon'\in (0, 1/2]$.
   Applying \Eqref{eq:online_lb_eq_8}, we obtain that
   \begin{align}\label{eq:online_lb_eq_9}
       &\abs{\brac{\paren{\p^{{\paren{\wav}}}-\p^{{\paren{\waw}}}}V_{h_a+2}}\paren{s_{i_a},a_{k_a}}}\notag\\
       &=\abs{\frac{\epsilon'}{S}\cdot\sum_{i\in [S]}V_{h_a+2}(\bar{s}_i)\cdot (w_i-v_i)}\notag\\
       &\geq \abs{\frac{\epsilon'}{S}\cdot\sum_{i\in [S]}V^{\hard}_{h_a+2}(\bar{s}_i)\cdot (w_i-v_i)}-\frac{\epsilon'}{S}\cdot\sum_{i\in [S]}\abs{V^{\hard}_{h_a+2}(\bar{s}_i)-V_{h_a+2}(\bar{s}_i)}\cdot \abs{(w_i-v_i)}\tag{by triangle inequality}\\
       &\geq \abs{\frac{\epsilon'}{S}\cdot\sum_{i\in [S]}V^{\hard}_{h_a+2}(\bar{s}_i)\cdot (w_i-v_i)}-\frac{2\epsilon'}{S}\cdot\sum_{i\in [S]}\abs{V^{\hard}_{h_a+2}(\bar{s}_i)-V_{h_a+2}(\bar{s}_i)}\notag\\
       &\geq \frac{H\epsilon'}{16}-2\lep,\label{eq:offlime_lb_final_2}
   \end{align}
   where the second line is by the triangle inequality and the last line comes from \Eqref{eq:online_lb_eq_1}and \Eqref{eq:online_lb_eq_8}. Combining \Eqref{eq:finalbound_lb_offline_1} and \Eqref{eq:offlime_lb_final_2}, we complete the proof.

   \end{proof}
\def \kA{\mathfrak{A}}
   
   \subsection{Proof of Theorem \ref{thm:lower_bound_online}}
   \begin{proof}[Proof of Theorem \ref{thm:lower_bound_online}]
   Our method is similar to the one used for the proof of \citet[Theorem~B.3]{metelli2023theoretical}.
   For any $\epsilon\in (0,1/2]$, $\delta\in (0,1)$,
       we consider an online algorithm $\mathfrak{A}$ such that for any IRL problem $(\m,\piE)$, we have
       \begin{align}
       \underset{(\m,\piE),\mathfrak{A}}{\pp}\paren{\DallTheta\paren{\sRR, \hRR}\leq \epsilon}\geq 1-\delta, \label{eq:ollowb_main_proof_1}
      \end{align}
       where $\underset{(\m,\piE),\mathfrak{A}}{\pp}$ denotes the probability measure induced  by executing the algorithm $\kA$ in the IRL problem $(\m ,\piE)$,
 $\sRR$ is the ground truth reward mapping and $\hRR$ is the estimated reward mapping outputted by executing $\kA$ in $(\m ,\piE)$.
        We define the the identification function for any $(a,\mathbf{w})\in \mc{I}\times \overline{\mc{W}}^{{\mc{I}}}$ by
       \$
       \mathbf{\Phi}_{a,\mathbf{w}}\defeq \argmin_{{v}\in \barw} \DallTheta\paren{\RR^{\paren{\wav}},\widehat{\RR};\pp^{\paren{\wav}}},
       \$
       where $\RR^{\paren{\bfw}}$ is the ground truth reward mapping induced by $(\m_{\bfw},\piE)$.
       Let $v^\star=\mathbf{\Phi}_{a,\mathbf{w}}$.
       For any $v\neq v^{\star}\in \lcW$, by definition of $v^{\star}$, we have
       \[
       \DallTheta\paren{\RR^{\paren{\wavv}},\widehat{\RR};\pp^{\paren{\wavv}}}
       \leq \DallTheta\paren{\RR^{\paren{\wav}},\widehat{\RR};\pp^{\paren{\wav}}}.
       \]
      By applying Lemma \ref{lem_low_1}, we obtain that  
      \$
      \frac{H\epsilon'}{16}\leq \DallTheta\paren{\RR^{\paren{\wav}},\widehat{\RR};\pp^{\paren{\wav}}}+7\DallTheta\paren{\RR^{\paren{\wavv}},\widehat{\RR};\pp^{\paren{\wavv}}}
      \leq 8\DallTheta\paren{\RR^{\paren{\wav}},\widehat{\RR};\pp^{\paren{\wav}}}.
      \$
      \cblue
      Next, we set $\epsilon'=\frac{256\epsilon}{H}$ which implies that
      \#\label{eq:online_lb_main_1}
   \frac{H\epsilon' }{16}\geq 16\epsilon .
     \#
     Here, to employ Lemma~\ref{lem_low_1},  we need $\epsilon'\in (0,1/2]$ which is equivalent to $0<\epsilon\leq H/512$.
     \cblack
     Then, it holds that
     \[
     \DallTheta\paren{\RR^{\paren{\wav}},\widehat{\RR};\pp^{\paren{\wav}}}\geq 2\epsilon>\epsilon,
     \]
     which implies that
     \begin{align}
\set{v\neq\mathbf{\Phi}_{a,\mathbf{w}}}\subseteq \set{\DallTheta\paren{\RR^{\paren{\wav}},\widehat{\RR};\pp^{\paren{\wav}}}> \epsilon}.\label{low_main_online_1}
 \end{align}
    By \Eqref{low_main_online_1}, we have the following lower bound for the probability
    \#\label{eq:online_lb_main_2}
    \delta&\geq \sup_{v\in \lcW}\underset{(\m_{\wav},\piE),\mathfrak{A}}{\pp}\paren{\DallTheta\paren{\RR^{\paren{\wav}},\widehat{\RR};\pp^{\paren{\wav}}}> \epsilon}\notag\\
    &\geq\sup_{v\in \lcW}\underset{(\m_{\wav},\piE),\mathfrak{A}}{\pp}\paren{v\neq\mathbf{\Phi}_{a,\mathbf{w}}}\notag\\
    &\geq \frac{1}{|\lcW|}\sum_{v\in \lcW}\underset{(\m_{\wav},\piE),\mathfrak{A}}{\pp}\paren{v\neq\mathbf{\Phi}_{a,\mathbf{w}}},
    \#
   By applying Theorem~\ref{thm:Fano} with $\pp_0=\underset{(\m_{\waz},\piE),\mathfrak{A}}{\pp}$, $\pp_w=\underset{(\m_{\waw},\piE),\mathfrak{A}}{\pp}$, we have
     \#\label{eq:online_lb_main_100}
     \frac{1}{|\lcW|}\sum_{(\m_\wav,\piE),\mathfrak{A}}\paren{v\neq\mathbf{\Phi}_{a,\mathbf{w}}}
     \geq 1-\frac{1}{\log|\barw|}\paren{\frac{1}{|\barw|}\sum_{v\in \barw}D_{\textrm{KL}}(\underset{(\m_{\wav},\piE),\mathfrak{A}}{\pp},\underset{(\m_{\waz},\piE),\mathfrak{A}}{\pp})-\log 2}.
     \#
    Our next step is to bound the KL divergence. Using the same scheme in the proof \citet[Theorem B.3]{metelli2021provably}, we can compute the KL-divergence as follows:
    \#
   &D_{\textrm{KL}}(\underset{(\m_{\wav},\piE),\mathfrak{A}}{\pp},\underset{(\m_{\waz},\piE),\mathfrak{A}}{\pp})\notag\\
    &~=\EE_{(\m_\wav,\piE),\mathfrak{A}}\brac{ \sum_{t=1}^N D_{\textrm{KL}}\paren{\pp^{\paren{\waw}}_{h_t}(\cdot\given s_t, a_t),\pp^{\paren{\waz}}_{h_t}(\cdot\given s_t, a_t)}}\notag\\
    &~\leq \EE_{(\m_\wav,\piE),\mathfrak{A}}\brac{N_{h_a}(s_{i_a},a_{k_a}) } D_{\textrm{KL}}\paren{\pp^{\paren{\paren{\wav}}}_{h_a}(\cdot\given s_{i_a},a_{k_a}),\pp^{\paren{\waz}}_{h_a}(\cdot\given s_{i_a},a_{k_a})}\notag\\
    &~\leq 2(\epsilon')^2\EE_{(\m_\wav,\piE),\mathfrak{A}}\brac{N_{h_a}(s_{i_a},a_{k_a}) },
    \#
    where $N_h(s,a)\defeq\sum_{t=1}^N\indic{(h_t,s_t,a_t)=(h,s,a)}$ for any given $(h,s,a)\in [H]\times\cS\times \cA$ and the last inequality comes from \citet[Lemma E.4]{metelli2021provably}. Combining \Eqref{eq:online_lb_main_2} and \Eqref{eq:online_lb_main_100}, we have 
    \$
    \delta\geq 1-\frac{1}{\log(|\barw|)}\paren{\frac{1}{|\barw|}\sum_{v\in \barw}2(\epsilon')^2\EE_{(\m_\wav,\piE),\mathfrak{A}}\brac{N_{h_a}(s_{i_a},a_{k_a}) }-\log 2}
    \$
    for any $\bfw$.
     It also holds for any $a\in \cI$ that
    \#\label{eq:online_lb_main_99}
    \frac{1}{|\barw|}\sum_{v\in \barw}\EE_{(\m_\wav,\piE),\mathfrak{A}}\brac{N_{h_a}(s_{i_a},a_{k_a}) }
    \geq
    \frac{(1-\delta)\log|\barw|-\log 2}{2(\epsilon')^2}.
    \#
    By summing \Eqref{eq:online_lb_main_99} over all $\mathbf{w}$, we obtain that
    \#\label{eq:online_lb_main_98}
    &\sum_{a\in \mathcal{I}}\frac{1}{|\barw^{\mathcal{I}}|}\sum_{\mathbf{w}\in \barw^{{\mc{I}}}} \frac{1}{|\barw|}\sum_{v\in \barw}\EE_{(\m_\wav,\piE),\mathfrak{A}}\brac{N_{h_a}(s_{i_a},a_{k_a})}\notag\\
    &~=\frac{1}{|\barw^{\mathcal{I}}|}\sum_{\mathbf{w}\in \barw^{{\mc{I}}}}\sum_{a\in \mathcal{I}} \EE_{(\m_\mathbf{w},\piE),\mathfrak{A}}\brac{N_{h_a}(s_{i_a},a_{k_a})}\notag\\
    &~\geq HKA\frac{(1-\delta)\log|\barw|-\log 2}{2(\epsilon')^2}.
    \#
    \cblue
    Hence, there exists a $\mathbf{w}^{\textrm{bad}}\in \overline{\cW}^{\cI}$ such that
    \#
    \EE_{(\m_{\mathbf{w}^{\textrm{bad}}},\pi^{\msf{E}}),\mathfrak{A}}\brac{N}&\geq
    \sum_{a\in \mathcal{I}} \EE_{(\m_{\mathbf{w}^{\textrm{bad}}},\pi^{\msf{E}}),\mathfrak{A}}\brac{N_{h_a}^t(s_{i_a},a_{k_a})}
    \geq 
     HKA\frac{(1-\delta)\log|\barw|-\log 2}{2(\epsilon')^2}\notag\\
     &=
     H^3KA\frac{(1-\delta)\log|\barw|-\log 2}{131072\epsilon^2},
     \#
     where the last line is by $\epsilon'=\frac{\epsilon}{256 H}$.
By taking $\delta=1/3$, we obtain that
\begin{align}
    \EE_{(\m_{\mathbf{w}^{\textrm{bad}}},\pi^{\msf{E}}),\mathfrak{A}}\brac{N}&\geq 
     H^3KA\frac{(1-\delta)\log|\barw|-\log 2}{131072\epsilon^2}=
     H^3KA\frac{2\log|\barw|-3\log 2}{393216\epsilon^2}\notag\\
     &=\Omega\paren{\frac{H^3SKA}{\epsilon^2}}=\Omega\paren{\frac{H^3SA\min\set{S,A}}{\epsilon^2}},
\end{align}
where the last line follows from \Eqref{eq:online_lb_main_1} and $\log|\barw|\geq \frac{S}{10}$.
\cblack
   \end{proof}

\section{Lower bound in the offline setting}\label{appendix:proof for low bound of offline learning}
\subsection{Lower bound of offline IRL problems}

We direct our attention towards the lower bound analysis of the offline IRL problems, particularly in scenarios where $\para=\pV\times \pA$. 
In this case $\cn$ is upper-bounded by $\widetilde{\cO}(S)$, and the corresponding upper bound of the sample complexity becomes $\widetilde{\cO}\left(\frac{C^\star H^4S^2A}{\epsilon^2}\right)$.

Following \citet{metelli2023theoretical} we define the $(\epsilon,\delta)$-PAC algorithm for offline IRL problems for all $\epsilon, \delta\in (0,1)$.

\begin{definition}[$(\epsilon,\delta)$-PAC algorithm for offline IRL problems]
We say an offline IRL algorithm $\mathfrak{A}$ is an $(\epsilon,\delta)$-PAC algorithm for offline IRL problems if for any offline IRL problem $(\cM,\piE,\pib,\pival)$ and any parameter set $\Theta$, with probability $1-\delta$, 
$\mathfrak{A}$ outputs a reward mapping $\hRR$
such that
\[
D^{\pival}_\para(\hRR,\sRR)\leq\epsilon.
\]
\end{definition}
\cblue
\begin{theorem}[Lower bound for offline IRL problems]\label{thm:lower for offline}
Fix $\para=\pV\times\pA$ and let $\mathfrak{A}$ be an $(\epsilon,\delta)$-PAC algorithm for offline IRL problems, where $\delta\leq 1/3$.
Then, there exists  an offline IRL problem $(\m,\piE, \pib, \pival)$ such that, if $H,S\geq 4, A\ge 2, C^\star\geq 2$, there exists an absolute constant $c_0$ such that the sample complexity $N$ is lower bounded by
\[
{N}\geq \frac{c_0 H^2SC^\star\min\set{S,A}}{\epsilon^2}.
\]
where $0<\epsilon\leq \paren{H-2}/1024$.
\end{theorem}
\cblack
The hard instance construction and
the proof of Theorem~\ref{thm:lower for offline} can be found to Section~\ref{app:hard-offline} and  Section~\ref{app:lower_offline}, respectively. Our proof involves a modification of the challenging instance constructed in \citet{metelli2023theoretical}. Specifically, when $S\leq A$, the lower bound scales with ${\Omega}\paren{C^\star S^2}$, matching the $C^\star S^2$ factor dependence observed in the upper bound (Theorem~\ref{thm:offline_main}).

\subsection{Hard instance construction}\label{app:hard-offline}

We consider the \MDPR{} $\m_{\bfw}$ indexed by vector ${\bfw}\in \lcW^{\cI}$, defined in Section~\ref{appendix:proof for lower bound online learning}. We assume $C^\star\geq 2$.
Fix $i^{\star}\in [K]$,
we construct the behavior policy $\pi^{\msf{b}}$ as follows: 
\begin{align}
 \begin{cases}
 &\pib_h(a_0|s_{\mrm{start}})=1\qquad \text{for all}~i\in [K]~\text{and}~h\in [H-1],\\
 &\pi^{\msf{b}}_H(a_i|s_{\mrm{start}})=\frac{1}{K}\qquad \text{for all}~i\in [K],\\
 &\pib_{H+1}(a_0|{s}_i)=1\qquad \text{for all}~i\neq i^{\star},\\
 &\pib_{H+1}(a_0|{s}_{i^\star})=1-\frac{1}{C^\star},\qquad\pib_{H+1}(a_1|{s}_{i^\star})=\frac{1}{C^\star},\\
 &\pib_{h}(a_0|\bar{s}_i)=1\qquad \text{for all}~i\in [S]~\text{and}~h\ge H+2.\\
 \end{cases}
\end{align}
And evaluation policy $\pival$ is defined by
\begin{align}
 \begin{cases}
 &\pival_h(a_0|s_{\mrm{start}})=1\qquad \text{for all}~h\in [H-1],\\
 &\pival_H(a_{i^\star}|s_{\mrm{start}})=1,\\
 &\pival_{H+1}(a_0|{s}_i)=1\qquad \text{for all}~i\neq i^{\star},\\
 &\pival_{H+1}(a_1|{s}_{i^\star})=1,\\
 &\pival_{h}(a_0|\bar{s}_i)=1\qquad \text{for all}~i\in [S]~\text{and}~h\ge H+2.
 \end{cases}
\end{align}
For all $\bfw\in \lcW^{\cI}$, we can show that $\pival$ has $C^\star$-concentrability in $\cM_{\mathbf{w}}$.
\begin{lemma}\label{lem:lowbb_off}
Suppose that $\epsilon'\in (0,1/2]$. 
    For any $\mathbf{w}\in \lcW^{\cI}$, it holds that
    \[
    \sum_{(h,s,a)\in [2H+2]\times \cS\times \cA} \frac{d^{\pival}(s,a)}{d^{\pib}(s,a)}\leq 3C^\star (H+2)S.
    \]
\end{lemma}
\begin{proof}
    By the construction of behavior policy $\pib$, we have
    \[
   \supp\paren{ d^{\pival}_h(\cdot,\cdot)}\subseteq\set{(s_{\mrm{start}}, a_0), (s_{\mrm{start}}, a_{k^\star}) ,(s_{i^\star}, a_1), (\bar{s}_1, a_0),\ldots, ({\bar{s}_S}, a_0)}.
    \]
    Since $\pib_h=\pival_h$ for all $h\in [H-1]$, then 
    \begin{align}  d^{\pib}_h(s_{\mrm{start}},a_0)=d^{\pival}_h(s_{\mrm{start}},a_0)=1\label{eq:lb_offline_1}
    \end{align}
    for all $h\in [H-1]$.

    At stage $h=H$, we have
    \begin{align}
    d^{\pib}_H(s_{\mrm{start}},a_{i^{\star}})=\frac{1}{K},\qquad d^{\pival}_H(s_{\mrm{start}},a_{i^{\star}})=1.\label{eq:lb_offline_2}
    \end{align}
    At stage $h= H+1$, we have
    \begin{align}
    d^{\pib}_{H+1}(s_{i^{\star}},a_1)=\frac{1}{C^\star K},\qquad d^{\pival}_{H+1}(s_{i^{\star}},a_1)=1.\label{eq:lb_offline_3}
    \end{align}
    At stage $h\in \set{H+2,\ldots,2H+2}$, by direct computation, we obtain that
    \begin{align}
    d^{\pib}_{h}(\bar{s}_{j},a_0)=\frac{C^{\star}K-1}{C^{\star}SK}+\frac{1+\epsilon'w^{(i^{\star},j,1)}_{H}}{C^{\star}SK},\qquad d^{\pival}_{h}(\bar{s}_{j},a_1)=\frac{1+\epsilon'w^{(i^{\star},j,1)}_{H}}{S},
    \end{align}
    for all $j\in [S]$.
    Since $0<\epsilon\leq 1/2$ and $C^{\star}\geq 1 $, we have
    \begin{align}
        d^{\pib}_{h}(\bar{s}_{j},a_0)&=\frac{C^{\star}K-1}{C^\star SK}+\frac{1+\epsilon'w^{(i^{\star},j,1)}_{H}}{C^{\star}SK}\notag\\
        &\geq \frac{C^{\star}K-1}{C^\star SK}+\frac{1}{2C^{\star}SK}=\frac{1}{S}(1-\frac{1}{2C^{\star}K})\ge \frac{1}{2S}\label{eq:of_lowb_lem_11}
    \end{align}
    and 
    \begin{align}
        d^{\pival}_{h}(s_{i^{\star}},a_1)=\frac{1+\epsilon'w^{(i^{\star},j,1)}_{H+1}}{S}\le \frac{3}{2S},\label{eq:of_lowb_lem_12}
    \end{align}
    for all $h\ge H+2$.
    By \Eqref{eq:of_lowb_lem_11} and \eqref{eq:of_lowb_lem_11}, we obtain that
    \begin{align}\label{eq:of_lowb_lem_13}
        \frac{d^{\pival}_{h}(\bar{s}_{j},a_0)}{d^{\pib}_{h}(\bar{s}_{j},a_0)}\leq 3,
    \end{align}
    for all $h\ge H+2$.

    Combining \Eqref{eq:lb_offline_1}, \Eqref{eq:lb_offline_2} and \Eqref{eq:lb_offline_3}, we have
    \begin{align}
        \sum_{h=1}^{2H+2}\sum_{(s,a)\in \cS\times \cA} \frac{d^{\pival}(s,a)}{d^{\pib}(s,a)}&=\sum_{h\in [H-1]}\frac{d^{\pival}_{h}(s_{\mrm{start}},a_0)}{d^{\pib}_{h}(s_{\mrm{start}},a_0)}
        +\frac{d^{\pival}_{H}(s_{\mrm{start}},a_{i^\star})}{d^{\pival}_{H}(s_{\mrm{start}},a_{i\star})}\\
        &+
        \frac{d^{\pival}_{H+1}(s_{i^\star},a_1)}{d^{\pib}_{H+1}(s_{i^\star},a_1)}
        +
        \sum_{h\geq H+2}\sum_{i\in [S]}\frac{d^{\pival}_{h}(\bar{s}_{i},a_0)}{d^{\pib}_{h}(\bar{s}_{i},a_0)}\\
        &=H-1+K+C^\star K+\sum_{h\geq H+2}\sum_{i\in [S]}\frac{d^{\pival}_{h}(\bar{s}_{i},a_0)}{d^{\pib}_{h}(\bar{s}_{i},a_0)}\\
        &\leq H-1+K+C^\star K+3(H+1)S\leq C^{\star}(2H+2)(2S+1),
    \end{align}
    where the last second inequality is by \Eqref{eq:of_lowb_lem_13} and the last inequality is by $C^\star\geq 2$. This completes the proof.
\end{proof}
Lemma~\ref{lem:lowbb_off} demonstrate that $\pib$ and $\pival$ satisfies $C^\star$-concentrability (Assumption~\ref{ass:off_1}) in any $\cM_{\mathbf{w}}$.
\paragraph{Notations.}
   To distinguish with different \MDPR{}s, we still use $V^{\pi}_h\paren{\cdot; r, \pp^{\pbfW}}$ to denote the value function of $\pi$ in MDP $\m_{\bfw}\cup r$. Given two rewards $r$ $r'$ and $\bfw\in \lcW^{\cI}$, we define $\dpival(r,r';\pp^{\pbfW})$ by:
   \[
   \dpival(r,r';\pp^{\pbfW})\defeq \sup_{\pi,h\in[H]}\EE_{\pp^{\pbfW}} \abs{V^{\pival}_h(s_h;r,\pp^{\pbfW})-V^{\pival}_h(s_h;{r'},\pp^{\pbfW})}.
   \]
   Correspondingly, given a parameter set $\para$, two reward mappings $\RR$, $\RR'$, we define $\Dpivalt(\RR,\RR';\pp^{\pbfW})$ by
   \[
   \Dpivalt(\RR,\RR';\pp^{\pbfW})\defeq
   \sup_{(V,A)\in \para} \dpival\paren{\RR\paren{V,A},\RR'\paren{V,A};\pp^{\pbfW}}.
   \]
   In this section, we only consider the case that $\Theta=\pV\times \pA$.

\begin{lemma}\label{lem_low_off_1} 
   Given any $\mathbf{w}\in \cWI$, $w\neq v\in \lcW$, and $i^\star\in [K]$. Let 
   $\RR^{\paren{\waw}}$, $\RR^{\paren{\wav}}$ be the ground truth reward mappings induced by ${\m_{\waw}}$ , ${\m_{\wav}}$ where $a=(i^{\star
   },H+1,1)\in \cI$. Set 
$\para=\pV\times \pA$. For any rewarding mapping $\RR$ and $\epsilon'\in (0,1/2]$, we have
   \[
   7\Dpivalt\paren{\RR^{\paren{\waw}},\RR;\pp^{\waw}}+\Dpivalt\paren{\RR^{\paren{\wav}},\RR;\pp^{\wav}}\geq \frac{H\epsilon'}{16}.
   \]
   \end{lemma}
   \begin{proof}
   We consider similar construction of bad parameter  $V^{\mrm{bad}}$, $A^{\mrm{bad}}$ in the Proof of Lemma~\ref{lem_low_1}. To summarize, $\paren{V^{\hard},A^{\hard}}$ is given by
       \begin{itemize}
       \item We set $A^{\hard}_h(s,a)=0$ for all $(h,s,a)\in [2H+2]\times\cS\times \cA$.
       \item We set $V^{\hard}_h$ by
\begin{align}
V^{\hard}_h(s)\defeq
          \begin{cases}
          \frac{(2H+2-h)\cdot\paren{w_i-v_i}}{2}&\text{if}~s=\bar{s}_i, h=H+2,\\
           0&\text{otherwise}.
       \end{cases}
\end{align}
   \end{itemize}
   Similarly, we define $r^{\hard}_w$, $r^{\hard}_v$ and $r$ by
   \[
   r^{\hard}_w\defeq \RR^{\paren{\waw}}\paren{V^{\hard},A^{\hard}},\qquad r^{\hard}_w\defeq \RR^{\paren{\wav}}\paren{V^{\hard},A^{\hard}},
   \qquad r\defeq \RR\paren{V^{\hard},A^{\hard}}.
   \]
   By definition of $\RR^{\paren{\waw}}$, $\RR^{\paren{\wav}}$, we have
   \begin{align}
       \abs{r^{\hard}_{w,H+1}(s_{i^\star},a_1)-r^{\hard}_{v,H+1}(s_{i^\star},a_1)}&=\abs{\brac{\paren{\pp^{\paren{\waw}}_{H+2}-\pp^{\paren{\wav}}_{H+2}}V^{\hard}_{H+1}}(s_{i^\star},a_1)}\notag\\
       &=\epsilon'\cdot\abs{\sum_{i\in [S]}\frac{\paren{w_i-v_i}V^{\hard}_{H+2}}{S}}\notag\\
       &=\frac{H\epsilon'}{2S}\cdot\sum_{i\in [S]}{\paren{w_i-v_i}^2}\geq \frac{H\epsilon'}{16},\label{eq:main_low_B_offline}
   \end{align}
   where the last inequality follows from \Eqref{eq:online_lb_corr}.
   By definition of $\Dpivalt$, we have
 \begin{align}
  \Dpivalt\paren{\RR^{\paren{\waw}},\RR;\pp^{\paren{\waw}} }&\geq \dpival\paren{r^{\hard}_w,r;\pp^{\paren{\waw}}}\notag\\
  &\geq\EE_{\pp^{\paren{\waw}},\pival}\abs{V^{\pival}_{H+2}(s;r^{\hard}_w,\pp^{\paren{\waw}})-V^{\pival}_{H+2}(s;r, \pp^{\paren{\waw}})}\notag\\
  &=\sum_{i\in [S]} \frac{1+\epsilon'\cdot w_i}{S}\abs{V^{\pival}_{H+2}(\bar{s}_i;r^{\hard}_w,\pp^{\paren{\waw}})-V^{\pival}_{H+2}(\bar{s}_i;r, \pp^{\paren{\waw}})}\notag\\
          &\geq \sum_{i\in [S]} \frac{1}{2S}\abs{V^{\pival}_{H+2}(s;r^{\hard}_w,\pp^{\paren{\waw}})-V^{\pival}_{H+2}(s;r, \pp^{\paren{\waw}})},\label{eq:lb_offline_ttt}
 \end{align}
 where the last line is due to $\epsilon'\in (0,1/2]$.
By construction of $\pival$, in \MDPR{} $\m_{\wav}$, the visiting probability $d^{\pival}_{H+1}$ is given by 
\[
d^{\pival}_{H+1}\paren{s_{i^{\star}},a_1;\pp^{(\waw)}}=1.
\]
 
 For $\Dpivalt\paren{\RR^{\paren{\wav}},\RR;\pp^{\paren{\wav}} }$, we also have
 \begin{align}
&\Dpivalt\paren{\RR^{\paren{\wav}},\RR;\pp^{\paren{\wav}} }\geq \dpival\paren{r^{\hard}_v,r;\pp^{\paren{\wav}}}\notag\\
     &~\ge  \EE_{\pp^{\paren{\wav}}, \pival}\abs{V^{\pival}_{H+1}(s;r^{\hard}_v,\pp^{\paren{\wav}})-V^{\pival}_{H+1}(s;r, \pp^{\paren{\wav}})}\notag\\
   &~=\abs{V^{\pival}_{H+1}(s_{i^\star};r^{\hard}_v,\pp^{\paren{\wav}})-V^{\pival}_{H+1}(s_{i^\star};r, \pp^{\paren{\wav}})}\notag\\
   &~=\Bigg|r^{\hard}_{v,H+1}(s_{i^\star},a_1)-r_{H+1}(s_{i^\star},a_1)\notag\\
   &-\sum_{i\in[S]}\pp^{\paren{\wav}}_{H+1}(\bar{s}_i|s_{i^\star},a_1)\cdot\paren{V^{\pival}_{H+2}(\bar{s}_i;r^{\hard}_{v},\pp^{\paren{\wav}})-V^{\pival}_{H+2}(\bar{s}_i;r, \pp^{\paren{\wav}})}\Bigg|\notag\\
   &~\geq \abs{r^{\hard}_{v,H+1}(s_{i^\star},a_1)-r_{H+1}(s_{i^\star},a_1)}\notag\\
   &-\sum_{i\in [S]}\frac{1+\epsilon'\cdot v_i}{S}\cdot\abs{{V^{\pival}_{H+2}(\bar{s}_i;r^{\hard}_{v},\pp^{\paren{\wav}})-V^{\pival}_{H+2}(\bar{s}_i;r, \pp^{\paren{\wav}})}},\label{eq:off_low_b_lem_2_2}
 \end{align}
 where the second last line is by the bellman equation and the last line is due to the triangle inequality.
 Since $\pp^{\paren{\waw}}_h=\pp^{\paren{\waw}}_h$ and $r^{\hard}_{w,h}=r^{\hard}_{v,h}$ for all $h\geq H+2$, we have
 \#\label{eq:off_low_b_lem_2_1}
 V^{\pival}_{H+2}(\bar{s}_i;r, \pp^{\paren{\wav}})=V^{\pival}_{H+2}(\bar{s}_i;r, \pp^{\paren{\waw}}),\qquad V^{\pival}_{H+2}(\bar{s}_i;r^{\hard}_v, \pp^{\paren{\wav}})=V^{\pival}_{H+2}(\bar{s}_i;r^{\hard}_w, \pp^{\paren{\waw}}).
 \#
Apply \Eqref{eq:off_low_b_lem_2_1} to \Eqref{eq:off_low_b_lem_2_2}, we have
\begin{align}
&\Dpivalt\paren{\RR^{\paren{\wav}},\RR;\pp^{\paren{\wav}} }\notag\\
&~\geq \abs{r^{\hard}_{v,H+1}(s_{i^\star},a_1)-r_{H+1}(s_{i^\star},a_1)}\notag\\
&-\sum_{i\in [S]}\frac{1+\epsilon'\cdot v_i}{S}\cdot\abs{{V^{\pival}_{H+2}(\bar{s}_i;r^{\hard}_{v},\pp^{\paren{\wav}})-V^{\pival}_{H+2}(\bar{s}_i;r, \pp^{\paren{\wav}})}}\notag\\
&~= \abs{r^{\hard}_{v,H+1}(s_{i^\star},a_1)-r_{H+1}(s_{i^\star},a_1)}\notag\\
&-\sum_{i\in [S]}\frac{1+\epsilon'\cdot v_i}{S}\cdot\abs{{V^{\pival}_{H+2}(\bar{s}_i;r^{\hard}_{w},\pp^{\paren{\waw}})-V^{\pival}_{H+2}(\bar{s}_i;r, \pp^{\paren{\waw}})}}\notag\\
&~\geq \abs{r^{\hard}_{v,H+1}(s_{i^\star},a_1)-r_{H+1}(s_{i^\star},a_1)}\notag\\
&-\sum_{i\in [S]}\frac{3}{2S}\cdot\abs{{V^{\pival}_{H+2}(\bar{s}_i;r^{\hard}_{w},\pp^{\paren{\waw}})-V^{\pival}_{H+2}(\bar{s}_i;r, \pp^{\paren{\waw}})}}\notag\\
&~\geq \abs{r^{\hard}_{v,H+1}(s_{i^\star},a_1)-r_{H+1}(s_{i^\star},a_1)}-3\Dpivalt\paren{\RR^{\paren{\waw}},\RR;\pp^{\paren{\waw}} }\label{eq:off_low_b_lem_2_5},
\end{align}
where the last second inequality comes from $\epsilon'\in (0,1/2]$ and the last inequality comes from \Eqref{eq:lb_offline_ttt}.

We next bound $\abs{r^{\hard}_{w,H+1}(s_{i^\star},a_1)-r_{H+1}(s_{i^\star},a_1)}$ by $\Dpivalt\paren{\RR^{\paren{\waw}},\RR;\pp^{\paren{\waw}} }$.
\begin{align}
    &\Dpivalt\paren{\RR^{\paren{\waw}},\RR;\pp^{\paren{\waw}} }\geq \dpival\paren{r^{\hard}_w,r;\pp^{\paren{\waw}}}\notag\\
     &~\ge  \EE_{\pp^{\paren{\waw}}, \pival}\abs{V^{\pival}_{H+1}(s;r^{\hard}_v,\pp^{\paren{\waw}})-V^{\pival}_{H+1}(s;r, \pp^{\paren{\waw}})}\notag\\
   &~=\abs{V^{\pival}_{H+1}(s_{i^\star};r^{\hard}_w,\pp^{\paren{\waw}})-V^{\pival}_{H+1}(s_{i^\star};r, \pp^{\paren{\waw}})}\notag\\
   &~=\Bigg|r^{\hard}_{w,H+1}(s_{i^\star},a_1)-r_{H+1}(s_{i^\star},a_1)\notag\\
   &-\sum_{i\in[S]}\pp^{\paren{\waw}}_{H+1}(\bar{s}_i|s_{i^\star},a_1)\cdot\paren{V^{\pival}_{H+2}(\bar{s}_i;r^{\hard}_{w},\pp^{\paren{\waw}})-V^{\pival}_{H+2}(\bar{s}_i;r, \pp^{\paren{\waw}})}\Bigg|\notag\\
   &~\geq \abs{r^{\hard}_{w,H+1}(s_{i^\star},a_1)-r_{H+1}(s_{i^\star},a_1)}\notag\\
   &-\sum_{i\in [S]}\frac{1+\epsilon'\cdot w_i}{S}\cdot\abs{{V^{\pival}_{H+2}(\bar{s}_i;r^{\hard}_{w},\pp^{\paren{\waw}})-V^{\pival}_{H+2}(\bar{s}_i;r, \pp^{\paren{\waw}})}}\notag\\
   &~\geq \abs{r^{\hard}_{w,H+1}(s_{i^\star},a_1)-r_{H+1}(s_{i^\star},a_1)}\notag\\
   &-\sum_{i\in [S]}\frac{3}{2S}\cdot\abs{{V^{\pival}_{H+2}(\bar{s}_i;r^{\hard}_{w},\pp^{\paren{\waw}})-V^{\pival}_{H+2}(\bar{s}_i;r, \pp^{\paren{\waw}})}}\notag\\
   &~\geq \abs{r^{\hard}_{w,H+1}(s_{i^\star},a_1)-r_{H+1}(s_{i^\star},a_1)}-3\Dpivalt\paren{\RR^{\paren{\waw}},\RR;\pp^{\paren{\waw}} },\label{eq:lb_offline_tsd}
\end{align}
where the last second inequality comes from $\epsilon'\in (0,1/2]$ and  the last inequality is by \Eqref{eq:lb_offline_ttt}. \Eqref{eq:lb_offline_tsd} is equivalent to
\begin{align}\label{eq:off_low_b_lem_2_3}
4\Dpivalt\paren{\RR^{\paren{\waw}},\RR;\pp^{\paren{\waw}} }\geq \abs{r^{\hard}_{w,H+1}(s_{i^\star},a_1)-r_{H+1}(s_{i^\star},a_1)}.
\end{align}
Combining \Eqref{eq:off_low_b_lem_2_5} and \Eqref{eq:off_low_b_lem_2_3}, we conclude that
\begin{align}
&7\Dpivalt\paren{\RR^{\paren{\waw}},\RR;\pp^{\paren{\waw}} }+\Dpivalt\paren{\RR^{\paren{\wav}},\RR;\pp^{\paren{\wav}} }\notag\\
&~\geq \abs{r^{\hard}_{w,H+1}(s_{i^\star},a_1)-r_{H+1}(s_{i^\star},a_1)} +\abs{r^{\hard}_{v,H+1}(s_{i^\star},a_1)-r_{H+1}(s_{i^\star},a_1)}\notag\\
&~\geq \abs{r^{\hard}_{v,H+1}(s_{i^\star},a_1)-r^{\hard}_{w,H+1}(s_{i^\star},a_1)}\geq \frac{H\epsilon'}{16},
\end{align}
where the last inequality comes from \Eqref{eq:main_low_B_offline}.
This completes the proof.

\end{proof}
\subsection{Proof for Theorem~\ref{thm:lower for offline}}\label{app:lower_offline}
Our proof is similar to the proof of Theorem~\ref{thm:lower_bound_online} in Section~\ref{appendix:proof for lower bound online learning}. 
\def \kA{\mathfrak{A}}
   \begin{proof}[Proof of Theorem \ref{thm:lower for offline}]
   For any $\epsilon\in (0,1/2]$, $\delta\in (0,1)$,
       We consider an offline IRL algorithm $\mathfrak{A}$ such that for any IRL problem $(\m,\piE)$, we have
       \begin{align}
       \underset{(\m,\piE),\mathfrak{A}}{\pp}\paren{\Dpivalt\paren{\sRR, \hRR}\leq \epsilon}\geq 1-\delta, \label{eq:offlowb_main_proof_1}
      \end{align}
       where $\underset{(\m,\piE),\mathfrak{A}}{\pp}$ denotes the probability measure induced  by executing the algorithm $\kA$ in the IRL problem $(\m ,\piE)$,
 $\sRR$ is the ground truth reward mapping and $\hRR$ is the estimated reward mapping outputted by executing $\kA$ in $(\m ,\piE)$.
        Fix $i^\star\in [S]$, We define the the identification function for any $\bfw\in \lcW$  by
       \$
       \mathbf{\Phi}_{\mathbf{w}}\defeq \argmin_{{v}\in \barw} \Dpivalt\paren{\RR^{\paren{\wav}},\widehat{\RR};\pp^{\paren{\wav}}},
       \$
       where $a=(i^\star,H+1,1)$, $\RR^{\paren{\wav}}$ is the ground truth reward mapping induced by $(\m_{\wav},\piE)$.
       Let $v^\star=\mathbf{\Phi}_{a,\mathbf{w}}$.
       For any $v\neq v^{\star}\in \lcW$, by definition of $v^{\star}$, we have
       \[
       \Dpivalt\paren{\RR^{\paren{\wavv}},\widehat{\RR};\pp^{\paren{\wavv}}}
       \leq \Dpivalt\paren{\RR^{\paren{\wav}},\widehat{\RR};\pp^{\paren{\wav}}}.
       \]
      By applying Lemma \ref{lem_low_1}, we obtain that  
      \$
      \frac{H\epsilon'}{16}\leq \Dpivalt\paren{\RR^{\paren{\wav}},\widehat{\RR};\pp^{\paren{\wav}}}+7\Dpivalt\paren{\RR^{\paren{\wavv}},\widehat{\RR};\pp^{\paren{\wavv}}}
      \leq 8\DallTheta\paren{\RR^{\paren{\wav}},\widehat{\RR};\pp^{\paren{\wav}}}.
      \$
      \cblue
      Next, we set $\epsilon'=\frac{256\epsilon}{H}$ which implies that
      \#
   \frac{H\epsilon' }{16}\geq 16\epsilon .
     \#
     Here, to employ Lemma~\ref{lem_low_off_1},  we need $\epsilon'\in (0,1/2]$ which is equivalent to $0<\epsilon\leq H/512$.
     \cblack
     Then, it holds that
     \[
     \Dpivalt\paren{\RR^{\paren{\wav}},\widehat{\RR};\pp^{\paren{\wav}}}\geq 2\epsilon> \epsilon,
     \]
     which implies that
     \[
     \set{v\neq\mathbf{\Phi}_{\mathbf{w}}}\subseteq \set{\Dpivalt\paren{\RR^{\paren{\wav}},\widehat{\RR};\pp^{\paren{\wav}}}\geq \epsilon}.
     \]
    By \Eqref{eq:offlowb_main_proof_1}, we have the following lower bound for the probability
    \#\label{eq:offline_lb_main_2}
    \delta&\geq \sup_{v\in \lcW}\underset{(\m_{\wav},\piE),\mathfrak{A}}{\pp}\paren{\Dpivalt\paren{\RR^{\paren{\wav}},\widehat{\RR};\pp^{\paren{\wav}}}\geq \epsilon}\notag\\
    &\geq\sup_{v\in \lcW}\underset{(\m_{\wav},\piE),\mathfrak{A}}{\pp}\paren{v\neq\mathbf{\Phi}_{\mathbf{w}}}\notag\\
    &\geq \frac{1}{|\lcW|}\sum_{v\in \lcW}\underset{(\m_{\wav},\piE),\mathfrak{A}}{\pp}\paren{v\neq\mathbf{\Phi}_{\mathbf{w}}}.
    \#
   By applying Theorem~\ref{thm:Fano} with $\pp_0=\underset{(\m_{\waz},\piE),\mathfrak{A}}{\pp}$, $\pp_w=\underset{(\m_{\waw},\piE),\mathfrak{A}}{\pp}$, we have
      \#\label{eq:offline_lb_main_100}
      \frac{1}{|\lcW|}\sum_{v\in \lcW}\underset{(\m_{\waw},\piE),\mathfrak{A}}{\pp}\paren{v\neq\mathbf{\Phi}_{\mathbf{w}}}
     \geq 1-\frac{1}{\log|\barw|}\paren{\frac{1}{|\barw|}\sum_{v\in \barw}D_{\textrm{KL}}\paren{\underset{(\m_{\wav},\piE),\mathfrak{A}}{\pp},\underset{(\m_{\waz},\piE),\mathfrak{A}}{\pp}}-\log 2}.
      \#
    Our next step is to bound the KL divergence. Using the same scheme in the proof \citet[Theorem B.3]{metelli2021provably}, we can compute the KL-divergence as follows:
    \#
   &D_{\textrm{KL}}(\underset{(\m_{\wav},\piE),\mathfrak{A}}{\pp},\underset{(\m_{\waz},\piE),\mathfrak{A}}{\pp})\notag\\
    &~=\EE_{(\m_\wav,\piE),\mathfrak{A}}\brac{ \sum_{t=1}^N D_{\textrm{KL}}\paren{\pp^{\paren{\wav}}{h_t}(\cdot\given s_t, a_t),\pp^{\paren{\waz}}_{h_t}(\cdot\given s_t, a_t)}}\notag\\
    &~\leq \EE_{(\m_\wav,\piE),\mathfrak{A}}\brac{N_{h_a}(s_{i_a},a_{k_a}) } D_{\textrm{KL}}\paren{\pp^{\paren{\wav}}_{H+1}(\cdot\given s_{i^\star},a_{1}),\pp^{\paren{\waz}}_{H+1}(\cdot\given s_{i^\star},a_{1})}\notag\\
    &~\leq 2(\epsilon')^2\EE_{(\m_\wav,\piE),\mathfrak{A}}\brac{N_{H+1}(s_{i^\star},a_{1}) },
    \#
    where $N_h(s,a)\defeq\sum_{t=1}^N\indic{(h_t,s_t,a_t)=(h,s,a)}$ for any given $(h,s,a)\in [2H+2]\times\cS\times \cA$ and the last inequality comes from \citet[Lemma E.4]{metelli2021provably}. Combining \Eqref{eq:offline_lb_main_2} and \eqref{eq:offline_lb_main_100}, we have 
    \$
    \delta\geq 1-\frac{1}{\log(|\barw|)}\paren{\frac{1}{|\barw|}\sum_{v\in \barw}2(\epsilon')^2\EE_{(\m_\wav,\piE),\mathfrak{A}}\brac{N_{h_a}(s_{i^\star},a_{1}) }-\log 2}
    \$
    for any $\bfw$.
     It also holds that
    \#\label{eq:offline_lb_main_99}
    \frac{1}{|\barw|}\sum_{v\in \barw}\EE_{(\m_\wav,\piE),\mathfrak{A}}\brac{N_{H+1}(s_{i^\star},a_{1}) }
    \geq
    \frac{(1-\delta)\log|\barw|-\log 2}{2(\epsilon')^2}.
    \#
    \cblue
    Hence, there exists a $\mathbf{w}^{\textrm{hard}}\in \lcW$ such that
    \#
    \EE_{(\m_{\mathbf{w}^{\textrm{hard}}},\piE),\mathfrak{A}}\brac{N_{H+1}(s_{i^\star},a_1)}
    &\geq 
     \frac{(1-\delta)\log|\barw|-\log 2}{2(\epsilon')^2}.
     \#
     By taking $\delta=1/3$, we have 
     \begin{align}
         \EE_{(\m_{\mathbf{w}^{\textrm{hard}}},\piE),\mathfrak{A}}\brac{N_{H+1}(s_{i^\star},a_1)}\geq 
     \frac{(1-\delta)\log|\barw|-\log 2}{2(\epsilon')^2}= 
     \frac{2\log|\barw|-3\log 2}{6(\epsilon')^2}=\Omega\paren{\frac{H^2S}{\epsilon^2}},
     \end{align}
where the last equality follows from $\epsilon'=\frac{256\epsilon}{H}$ and $\log\abs{\lcW}\geq \frac{S}{10}$. By construction of $\pib$, it holds that
     $
     N_{H+1}(s_{i^\star},a_1)\sim \Bin\paren{K,\frac{1}{C^\star K}},
     $
     which implies that
     \[
     \EE_{(\m_{\mathbf{w}^{\textrm{hard}}},\piE),\mathfrak{A}}\brac{N}\geq C^\star K\cdot\Omega\paren{\frac{H^2S}{\epsilon}}=\Omega\paren{\frac{C^\star H^2SK}{\epsilon^2}}=\Omega\paren{\frac{C^\star H^2S\min\set{S,A}}{\epsilon^2}}.
     \]
     \cblack
   \end{proof}

\section{Transfer learning}
\label{appendix:proof for transfer learning}
In this section, we explore the application of IRL in the context of transfer learning. Specifically, we apply the rewards learned by Algorithm~\ref{alg:RLPfull} and Algorithm~\ref{alg:RLEF} to do RL in a \emph{different} environment.

To distinguish different environments, given a transition dynamics $\pp$ and policy $\pi$, we introduce the following notations: $\set{d^{\pp,\pi}_h}_{h\in [H]}$ represents the visitation probability induced by $\pp$ and $\pi$, $d^{\pp,\pi}$ signifies the metric $d^{\pi}$ evaluated on $\pp$, and correspondingly $D_\Theta^{ \pp,\pi}$ denotes the metric $D_\Theta^{\pi}$ evaluated on $\pp$.

\subsection{Transfer learning between IRL problems}

We introduce the transfer learning setting outlined in \citet{metelli2021provably}, where they consider two IRL problems: $(\cM,\piE)$ (the source IRL problem), $(\cM',\paren{\pi'}^{\sf E})$ (the target IRL problem). Here, $\cM$, $\cM'$ share the same state space and action space, but different dynamics. Suppose that we can learn the source IRL problem and obtain a solution $r$. However, $r$ is not necessarily a solution for $(\cM',\paren{\pi'}^E)$, hence, in order to facilitate the transfer learning, we enforce the following assumption.
\begin{assumption}\label{ass:trans_met}
   If $r$ represents a solution to the source IRL problem $(\cM,\piE)$, it also stands as a solution to the target IRL problem $(\cM',\paren{\pi'}^E)$.
\end{assumption}
Assumption~\ref{ass:trans_met} is also supposed in \citet{metelli2021provably}.
We remark that in numerous practical scenarios, Assumption~\ref{ass:trans_met} may not be precisely met, but could be approximated: when the two IRL problems are very close\footnote{Here, we say $(\cM,\piE)$ and $(\cM',\paren{\pi'}^E)$ are very close if the transitions of the two IRL problems are close under certain metric.} to each other, the solutions to the two IRL problems exhibit a high degree of similarity..

\subsection{Transfer learning between two \MDPR{}s}
In this section, we consider a more general setting, where we focus solely on a source IRL problem and a target \MDPR{}.

We consider two \MDPR{}s $\cM=(\cS,\cA, H,\pp)$ (source \MDPR), $\cM'=(\cS,\cA, H,\pp')$ (the target \MDPR{}), which share the same state space and action space, but different dynamics, and an expert policy $\piE$.  Let $\sRR$ be the ground truth reward mapping of the IRL problem $(\cM,\piE)$ and $\hRR$ be the estimated reward mapping learned from $(\cM,\piE)$. In  this setting, we evaluate $\hRR$ in $\cM'$. 

As we see in Section~\ref{sec:inro}, Inverse reinforcement learning (IRL) and behavioral cloning (BC) are highly related.  As mentioned in \cite{metelli2021provably}, transfer learning makes IRL more powerful than BC, and a lot of literature has used IRL to do transfer learning \citep{syed2007game,metelli2021provably,abbeel2004apprenticeship,fu2017learning,levine2011nonlinear}.

Inspired by the single policy concentrability of policies, we propose the following transferability assumption.
\begin{definition}[Weak transferability]\label{def:week-Transferability}
Given transitions $(\pp, \pp')$, and policies $(\pi, \pi')$, we say $(\pp', \pi')$ is $C^{\msf{wtran}}$-weakly transferable from $(\pp, \pi)$ if it holds that
    \[
    \sup_{s,a} \frac{d_h^{\pp',\pi'}(s,a)}{d_h^{\pp, \pi}(s,a)}\leq \wCtran. 
    \]
\end{definition}
\begin{definition}[Transferability]\label{def:Transferability}
Given source and target transitions $\pp$, $\pp'$, and target policy $\pi'$, we say $\pi'$ is $C^{\msf{tran}}$-transferable from $\pp$ to $\pp'$ if it holds that
    \[
    \inf_{\pi}\sup_{s,a} \frac{d_h^{\pp',\pi'}(s,a)}{d_h^{\pp, \pi}(s,a)}\leq \Ctran.
    \]
\end{definition}

 We remark that given a policy $\pi$ and a dynamics $(\pp,\pp')$, transferability measures how hard one can learn the states 
$\pi'$ frequently goes to in $\pp$ in a different environment $\pp'$ while given a policy pair $(\pi,\pi')$ and a dynamics pair $(\pp,\pp')$, weak-transferability measures how hard one learn the states $\pi$ frequently visits in $\pp$ via policy  $\pi'$ in $\pp'$. Without transferability, we can't obtain information on the policy of interest in the target MDP, which makes transfer learning hard to perform.

\subsection{Theoretical guarantee}
We then present the main theorems in this section.
\begin{theorem}[Transfer learning in the offline setting]\label{thm:trans_offline}
Suppose $(\pp',\pival)$ is $\wCtran$-weakly transferable from $(\pp,\pib)$ (Definition~\ref{def:week-Transferability}).
In addition, we assume $\piE$ is well-posed (Definition~\ref{def:well_pose}) when we receive feedback \textrm{in option} $1$. 
Then for both options, with probability at least $1-\delta$, \RLP{} (Algorithm~\ref{alg:RLPfull}) outputs a reward mapping $\hRR$ such that 
\begin{align*}
    D^{\pp',\pival}_{\Theta}\paren{\sRR,\hRR}\leq \epsilon, \brac{\hRR(V,A)}_{h}(s,a)\leq \brac{\sRR(V,A)}_{h}(s,a)
\end{align*}
for all $(V,A)\in \para$ and $(h,s,a)\in [H]\times \cS\times \cA$, as long as the number of episodes
\begin{align*}
  K\geq \widetilde{\cO}\paren{\frac{H^4S\wCtran A\cns}{\epsilon^2}+\frac{ H^2S\wCtran A\eta}{\epsilon}}. 
\end{align*}
Above, $\cns\defeq \cn$, $\eta\defeq \Delta^{-1}\indic{{\rm option}~1}$, and $\tO(\cdot)$ hides ${\rm polylog}\paren{H,S,A,1/\delta}$ factors.
\end{theorem}

\begin{theorem}[Transfer learning in the online setting]\label{thm:trans_online}
Suppose $\piE$ is well-posed (Definition~\ref{eq:def_well_pose}) when we receive feedback \textrm{in option $1$}. Let $\sRR$ be the ground truth reward mapping of IRL problem $(\m,\piE)$.
Then for the online setting, for sufficiently small $\epsilon\le H^{-9}(SA)^{-6}$, with probability at least $1-\delta$, \RLE{} (Algorithm~\ref{alg:RLEF}) with $N=\tO(\sqrt{H^9S^7A^7K})$ outputs a reward mapping $\hRR$ such that
\begin{align*}
     &{\sup_{\pival \textrm{is}~\Ctran\text{-}\textrm{transferable from~} \pp~ \textrm{to}~\pp'}}D^{\pp',\pival}_{\Theta}(\sRR,{\hRR})\leq\epsilon
     ,\qquad \brac{\hRR(V,A)}_{h}(s,a)\leq \brac{\sRR(V,A)}_{h}(s,a)
\end{align*}
for all $(V,A)\in \para$ and $(h,s,a)\in [H]\times \cS\times \cA$, as long as the total the number of episodes
\begin{align*}
   K + NH \ge \widetilde{\cO}\paren{\frac{ HS\Ctran A\paren{\Ctran+H^3\cns}}{\epsilon^2}+\frac{ H^2S\Ctran A\eta}{\epsilon}}.
\end{align*}

\end{theorem}

\cblue
\paragraph{Application: Performing RL algorithms in different environments}

With Theorem~\ref{thm:trans_offline} and Theorem~\ref{thm:trans_online} in place, as a concrete application, we consider utilizing rewards learned by IRL algorithms to execute RL algorithms in a different environment ($\cM'$). The following two corollaries provide guarantees for the performance of learned rewards in executing RL algorithms in the offline and the online setting, respectively. Both of these corollaries are direct consequences of Proposition~\ref{prop:mon_use_reward}.

\begin{corollary}[Performing RL algorithms with learned rewards in the offline setting]\label{coro:perform_RL_offline}
Fix $\theta=(V,A)\in \para$, 
let $r^\theta\defeq \sRR(V,A)$ and $\widehat{r}^\theta\defeq \hRR(V,A)$, where $\hRR$ are recovered reward mapping outputted by Algorithm~\ref{alg:RLPfull}.
Suppose that there exists a policy $\pi$ such that  $\pi$ is $\bar{\epsilon}$-optimal in MDP $\cM'\cup r^{\theta}$ and 
$(\pp',\pi)$ is $\wCtran$-weakly transferable from $(\pp,\pib)$ 
(Definition~\ref{def:week-Transferability}). Let $\widehat{\pi}$ be an $\epsilon'$-optimal policy in $\cM'\cup \widehat{r}^\theta$ (learned by some RL algorithms with $\widehat{r}^\theta$).
Under the same assumption of Theorem~\ref{thm:trans_offline}
, for both options, we have $V_1^{\star}(s_1;\cM'\cup r^\theta)-V_1^{\widehat{\pi}}(s_1;\cM'\cup r^\theta)\leq \epsilon+\epsilon'+2\bar{\epsilon}$, 
as long as the number of episodes
\begin{align*}
  K\geq \widetilde{\cO}\paren{\frac{ H^4S\wCtran A\cns}{\epsilon^2}+\frac{H^2S\wCtran A\eta}{\epsilon}}.
\end{align*}
Above, $\cns\defeq \cn$, $\eta\defeq \Delta^{-1}\indic{{\rm option}~1}$, and $\tO(\cdot)$ hides ${\rm polylog}\paren{H,S,A,1/\delta}$ factors.

\end{corollary}

\begin{corollary}[Performing RL algorithms with learned rewards in the online setting]\label{coro:performing RL online}
Fix $\theta=(V,A)\in \para$, 
let $r^\theta\defeq \sRR(V,A)$ and $\widehat{r}^\theta\defeq \hRR(V,A)$, where $\hRR$ are recovered reward mapping outputted by Algorithm~\ref{alg:RLEF} with $N=\tO(\sqrt{H^9S^7A^7K})$.Suppose that there exists a policy $\pi$ such that  $\pi$ is $\bar{\epsilon}$-optimal in MDP $\cM'\cup r^{\theta}$ and $\pi$ is
$\Ctran$-transferable from $\pp$ to $\pp'$ (Definition~\ref{def:Transferability}). Let $\widehat{\pi}$ be an $\epsilon'$-optimal policy in $\cM'\cup \widehat{r}^\theta$ (learned by some RL algorithms with $\widehat{r}^\theta$), then for the online setting, for sufficiently small $\epsilon\le H^{-9}(SA)^{-6}$, we have $V_1^{\star}(s_1;\cM'\cup r^\theta)-V_1^{\widehat{\pi}}(s_1;\cM'\cup r^\theta)\leq \epsilon+\epsilon'+2\bar{\epsilon}$, as long as the number of episodes
\begin{align*}
   K + NH \geq\widetilde{\cO}\paren{\frac{ HS\Ctran A\paren{\Ctran+H^3\cns}}{\epsilon^2}+\frac{ H^2S\Ctran A\eta}{\epsilon}}.
\end{align*}
\end{corollary}

\cblack

\paragraph{Application: learning IRL problems by transfer learning}
We return to the topic of transfer learning between IRL problems.
We note that our findings related to transfer learning between \MDPR{}s can also be employed in the context of transfer learning between IRL problems. As the illustrated in Theorem~\ref{thm:trans_offline} and Theorem~\ref{thm:trans_online}, we can efficiently learn a ${\hRR}$ such that the distance $\Dpivalt(\hRR,\sRR)\leq 2\epsilon$, where $\sRR$ is the ground truth reward mapping of $(\cM,\piE)$. By Assumption~\ref{ass:trans_met}, the rewards induced by $\sRR$ are solutions of $\paren{\cM',\paren{\pi'}^{\sf E}}$, hence the rewards induced by $\hRR$ also approximate the solutions of $\paren{\cM',\paren{\pi'}^{\sf E}}$.

\subsection{Proof of Theorem~\ref{thm:trans_offline}}
Note that under the same assumptions in Theorem~\ref{thm:trans_offline}, the concentration event $\cE$ defined in Lemma~\ref{lem:concentration_1} still holds with $1-\delta$.
By the week-transferablity of $(\pival, \pib)$, we have
\begin{align}
\sum_{h\in [H]}\sum_{(s,a)\in \cS\times \cA}\frac{d^{\pp',\pival}_h(s,a)}{d^{\pp,\pib}_h(s,a)}
&\leq \wCtran\sum_{h\in [H]}\sum_{(s,a)\in\cS\times \cA}\indic{d^{\pp',\pival}_h(s,a)\neq 0}\notag\\
&\leq \wCtran\sum_{h\in [H]}\sum_{(s,a)\in\cS\times \cA}\indic{a\in \pival_h(\cdot|s)}
\leq \wCtran HSA.\label{eq:transfer_1}
\end{align}
For any $\theta=(V,A)\in \para$, define $r^\theta=\sRR(V,A)$, and $\widehat{r}^\theta=\hRR(V,A)$.
With \Eqref{eq:transfer_1} at hand, we can  repeat the proof of Lemma~\ref{lem:offine_performance_decomposition}, thereby obtaining that
\begin{align}
    d^{\pp',\pival}\paren{r^\theta,\widehat{r}^\theta}\lesssim \frac{\wCtran H^2SA\eta\iota}{K}+\underbrace{\sum_{h\in[H]}\sum_{(s,a)\in \s\times\a}d^{\pp',\pival}_{h}(s,a)b^\theta_h(s,a)}_{\mrm{(I)}},\label{eq:tran_off_II}
\end{align}
holds on the event $\cE$.
where $\eta$, $b^\theta_h(s,a)$ are specified in Lemma~\ref{lem:concentration_1}.

Furthermore, similar to \Eqref{offline_main_term_I_decom}, and through the application of the triangle inequality, we can decompose $\sum_{(s,a)\in \s\times\a}d^{\pp',\pival}_{h}(s,a)b^\theta_h(s,a)$ as follows:
\begin{align}
        (\mathrm{I})&=\sum_{h\in[H]}\sum_{(s,a)\in \s\times\a}d^{\pp',\pival}_{h}(s,a)b^\theta_h(s,a)\notag\\
        &\lesssim \sum_{h\in [H]}\sum_{(s,a)\in \s\times\a}d^{\pival}_{h}(s,a)\cdot \Bigg\{\sqrt{\frac{\cn\iota}{{N}_h^b(s,a)\vee 1}\brac{\widehat{\v}_h{V}_{h+1}}(s,a)}+
 \frac{H\cn\iota}{{N}_h^b(s,a)\vee 1}\Bigg\}\notag\\
 &+\sum_{h\in [H]}\sum_{(s,a)\in \s\times\a}d^{\pival}_{h}(s,a)\cdot\frac{\epsilon}{H}\paren{1+\sqrt{\frac{\cn\iota}{{N}_h^b(s,a)\vee 1}}}\notag\\
        &\leq \underbrace{\sum_{h\in [H]}\sum_{(s,a)\in \s\times\a}d^{\pp',\pival}_{h}(s,a)\cdot \sqrt{\frac{\cn\iota}{{N}_h^b(s,a)\vee 1}\brac{{\v}_h{V}_{h+1}}(s,a)}}_{(\mathrm{I.a})}\notag\\
        &+\underbrace{\sum_{h\in [H]}\sum_{(s,a)\in \s\times\a}d^{\pp',\pival}_{h}(s,a)\cdot \sqrt{\frac{\cn\iota}{{N}_h^b(s,a)\vee 1}\brac{\paren{\widehat{\v}_h-\v_h}{V}_{h+1}}(s,a)}}_{(\mathrm{I.b})}\notag\\
        &+
 \underbrace{\sum_{h\in [H]}\sum_{(s,a)\in \s\times\a}d^{\pp',\pival}_{h}(s,a)\cdot\frac{H\cn\iota}{{N}_h^b(s,a)\vee 1}}_{(\mathrm{I.c})}\notag\\
 &
+\underbrace{\epsilon\cdot\sum_{h\in [H]}\sum_{(s,a)\in \s\times\a}d^{\pp',\pival}_{h}(s,a)\cdot\paren{\frac{1}{H}+\sqrt{\frac{\cn\iota}{H^2\cdot{N}_h^b(s,a)\vee 1}}}}_{(\mathrm{I.d})}.
   \label{tran_off_main_term_I_decom}   
   \end{align}
   Thanks to \Eqref{eq:transfer_1}, we can employ a similar argument as in the proof of \Eqref{eq:offline_main_term_I_a}, \Eqref{eq:offline_main_term_I_b}, \Eqref{eq:offline_main_term_I_c}, and \Eqref{eq:offline_main_term_I_d}, which allows us to deduce that
   \begin{align}
           &(\mrm{I.a})\lesssim \sqrt{\frac{\wCtran H^4SA\wellsupp\cn\iota}{K}},\notag\\
           &(\mrm{I.b})\lesssim\sqrt{\frac{\wCtran H^2SA\cn}{K}}+{\frac{\wCtran H^3SA\cn\iota^{5/2}}{K}},\notag\\
           &(\mrm{I.c})\lesssim {\frac{\wCtran H^2SA\cn\iota}{K}},
\epsilon\cdot(1+\sqrt{\frac{\wCtran  SA\cn\iota}{K}}),\label{tran_off_main_term_I_I}
   \end{align}
   Combining \Eqref{eq:tran_off_II}, \Eqref{tran_off_main_term_I_decom}
and \Eqref{tran_off_main_term_I_I}, we conclude that
\begin{align}
D^{\pp',\pival}_{\para}\paren{\sRR,\hRR}=\sup_{\theta\in \para}d^{\pp',\pival}\paren{r^\theta,\widehat{r}^\theta}&\lesssim \frac{\wCtran H^2SA\eta\iota}{K}+\sqrt{\frac{\wCtran H^4SA\wellsupp\cn\iota}{K}}\notag\\
&+{\frac{\wCtran H^3SA\wellsupp\cn\iota^{5/2}}{K}}+\epsilon.
\end{align}
The right-hand-side is upper bounded by $2\eps$ as long as
   \[
   K\geq\widetilde{\cO}\paren{\frac{\wCtran H^4SA\cn}{\epsilon^2}+\frac{\wCtran H^2SA\eta}{\epsilon}}.
   \]
 Here $poly \log\paren{H,S,A,1/\delta}$ are omitted.

\subsection{Proof of Theorem~\ref{thm:trans_online}}
Under the assumptions in Theorem~\ref{thm:trans_online}, the concentration event $\cE$ defined in Lemma~\ref{lem:concentration_2} still holds with $1-\delta$.
\def \bih{\bar{\cI}_h}
\def \ih{\cI_h}
Fix $\pi$ such that $\pi$ satisfies $\Ctran$-concentrability from $\pp$ to $\pp'$.
We define
\[
\bar{\cI}_h\defeq \set{(s,a)\in \cS\times \cA\mid \widehat{d}^{\pp',\pi}_h(s,a)\geq \frac{\xi}{N}+e^{\pi}_h(s,a)},
\]
for all $h\in [H]$.
Similar to \Eqref{eq:off_upper_bound_6}, we have the following decomposition:
\begin{align}\label{eq:tran_on_1}
    &d^{\pp',\pi}\paren{r_h^{\theta},\widehat{r}_h^\theta}\leq \sum_{(h,s,a)\in \hsa}d^{\pp',\pi}_h(s,a)\cdot\abs{r_h^{\theta}(s,a)-\widehat{r}_h^\theta(s,a)}\notag\\
    &\leq \underbrace{\sum_{h\in[H]}\sum_{(s,a)\notin \ih\cup\bih}d^{\pp',\pi}_h(s,a)\cdot \abs{r^{\theta}_h(s,a)-\widehat{r}^\theta_h(s,a)}}_{{(\mathrm{I})}}
    +\underbrace{\sum_{h\in[H]}\sum_{(s,a)\in \ih\cup\bih}d^{\pp',\pi}_h(s,a)\cdot \abs{r^{\theta}(s,a)-\widehat{r}_h^\theta(s,a)}}_{(\mathrm{II})},
\end{align}
where set $\cI_h$ is defined in \Eqref{eq:def_ICH}.

We further decompose the term (I) as follows:
\begin{align}
    \mrm{(I)}\leq \underbrace{\sum_{h\in[H]}\sum_{(s,a)\notin \ih}d^{\pp',\pi}_h(s,a)\cdot \abs{r^{\theta}_h(s,a)-\widehat{r}^\theta_h(s,a)}}_{\mrm{(I.a)}}+\underbrace{\sum_{h\in[H]}\sum_{(s,a)\notin \bih}d^{\pp',\pi}_h(s,a)\cdot \abs{r^{\theta}_h(s,a)-\widehat{r}^\theta_h(s,a)}}_{\mrm{(I.b)}}.
\end{align}
By the definition of transferability, there exists a policy $\pi'$ such that
\[
d^{\pp',\pi}_h(s,a)\leq 2\Ctran d^{\pp,\pi'}_h(s,a),
\]
for any $(h,s,a)\in \hsa$.
For the term (I.a), we have
\begin{align}
    (\mrm{I.a})=\sum_{h\in[H]}\sum_{(s,a)\notin \ih\cup\bih}d^{\pp',\pi}_h(s,a)\cdot \abs{r^{\theta}_h(s,a)-\widehat{r}^\theta_h(s,a)}\leq 2\Ctran\sum_{(s,a)\notin \ih\cup\bih}d^{\pp,\pi'}_h(s,a)\cdot \abs{r^{\theta}_h(s,a)-\widehat{r}^\theta_h(s,a)}
\end{align}
Similar to \Eqref{eq:off_upper_bound_7}, on the event $\cE$, we have
\begin{align*}
 \sum_{h\in[H]}\sum_{(s,a)\notin \ih\cup\bih}d^{\pp,\pi'}_h(s,a)\cdot \abs{r^{\theta}_h(s,a)-\widehat{r}^\theta_h(s,a)} 
& \leq \sum_{h\in[H]}\sum_{(s,a)\notin \ih}d^{\pp,\pi'}_h(s,a)\cdot \abs{r^{\theta}_h(s,a)-\widehat{r}^\theta_h(s,a)}\notag\\
 &\lesssim \frac{\xi H^2SA}{N}+{\frac{H^2SA\eta}{K}}+\sqrt{\frac{HSA}{K}},
\end{align*}
which allows us to bound the term (I.a) as follows:
\begin{align}
    (\mrm{I.a})\lesssim  \frac{\Ctran \xi H^2SA}{N}+{\frac{\Ctran H^2SA\eta}{K}}+\Ctran\sqrt{\frac{HSA}{K}}.\label{eq:trans_offlne_I_a}
\end{align}
For the term (I.b), on the event $\cE$, we have
\begin{align}
   &{(\mathrm{I.b})}=\sum_{h\in[H]}\sum_{(s,a)\notin \bih}d^{\pp',\pi}_h(s,a)\cdot \abs{r^{\theta}(s,a)-\widehat{r}^\theta(s,a)}\notag\\
    &=\sum_{h\in[H]}\sum_{(s,a)\notin \bih}d^{\pp',\pi}_h(s,a)\cdot\Bigg|-A_{h}(s,a)\paren{\indic{a\in \supp\paren{\piE_h(\cdot|s)}}-\indic{a\in \supp\paren{\widehat{\pi}^{\sf E}_h(\cdot|s)}}}\notag\\
    &-\brac{\paren{\p_h-\widehat{\p}_h}V_{h+1}}(s,a)
    (s,a)-b^\theta_h(s,a)\Bigg|\notag\\
    &\leq  \sum_{h\in[H]}\sum_{(s,a)\notin \bih}d^{\pp',\pi}_h(s,a)\cdot\Bigg\{\abs{{A_{h}(s,a)\cdot\paren{\indic{a\in \supp\paren{\widehat{\pi}^{\sf E}_h(\cdot|s)}}-\cdot\indic{a\in\supp\paren{\piE_h(\cdot|s)}}}}}\notag\\
&+\abs{\brac{(\p_h-\widehat{\p}_h)V_{h+1}}(s,a)}+b^\theta_h(s,a)\Bigg\}\tag{by triangle inequality}\\
&\overset{\text{(i)}}{\lesssim}H\cdot \sum_{h\in[H]}\sum_{(s,a)\notin \bih}d^{\pp',\pi}_h(s,a)\notag\\
    &\overset{\text{(ii)}}{\leq} 2\Ctran H\cdot \sum_{h\in[H]}\sum_{(s,a)\notin \bih}\paren{\widehat{d}^{\pp,\pi'}_h(s,a)+e^{\pi'}_h(s,a)+\frac{\xi}{N}}\notag\\
    &\overset{\text{(iii)}}{\lesssim} \Ctran H\cdot \sum_{h\in[H]}\sum_{(s,a)\notin \bih}\paren{e^{\pi'}_h(s,a)+\frac{\xi}{N}}\notag\\
    &\lesssim \frac{\Ctran\xi H^2SA}{N}+\Ctran\sqrt{\frac{HSA}{K}},\label{eq:trans_offlne_I_b}
\end{align}
 where (i) is by $\|A_h\|_{\infty}$, $\|V_{h+1}\|_{\infty},b^\theta_h(s,a)\leq H$, (ii) comes from \Eqref{eq:d-error} and the concentration event
 $\cE(ii)$, and (iii) follows from the definition of $\bih$.
 \def \bcih{\ih\cup\bih}

 Combining \Eqref{eq:trans_offlne_I_a} and \Eqref{eq:trans_offlne_I_b}, we can conclude that
 \begin{align}
     \mrm{(I)}\lesssim \frac{\Ctran\xi H^2SA}{N}+{\frac{\Ctran H^2SA\eta}{K}}+\Ctran\sqrt{\frac{HSA}{K}}.\label{eq:trans_online_term_I}
 \end{align}

 For the term (II), following a similar approach as in \Eqref{eq:off_upper_bound_8}, we have
 \begin{align}
     \mrm{(II)}=&\sum_{h\in[H]}\sum_{(s,a)\in \bcih}{{d}_{h}^{\pp',\pi}(s, a)\cdot b^{\theta}_h(s,a)}\notag\\
     &=\sum_{h\in[H]}\sum_{(s,a)\in \bcih}{d}_{h}^{\pp',\pi}(s, a)\cdot\min\Bigg\{\sqrt{\frac{\cn\iota}{\widehat{N}_h^b(s,a)\vee 1}\brac{\widehat{\v}_h{V}_{h+1}}(s,a)}+
  \frac{H\cn\iota}{\widehat{N}_h^b(s,a)\vee 1}\notag\\
  &+\frac{\epsilon}{H}\paren{1+\sqrt{\frac{\cn\iota}{\widehat{N}_h^b(s,a)\vee 1}}}, H\Bigg\}\notag\\
     &\overset{\text{(i)}}{\leq} \sum_{h\in[H]}\sum_{(s,a)\in \bcih}{d}_{h}^{\pp',\pi}(s, a)\cdot\Bigg\{\min\set{\sqrt{\frac{\cn\iota}{\widehat{N}_h^b(s,a)\vee 1}\brac{\widehat{\v}_h{V}_{h+1}}(s,a)}, H}+
  \frac{H\cn\iota}{\widehat{N}_h^b(s,a)\vee 1}\\
  &+\frac{\epsilon}{H}\paren{1+\sqrt{\frac{\cn\iota}{\widehat{N}_h^b(s,a)\vee 1}}}\Bigg\}\notag\\
 &\overset{\text{(ii)}}{\leq}  \sum_{h\in[H]}\sum_{(s,a)\in \bcih}{d}_{h}^{\pp',\pi}(s, a)\cdot\Bigg\{{\sqrt{\frac{\cn\iota\brac{\widehat{\v}_h{V}_{h+1}}(s,a)+H}{\widehat{N}_h^b(s,a)\vee 1+1/H}}}+
  \frac{H\cn\iota}{\widehat{N}_h^b(s,a)\vee 1}\notag\\
  &+\frac{\epsilon}{H}\paren{1+\sqrt{\frac{\cn\iota}{\widehat{N}_h^b(s,a)\vee 1}}}\Bigg\}\notag\\
 &\overset{\text{(iii)}}{=}\underbrace{\sum_{h\in[H]}\sum_{(s,a)\in \bcih}{d}_{h}^{\pp',\pi}(s, a)\cdot{\sqrt{\frac{\cn\iota\brac{\widehat{\v}_h{V}_{h+1}}(s,a)+H}{K\EE_{\pi'\sim\mu^{\sf b}}\brac{\widehat{d}^{\pi'}_h\paren{s,a}}+1/H}}}}_{(\mathrm{II.a})}\notag\\
  &+
 \underbrace{\sum_{h\in[H]}\sum_{(s,a)\in \bcih}{d}_{h}^{\pp',\pi}(s, a)\cdot
  \frac{H\cn\iota}{K\EE_{\pi'\sim\mu^{\sf b}}\brac{\widehat{d}^{\pi'}_h\paren{s,a}}+1/H}}_{(\mathrm{II.b})}\notag\\
  &+
 \underbrace{\frac{\epsilon}{H}\sum_{h\in[H]}\sum_{(s,a)\in \bcih}{d}_{h}^{\pp',\pi}(s, a)\cdot\paren{1+\sqrt{\frac{\cn\iota}{K\EE_{\pi'\sim\mu^{\sf b}}\brac{\widehat{d}^{\pi'}_h\paren{s,a}}+1/H}}}}_{(\mathrm{II.c})}.
 \end{align}
 For the term (II.a), by the Cauchy-Schwarz inequality, we have
 \begin{align*}
 &(\mathrm{II.a})\leq \sqrt{\Ctran}\sum_{h\in [H]}\sum_{(s,a)\in \bcih}\sqrt{{d}_{h}^{\pp',\pi}(s, a)\cdot \paren{\cn\iota\brac{{\v}_h{V}_{h+1}}\left(s,a\right)+H}}\notag\\
 &\cdot \sqrt{\frac{d^{\pp,\pi'}_h(s,a)}{K\EE_{\pi'\sim\mu^{\sf b}}\brac{\widehat{d}^{\pi'}_h\paren{s,a}}+1/H}}\notag\\
 &\lesssim 
 \sqrt{\Ctran}\sum_{h\in [H]}\sum_{(s,a)\in \bcih}\sqrt{{d}_{h}^{\pp',\pi}(s, a)\cdot \paren{\cn\iota\brac{{\v}_h{V}_{h+1}}\left(s,a\right)+H}}\notag\\
 &\cdot \sqrt{\frac{2\widehat{d}^{\pp,\pi'}_h(s,a)+2e^{\pi'}_h(s,a)+\frac{\xi}{2N}}{K\EE_{\pi'\sim\mu^{\sf b}}\brac{\widehat{d}^{\pi'}_h\paren{s,a}}+1/H}}\notag\\
 &\lesssim 
 \sqrt{\Ctran}\sum_{h\in [H]}\sum_{(s,a)\in \bcih}\sqrt{{d}_{h}^{\pp',\pi}(s, a)\cdot \paren{\cn\iota\brac{{\v}_h{V}_{h+1}}\left(s,a\right)+H}}\notag\\
& \cdot \sqrt{\frac{\widehat{d}^{\pp,\pi'}_h(s,a)}{K\EE_{\pi'\sim\mu^{\sf b}}\brac{\widehat{d}^{\pi'}_h\paren{s,a}}+1/H}}\notag\\
 &\leq \sqrt{\Ctran}\underbrace{\set{\sum_{h\in [H]}\sum_{(s,a)\in \s\times\a}{d}_{h}^{\pp',\pi}(s, a)\cdot \paren{\cn\iota\brac{{\v}_h{V}_{h+1}}\left(s,a\right)+H}}^{1/2}}_{(\mathrm{II.a.1})}\\
 &~~\times \underbrace{\set{\sum_{h\in [H]}\sum_{(s,a)\in \s\times\a}\frac{\widehat{d}^{\pp,\pi'}_h(s,a)}{K\EE_{\pi'\sim\mu^{\sf b}}\brac{\widehat{d}^{\pi'}_h\paren{s,a}}+1/H}}^{1/2}}_{(\mathrm{II.a.2})}.
 \end{align*}
Following similar approaches as in \Eqref{eq:online_term_II_a_1} and \Eqref{eq:online_term_II_a_2}, we have
 \begin{align}
     (\mrm{II.a.1})\lesssim\sqrt{H^3\cn\iota},\qquad (\mrm{II.a.2})\lesssim \sqrt{\frac{HSA}{K}},
 \end{align}
 which implies that
 \begin{align}
     (\mrm{II.a})\lesssim \sqrt{\frac{\Ctran H^4SA\cn\iota}{K}}.\label{eq:trans_online_I_a}
 \end{align}
For the term (II.b), by \Eqref{eq:stopping_rule}, we have
 \begin{align}\label{eq:trans_online_I_b}
    (\mathrm{II.b})&= \sum_{h\in [H]}\sum_{(s,a)\in \s\times\a}{d}_{h}^{\pp',\pi}(s, a)\cdot \frac{H\cn\iota}{K\EE_{\pi'\sim\mu^{\sf b}}\brac{\widehat{d}^{\pi'}_h\paren{s,a}}+1/H}\notag\\
    &=\frac{\Ctran}{K}\cdot\sum_{h\in [H]}\sum_{(s,a)\in \s\times\a}{d}_{h}^{\pp,\pi'}(s, a)\cdot \frac{H\cn\iota}{\EE_{\pi'\sim\mu^{\sf b}}\brac{\widehat{d}^{\pi'}_h\paren{s,a}}+1/KH}\notag\\
    &\lesssim\frac{\Ctran}{K}\cdot\sum_{h\in [H]}\sum_{(s,a)\in \s\times\a}\widehat{d}_{h}^{\pp,\pi'}(s, a)\cdot \frac{H\cn\iota}{\EE_{\pi'\sim\mu^{\sf b}}\brac{\widehat{d}^{\pi'}_h\paren{s,a}}+1/KH}\notag\\
    &\lesssim \frac{\Ctran H^2SA\cn\iota}{K}.
 \end{align}
For the term (II.c), we have
 \begin{align}
     \mathrm{(II.c)}&=\frac{\epsilon}{H}\sum_{h\in[H]}\sum_{(s,a)\in \bcih}{d}_{h}^{\pp,\pi}(s, a)\cdot\paren{1+\sqrt{\frac{\cn\iota}{K\EE_{\pi'\sim\mu^{\sf b}}\brac{\widehat{d}^{\pi'}_h\paren{s,a}}+1/H}}}\notag\\
     &=\epsilon+\frac{\sqrt{\Ctran}\epsilon}{H}\sum_{h\in[H]}\sum_{(s,a)\in \bcih}\sqrt{{d}_{h}^{\pp',\pi}(s, a)}\cdot\sqrt{\frac{{d}_{h}^{\pp,\pi'}(s, a)\cn\iota}{K\EE_{\pi'\sim\mu^{\sf b}}\brac{\widehat{d}^{\pi'}_h\paren{s,a}}+1/H}}\notag\\
     &\lesssim \epsilon+\frac{\sqrt{\Ctran}\epsilon}{H}\sum_{h\in[H]}\sum_{(s,a)\in \bcih}\sqrt{{d}_{h}^{\pp',\pi}(s, a)}\cdot\sqrt{\frac{\widehat{d}_{h}^{\pp,\pi'}(s, a)\cn\iota}{K\EE_{\pi'\sim\mu^{\sf b}}\brac{\widehat{d}^{\pi'}_h\paren{s,a}}+1/H}}\notag\\
     &\leq \epsilon+\frac{\epsilon}{H}
     \sqrt{\sum_{h\in[H]}\sum_{(s,a)\in \bcih} {d}_{h}^{\pp',\pi}(s, a)}\cdot\sqrt{\sum_{h\in[H]}\sum_{(s,a)\in \mc{I}_h}\frac{\widehat{d}_{h}^{\pp,\pi'}(s, a)\cn\iota}{K\EE_{\pi'\sim\mu^{\sf b}}\brac{\widehat{d}^{\pi'}_h\paren{s,a}}+1/H}}\notag\\
     &\leq \epsilon\paren{1+\sqrt{\frac{\Ctran SA\cn\iota}{HK}}},\label{eq:trans_online_I_c} 
 \end{align}
where the second last line is by the Cauchy-Schwarz inequality and the last line is by \Eqref{eq:stopping_rule}.

 Then combining \Eqref{eq:trans_online_I_a} \Eqref{eq:trans_online_I_b}, and \Eqref{eq:trans_online_I_c}, we obtain the bound for the term (II)
 \begin{align}\label{eq:trans_online_term_II}
 (\mathrm{II})&\lesssim(\mathrm{II.a})+(\mathrm{II.b})+(\mathrm{II.c})\notag\\
 &\lesssim\sqrt{\frac{\Ctran H^4SA\cn\iota}{K}} +\frac{\Ctran H^2SA\cn\iota}{K}+\epsilon(1+\sqrt{\frac{\Ctran\cn\iota}{HK}})\notag\\
 &\lesssim\sqrt{\frac{\Ctran H^4SA\cn\iota}{K}}+\epsilon,
 \end{align}
 where the last line is from $\epsilon<1$.

 Finally, combining \Eqref{eq:trans_online_term_I} and \Eqref{eq:trans_online_term_II}, we get the final bound 
 \begin{align*}
     \DallTheta\paren{\RR^\star,\hRR}&= \sup_{\pi,\theta\in \para}d^{\pi}\paren{r_h^{\theta},\widehat{r}_h^\theta}\leq (\mathrm{I})+(\mathrm{II})\notag\\
     &\lesssim \frac{\Ctran \xi H^2SA}{N}+\sqrt{\frac{\Ctran H^4SA\cn\iota}{K}} +{\frac{\Ctran H^2SA\eta}{K}}+\Ctran\sqrt{\frac{HSA}{K}}+\epsilon
 \end{align*}
Hence, we can guarantee $ \DallTheta\paren{\RR^\star,\hRR}\leq 2\epsilon$, provided that
\begin{align}
&KH \geq N \geq \widetilde{\cO}\paren{{\sqrt{H^9S^7A^7K}}},\notag\\
&K\geq\widetilde{\cO}\paren{\frac{\Ctran HSA\paren{\Ctran+H^3\cn}}{\epsilon^2}+\frac{\Ctran H^2SA\eta}{\epsilon}}\label{eq:rarte_1}
\end{align} 
 Here $poly \log\paren{H,S,A,1/\delta}$ are omitted.
 Similar to the proof of Theorem~\ref{thm:online_main}, suppose $\epsilon\leq H^{-9}(SA)^{-6}$, set $N=\widetilde{\cO}$, when 
 \begin{align}
     K\geq\widetilde{\cO}\paren{\frac{\Ctran HSA\paren{\Ctran+H^3\cn}}{\epsilon^2}+\frac{\Ctran H^2SA\eta}{\epsilon}},
 \end{align}
\Eqref{eq:rarte_1} holds. And at this time, the total sample complexity is 
\begin{align}
    K+NH\geq \widetilde{\cO}\paren{\frac{\Ctran HSA\paren{\Ctran+H^3\cn}}{\epsilon^2}+\frac{\Ctran H^2SA\eta}{\epsilon}}.
\end{align}

 \end{document}